\documentclass[onefignum,onetabnum]{siamonline250211}
\usepackage{etoolbox}
\usepackage{amsmath}

\usepackage{graphicx}% Include figure files
\usepackage{dcolumn}% Align table columns on decimal point
\usepackage{bm}% bold math
\usepackage[utf8]{inputenc}
\usepackage[T1]{fontenc}
\usepackage{etoolbox}

\usepackage{amsfonts, bbold, dsfont, graphicx, cancel, array, mathtools, amssymb}
\usepackage{tikz, listofitems, xcolor, bigints, pgfplots}
\pgfplotsset{compat=1.18}
 
\usepackage{caption, subcaption}
\DeclareMathOperator*{\argmin}{argmin}
\usepackage{textcomp}
\usepackage{xcolor}
\usepackage{tabularx}

\usepackage{nicefrac}
\def\x{\bm x}
\def\f{\bm f}
\def\g{\bm g}
\def \delt{\Delta t}
\def \S{\bm S}
\def\y{\bm y}
\def\z{\bm z}
\def\w{\bm w}
\def\Z{\bm Z}
\def\I{\bm I}
\def\b{\bm b}
\def\ftilde{\Tilde{\bm {f}}}
\def\Stilde{\Tilde{\bm {S}}}
\def\xtilde{\Tilde{\bm {x}}}

\def\sgn{\text{sgn}}

\def\X{\bm X}

\def\R{\mathds{R}}
\def\rossler{R\"{o}ssler }
\def\lhalf{\ell_{\nicefrac{1}{2}}}

\definecolor{aquamarine}{rgb}{0.5, 1.0, 0.83}
\colorlet{myred}{red!80!black}
\colorlet{myblue}{blue!80!black}

\usetikzlibrary{automata, positioning}

\tikzstyle{data}=[thick, rectangle, draw=aquamarine, fill=aquamarine!10, minimum size=45,inner sep=0.5,outer sep=0.6]
\tikzstyle{neuron}=[thick, circle, draw=myblue!30, fill=myblue!8, minimum size=40, inner sep=0.5, outer sep=0.6]
\tikzstyle{layer}=[thick, rounded corners, draw=myred!30, fill=myred!10, minimum size=40, inner sep=0.5, outer sep=0.6]
\tikzstyle{connect}=[thick, gray]
\tikzstyle{connect arrow}=[->,thick,gray,shorten <=0.5,shorten >=1]
\tikzstyle{annot} = [text width=5em, text centered]
\tikzstyle{annot2} = [text width=7em, text centered]

\tikzstyle{NN}=[thick, rectangle, draw=magenta, fill=magenta!10, minimum size=60,inner sep=0.5,outer sep=0.6]
\tikzstyle{derivative}=[thick, circle, draw=myblue!30, fill=myblue!8, minimum size=40, inner sep=0.5, outer sep=0.6]
\tikzstyle{ODE}=[thick, rectangle, draw=gray, fill=gray!8, minimum size=30,inner sep=0.5,outer sep=0.6]
\tikzstyle{state}=[thick, circle, draw=aquamarine, fill=aquamarine!10, minimum size=40, inner sep=0.5, outer sep=0.6]
\tikzstyle{active arrow}=[->,thick,blue,shorten <=0.5,shorten >=1]
\tikzstyle{bi arrow}=[<->,thick,red,shorten <=0.5,shorten >=1]

\usepackage{lipsum}
\usepackage{amsfonts}
\usepackage{graphicx}
\usepackage{epstopdf}
\usepackage{algorithmic}
\ifpdf
  \DeclareGraphicsExtensions{.eps,.pdf,.png,.jpg}
\else
  \DeclareGraphicsExtensions{.eps}
\fi

\usepackage{enumitem}
\setlist[enumerate]{leftmargin=.5in}
\setlist[itemize]{leftmargin=.5in}

\newsiamremark{remark}{Remark}
\newsiamremark{hypothesis}{Hypothesis}
\crefname{hypothesis}{Hypothesis}{Hypotheses}
\newsiamthm{claim}{Claim}

\headers{Symbolic Neural ODEs}{N. Boddupalli and J. Moehlis}
\date{today}
\title{Symbolic Neural ODEs:\\Learning interpretable models from time-series data \thanks{Submitted to the editors June 12, 2026.
\funding{This work was supported by National Science Foundation Grant No. NSF-2016004.}}}

\author{Nibodh Boddupalli$^{\dagger}$\and Jeff Moehlis\thanks{Dept. of Mechanical Engineering, University of California, Santa Barbara, CA 93106
  (\email{nibodh@ucsb.edu}, \email{moehlis@ucsb.edu}).}
}

\usepackage{amsopn}

\makeatletter
\newcommand*{\addFileDependency}[1]{% argument=file name and extension
  \typeout{(#1)}% latexmk will find this if $recorder=0 (however, in that case, it will ignore #1 if it is a .aux or .pdf file etc and it exists! if it doesn't exist, it will appear in the list of dependents regardless)
  \@addtofilelist{#1}% if you want it to appear in \listfiles, not really necessary and latexmk doesn't use this
  \IfFileExists{#1}{}{\typeout{No file #1.}}% latexmk will find this message if #1 doesn't exist (yet)
}
\makeatother

\newcommand*{\myexternaldocument}[1]{%
    \externaldocument{#1}%
    \addFileDependency{#1.tex}%
    \addFileDependency{#1.aux}%
}
\ifpdf
\hypersetup{
  pdftitle={SIADS-Symbolic-Neural-ODEs},
  pdfauthor={N. Boddupalli and J. Moehlis}
}
\fi

\myexternaldocument{ex_supplement}

\begin{document}
\date{today}
\maketitle

% REQUIRED
\begin{abstract}
We present a machine learning framework for identifying sparse, interpretable models of dynamical systems directly from time-series data. Our approach parameterizes the underlying vector field using a neural architecture and trains it by minimizing a multi-step prediction loss over a finite horizon. To ensure numerical tractability, we optimize a mean absolute error objective averaged across prediction steps, and progressively increase the horizon during training. A key feature of this formulation is that it enforces consistency under repeated composition of the learned dynamics. As a result, the identified models exhibit significantly improved stability compared with approaches based on one-step regression of the vector field. When combined with sparsity-promoting regularization, this leads to parsimonious models that generalize beyond the training data. We demonstrate accurate recovery of systems exhibiting a wide range of behaviors, including stable and unstable fixed points, periodic orbits, and chaotic attractors. For chaotic systems, while long-term trajectory prediction is inherently limited by sensitivity to initial conditions, we show that multi-step training yields models with accurate short-term dynamics and strong agreement in long-time statistical properties, including mean, variance, and Lyapunov exponents. Moreover, we establish theoretical bounds linking trajectory error to statistical accuracy, providing a step toward a principled explanation for this behavior.
\end{abstract}

% REQUIRED
\begin{keywords}
  Nonlinear dynamics, Time-series, System Identification, Neural ODEs, Scientific Machine Learning (SciML), Interpretability, Multi-step prediction loss, Compositional consistency
\end{keywords}

% REQUIRED
\begin{AMS}
  93B30, 37M10, 68T07, 65L09, 34C28
\end{AMS}

\section{Introduction}

Mathematical models are central to describing time-evolving phenomena in science and engineering. From celestial mechanics to biological regulation, the ability to predict, control, and analyze dynamical systems depends critically on the accuracy of the underlying models. While many such models are derived from first principles, there is a growing need to infer governing equations directly from observed data, particularly in complex systems where mechanistic descriptions are incomplete or unavailable.

This challenge has led to a broad range of approaches for data-driven modeling, spanning system identification, time-series analysis, statistical inference, and machine learning. Classical techniques in system identification and control theory often rely on structured representations such as transfer functions or state-space models~\cite{tangirala2018principles}, which are well-suited for linear systems but become restrictive in nonlinear settings. Alternative approaches based on autoregressive models and delay embeddings~\cite{takens2006detecting, havok} provide flexible representations of temporal structure, while modal decomposition techniques such as Dynamic Mode Decomposition (DMD)~\cite{DMD, arba17} and Proper Orthogonal Decomposition (POD)~\cite{POD} offer low-dimensional descriptions of complex dynamics, particularly in post-transient regimes. These methods are closely related to the Koopman operator framework~\cite{mezic2005spectral}, but their finite-dimensional approximations are generally limited in their ability to capture non-hyperbolic attractors or chaotic dynamics~\cite{budi12}.

A complementary line of work focuses on directly learning the governing equations of motion. Sparse Identification of Nonlinear Dynamics (SINDy)~\cite{sindy} and related methods approximate the vector field as a linear combination of candidate functions, enabling interpretable models that can be analyzed using tools from dynamical systems theory~\cite{guckenheimer2013nonlinear}. These approaches typically rely on predefined dictionaries of basis functions and regression techniques, often with sparsity-promoting regularization such as $\ell_1$ penalties~\cite{l1_regularization}. While computationally efficient, such “fixed dictionary” methods face two key challenges: the combinatorial growth of candidate functions with system dimension (the curse of dimensionality), and the difficulty of selecting appropriate basis functions a priori.

More flexible alternatives include nonlinear regression and symbolic regression~\cite{bong07,quad16,quade2019glyph, ai_hilbert}, as well as machine learning methods based on neural networks. Neural networks (NNs) provide powerful function approximation capabilities~\cite{chen1995universal, goodfellow2016deep}, and have been widely applied to time-series forecasting and predictions. However, purely black-box approaches often lack interpretability and may fail to capture the structural properties of dynamical systems. Recent developments such as neural ordinary differential equations NODEs~\cite{neural_odes} and related architectures aim to bridge this gap by learning continuous-time dynamics, while extensions incorporating sparsity or symbolic structure seek to recover interpretable models~\cite{fronk23, raissi2018multistep}. More recently, transformer-based architectures have been introduced to map time-series directly to governing equations. However, these models exhibit poor generalizability unless pre-trained on massive in-distribution datasets using millions of parameters and enterprise GPUs. Consequently, they are often restricted to low-dimensional systems, such as strictly one-dimensional data~\cite{transformers_1D}. While some architectures accommodate multi-dimensional data~\cite{odeformer}, their capabilities are limited to analyzing stationary data, a single time-series, and non-chaotic dynamics.

Despite these advances, a fundamental challenge remains: most existing approaches learn dynamics by minimizing one-step prediction error, either by directly regressing the vector field or by matching short-term state transitions. While regression techniques via a weak formulation~\cite{weak_sindy, rosenfeld2022dynamic} mitigate challenges of estimating dynamics from noisy state data compared to estimating derivatives, local approximations do not guarantee that the learned model behaves correctly under repeated composition. In particular, small one-step errors can accumulate over time, leading to instability, spurious dynamics, or poor long-term predictions.

In this work, we address this issue by training models using a multi-step prediction loss over a finite horizon. Rather than fitting the vector field locally, we optimize the agreement between predicted and observed trajectories across multiple time steps. This enforces consistency under composition of the learned dynamics and implicitly penalizes models whose errors grow over time. To make this approach computationally tractable, we minimize a mean absolute error objective averaged across a prediction horizon and progressively increase the horizon during training. This perspective leads to several key advantages. First, multi-step training promotes models that are stable under iteration, avoiding pathological behaviors that can arise in one-step methods. Second, when combined with sparsity-promoting regularization, it yields compact, interpretable representations of the underlying dynamics. Third, it provides a natural framework for handling noisy data, as consistency across multiple steps acts as an implicit regularizer.

A particularly important setting is that of chaotic systems, where sensitivity to initial conditions makes long-term trajectory prediction inherently difficult. In such systems, even small modeling errors lead to exponential divergence of trajectories. Our framework clarifies this behavior by distinguishing between finite-horizon trajectory accuracy and long-term statistical fidelity. While multi-step training controls trajectory error over the training horizon, we show that it also leads to accurate estimation of statistical quantities such as means, variances, and Lyapunov exponents. We provide theoretical bounds linking trajectory error to statistical error, providing a step toward explaining why models can reproduce invariant properties of chaotic systems even when individual trajectories diverge.

We demonstrate the effectiveness of our approach on a range of benchmark systems exhibiting diverse dynamical behaviors, including fixed points, periodic orbits, and chaotic attractors. Across these examples, we observe that increasing the prediction horizon improves stability, robustness to noise, and agreement with both short-term dynamics and long-term statistics.

The paper is organized as follows: In \cref{sec:problem_formulation}, we outline our problem using the available data and objects of interest as a neural ODE (NODE). In \cref{sec:symbolic_nn}, we explain our neural network (NN) approach, formulate the numerical approximation of the continuous-time NODE, sparsity promoting regularization, information-theoretic post-training distillation, and the optimization methodology. In \cref{sec:Examples}, we demonstrate our approach by estimating dynamics of noisy time-series data from examples of dynamical systems spanning fixed points, periodic orbits, limit-cycles, and chaotic attractors. We summarize and discuss our findings, observations, and limitations in \cref{sec:conclusions}. For further reference, approximations and errors are available in \cref{sec:bounds} - \cref{appendix:trajectory_tracking}, bounds on statistics are shown \cref{sec:bounds_statistics}, and additional computation details are given in \cref{appendix:compute}.

\section{Problem formulation}
\label{sec:problem_formulation}

In this paper we consider time-series data from finite dimensional autonomous dynamical systems of the form
\begin{equation}\label{sys}
    \frac{d\x}{dt} = \f(\x),
\end{equation}
%\\
where $\x \in \mathds{M}$, $\f: \mathds{M} \to T(\mathds{M})$, and the $i^{th}$ state of $\x$ is denoted as $x_i$. We consider a few examples where $\mathds{M} \subseteq \R\times S^1$ but since most of our examples are $\mathds{M} = \R^n$, we simplify notation as $\theta_i \equiv x_i \bmod 2 \pi$.
Since we consider state data available at discrete time-intervals $\delt$, we consider the stroboscopic map
\begin{equation}\label{flow}
    \S(\x_0) \equiv \x(\delt) = \x_0 + \int_0^{\delt} \f(\x) dt,
\end{equation}
where $\x_0$ is an initial state. For brevity of notation, we denote the data points $\x(\delt), \x(2 \delt), \\ \cdots , \x(h \delt)$ using compositions of the stroboscopic map as
\begin{equation}\label{data_point}
    \x[h] \equiv \S^h(\x_0) \equiv \underbrace{\S \circ \S \circ \cdots \circ \S}_{h \text{ compositions}}(\x_0) = \x(h\delt),
\end{equation}
where $h \in \mathds{N}$. This data may be sampled from multiple trajectories or from multiple fragments of the same orbit. Suppose we have
\begin{equation}\label{m_total}
    m^* = m + p\times H
\end{equation}
total data points from $p$ non-consecutive trajectories and we set a maximum forecasting horizon of $H$ time-steps of numerical integration for the NODE. Then, we have $m$ data points denoted $\X_0$ for which the state data up to $H$ time-steps ahead is available to us. We denote this dataset as
\begin{equation}\label{X0}
        \X_0 \equiv \{\x[0] \mid \text{we know }\x[h] ~ \forall ~h = 1,2, \cdots, H\}.
\end{equation}
Therefore, even when we have data sampled sequentially from only one trajectory, this can be viewed as $m$ Initial Value Problems (IVPs): one for each data point in $\X_0$. We can then denote the subsets of the available data as
\begin{equation}\label{Xh}
        \X_h \equiv \{\x[h] = \S^h(\x[0]) \mid \x[0] \in \X_0\} \text{ for } h = 1,2, \cdots, H.
\end{equation}
Collectively denoting the above datasets as $\X \equiv \bigcup_{h=0}^H \X_h$ which has the full $m^*$ data points, we can outline the goal of this work as that of estimating $\f$ using $\X$.

Neural ordinary differential equations NODEs model the dynamics of a system by representing the governing vector field $\f(\x)$ with a parameterized function $\tilde{\f}(\x)$, and generating trajectories via numerical integration. In our setting, given an initial condition $\x[0] \in \X_0$, the learned dynamics produce a sequence $\{\tilde{\x}[h]\}_{h=1}^H$ by integrating
\begin{equation}
\frac{d\x}{dt} = \tilde{\f}(\x),
\end{equation}
using a numerical scheme such as a fourth-order Runge--Kutta method.

Due to the discrete sampling of the data, it is natural to view the dynamics through the stroboscopic map $\S(\x)$. The true system satisfies
\[
\x[h] = \S^h(\x[0]),
\]
while the learned model induces an approximate map $\tilde{\S}$ through numerical integration of $\tilde{\f}$, yielding
\[
\tilde{\x}[h] = \tilde{\S}^h(\x[0]).
\]
Thus, forecasting over multiple time steps corresponds to repeated composition of the learned map.

A fundamental requirement for any forecasting method is that it remains stable under repeated composition. While a model may achieve small error over a single time step, even small discrepancies in the learned dynamics can lead to large errors after repeated application.
%\[
%\tilde{\x}[h+1] = \tilde{\S}(\tilde{\x}[h]).
%\]
In nonlinear systems, these errors may grow rapidly, and in chaotic systems, even infinitesimal differences can lead to divergence of trajectories.  This issue is particularly important for NODEs, since the model is evaluated recursively through numerical integration. For example, using a forward Euler scheme,
\begin{equation}
\tilde{\x}[1] = \x[0] + \Delta t \, \tilde{\f}(\x[0]),
\end{equation}
while
\begin{equation}
\tilde{\x}[2] = \tilde{\x}[1] + \Delta t \, \tilde{\f}(\tilde{\x}[1]),
\end{equation}
so that the error at each step propagates into subsequent evaluations.

To account for this effect, we define the NODE horizon - the number of time-steps $H$ over which the NODE is optimized -- and construct a loss that measures the discrepancy between true and predicted trajectories over multiple steps. This ensures that the learned model $\tilde{\f}$ is trained not only for local accuracy, but also for its behavior under repeated composition.  As discussed in \cref{sec:horizon}, increasing the NODE horizon makes the optimization problem more sensitive to error growth over time, thereby favoring models with stable long-term behavior.  In this work, we adopt a SymANNTEx-style approach \cite{symanntex_paper} to represent $\tilde{\f}(\x)$ as a sparse, interpretable NN. This is particularly well-suited to the NODE setting, since the learned vector field must remain well-behaved under repeated evaluation during numerical integration. We abbreviate this \underline{Sym}ANNTEx-based \underline{N}eural \underline{ODE} as SymNODE. Enforcing sparsity and structure in $\tilde{\f}$ helps mitigate overfitting and improves stability under composition.

The NODE horizon $H$ plays a central role in the NODE formulation. It determines the extent to which the learned model is trained to remain accurate under repeated composition, and thus controls the balance between short-term accuracy and long-term stability. In practice, $H$ must be chosen to be large enough to capture meaningful dynamical behavior, while avoiding excessive sensitivity to numerical and chaotic effects.

\section{Symbolic Neural Network}\label{sec:symbolic_nn}

In this paper, we build upon the symbolic neural network (SymNN) architecture formulated in SymANNTEx~\cite{symanntex_paper} to recursively perform numerical integration and optimize the network weights to fit the state data over a finite time horizon.
As in~\cite{symanntex_paper}, the NN has two principal tiers of organization:  ``stacks'' and ``operational layers''. Stacks define the higher-level organizational structure, with the output of each stack being an input to subsequent stacks. The number of stacks is $K \ge 1$. Each stack is composed of $L \ge 1$ operational layers, which use primitive operations and functions to generate terms that are used in subsequent stacks and the final expression.  The SymNN architecture allows compositions and combinations of the primitive operations and functions to be generated -- thus ``generating'' a dictionary. % giving models of the form
% \begin{equation}\label{generative_dict}
% \ftilde(\x) = F_K \circ \cdots \circ F_2 \circ F_1(\x),
% \end{equation}
% where $F_k(z)$ is the (possibly $k$-dependent) subset of all possible functions of $z$ generated by the primitive operations and functions.  Here the composition $\circ$ is interpreted in an element-wise manner; for example, $S_2 \circ S_1$ can contain functions $f_2(f_1(x))$ for any $f_1 \in S_1$ and $f_2 \in S_2$.  We note that each $S_k$ includes the Identity operator $\mathcal{I}$.
This will be elaborated in \cref{NN_eqn}.  

\begin{figure}[h]
    \centering
    \resizebox{0.5\columnwidth}{!}{%
\begin{tikzpicture}

\node[data] (input) {$\z^{(k-1)}$};

\node[neuron, xshift=75, yshift=60, at=(input)] (ops_in) {$\w_{in} \cdot \z$};

\node[neuron, xshift=135, yshift=150, at=(input)] (exp) {$\exp{(\w_{in} \cdot \z)}$};
\node[neuron, xshift=135, yshift=90, at=(input)] (sin) {$\sin{(\w_{in} \cdot \z)}$};
\node[neuron, xshift=135, yshift=30, at=(input)] (sgn) {sgn${(\w_{in} \cdot \z)}$};

\node[neuron, xshift=195, yshift=60, at=(input)] (ops_out) {$\w_{out} \cdot \bm g(y_a)$};

\node[neuron, xshift=195, yshift=0, at=(input)] (linear) {$\w_{lin} \cdot \z$};
\node[neuron, xshift=195, yshift=-60, at=(input)] (power) {$\prod_i \left|z_i\right|^{w_i}$};
\node[neuron, xshift=195, yshift=-120, at=(input)] (product) {$\prod_i v_i$};

\node[data, xshift=290, yshift=0, at=(input)] (output) {$\z^{(k,l)}$};

\draw[connect] (input.east) -- (ops_in);

\draw[connect] (ops_in) -- (sin);
\draw[connect] (ops_in) -- (exp);
\draw[connect] (ops_in) -- (sgn);

\draw[connect] (sin) -- (ops_out);
\draw[connect] (exp) -- (ops_out);
\draw[connect] (sgn) -- (ops_out);

\draw[connect] (input.east) -- (product);
\draw[connect] (input.east) -- (power);
\draw[connect] (input.east) -- (linear);

\coordinate (output_merge) at ([xshift=-12]output.west);

\draw[connect arrow] (ops_out)  .. controls (8, 1.5) and (8.1,0) .. (output_merge) -- (output.west);
\draw[connect arrow] (linear) -- (output.west);
\draw[connect arrow] (power) .. controls (8, -1.5) and (8.1,0) .. (output_merge) -- (output.west);
\draw[connect arrow] (product) .. controls (8, -3.3) and (8.1,0) .. (output_merge) -- (output.west);

\draw[rounded corners, myred, very thick, dashed] (1.25, -5.5) rectangle (9, 6.5) {};

\node[annot, myred, left = 1.25, align=center, at=(exp)] {$l^{\text{th}}$ operational layer in $k^{\text{th}}$ stack};
\node[annot, above = 1, align= center, at=(input)] {Input from all $L$ layers of $(k-1)^{\text{th}}$ stack};
\node[annot, xshift = -7.5, yshift = -35, align = center, at = (input)] {$\R^{4L(k-1)+n+1}$};
\node[annot, above = 1, align= center, at=(output)] {Output of $l^{\text{th}}$ layer in $k^{\text{th}}$ stack};
\node[annot, xshift = 0, yshift = -35, align = center, at = (output)] {$\R^4$};

\end{tikzpicture}
        }
        \caption{Architecture of the $l^{\text{th}}$ operational layer in the $k^{\text{th}}$ stack.  This is a graphical representation of the primitive operations and functions defined by \cref{layer_form_1}-\cref{layer_form_4} in the main text.
        }
    \label{Layer_architecture}
\end{figure}
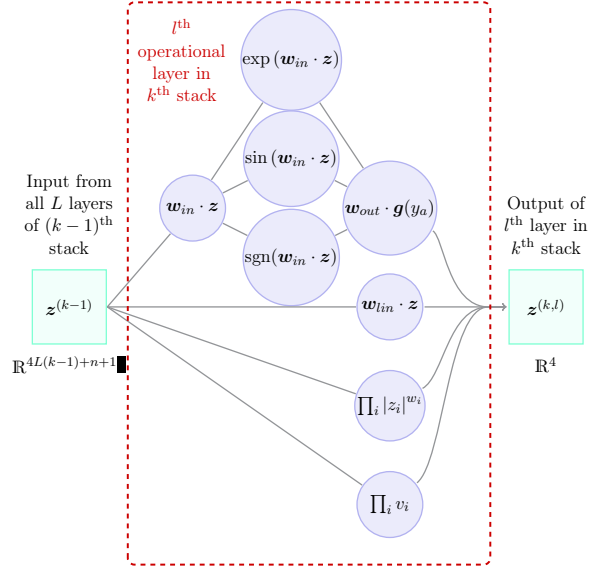
%We can now write down the output of each operational layer as a function. 
We begin by describing one of the operational layers in detail. As illustrated in ~\cref{Layer_architecture}, each operational layer $l$ in stack $k$ takes as input a vector $\bm z \in \R^{4L(k-1)+n+1}$, consisting of the state $\x \in \R^n$, an additional constant value of 2 to aid with scaling, and the outputs from the previous stacks which altogether are of dimensionality $4L(k-1)$.  The operational layer outputs a vector in $\R^4$. %While we showed \cref{linear_eq}-\cref{ops_eq} for each operational layer in the first stack as functions of the input for simplicity, they can be generalized to every operational layer as
The primitive functions that the operational layer generates are: 
\begin{subequations}
    \begin{align}
    f^{k,l}_1(\bm z) &= \bm w_{lin}^{k, l} \cdot \bm z,\label{layer_form_1}\\
    f^{k,l}_2(\bm z) &= \prod_{i = 1}^d | z_i |^{w_{pow, i}^{k,l}},\\
    f^{k,l}_3(\bm z) &= \prod_{i=1}^d [\sigma(w_{prod}^{k,l}) z_i + (1 - \sigma(w_{prod}^{k,l}))],\\
    f^{k,l}_4(\bm z) &= \bm w_{out}^{k,l} \cdot \begin{bmatrix}\exp(\bm w_{in}^{k,l} \cdot \bm z) \\ \sin(\bm w_{in}^{k,l} \cdot \bm z) \\ \text{sgn}(\bm w_{in}^{k,l} \cdot \bm z) \end{bmatrix}, \label{layer_form_4}
\end{align}
\end{subequations}
where $d = 4L(k-1)+n+1$, and the function $\sigma\left(w\right) = \frac{1}{1 + e^{-w}}$ is a sigmoidal function.  The various network weights $w$ will be learned from the data, which gives the interpretation that the primitive functions such as polynomials and $\sin(w x)$ can be viewed as activation functions, notably with global support.  These four functions \cref{layer_form_1}-\cref{layer_form_4} can be denoted as a vector $\bm f^{k,l}: \R^{4L(k-1)+n+1} \mapsto \R^4$ of functions:
\begin{equation}\label{layer_form}
    \bm f^{k,l}(\bm z) \equiv \begin{bmatrix}
        f^{k,1}_1 \\ f^{k,1}_2 \\ f^{k,1}_3 \\f^{k,1}_4
    \end{bmatrix}(\bm z).
\end{equation}
% \begin{equation}\label{layer_form}
%     f^{k,l} (\bm z) = \begin{bmatrix}
%         \bm w_{lin}^{k, l} \cdot \bm z \\ \prod_{i=1}^d \sigma[(w_{prod}^{k,l}) z_i + (1 - \sigma(w_{prod}^{k,l}))] \\ \prod_{i = 1}^d | z_i |^{w_{pow, i}^{k,l}} \\ \bm w_{out}^{k,l} \cdot \begin{bmatrix}
%             \exp(\bm w_{in}^{k,l} \cdot \bm z) \\ \sin(\bm w_{in}^{k,l} \cdot \bm z) \\ \text{sgn}(\bm w_{in}^{k,l} \cdot \bm z)
%         \end{bmatrix}
%     \end{bmatrix},
% \end{equation}

The operational layers for $k=1, \cdots K$ and $l = 1, \cdots L$ are put together into a NN as shown in \cref{model_architecture}.  This generates a NN model $\bm \ftilde(\x)$, which can be written as
\begin{equation}\label{NN_eqn}
    \bm \ftilde (\x) = \bm W_{out} \begin{bmatrix}
        \bm f^{K,1} \\ \bm f^{K,2} \\ \vdots \\ \bm f^{K, L} \\ \mathcal{I}
    \end{bmatrix} \circ \cdots \circ \begin{bmatrix}
        \bm f^{2,1} \\ \bm f^{2,2} \\ \vdots \\ \bm f^{2, L} \\ \mathcal{I}
    \end{bmatrix} \circ \begin{bmatrix}
        \bm f^{1,1} \\ \bm f^{1,2} \\ \vdots \\ \bm f^{1, L} \\ \mathcal{I}
    \end{bmatrix} \left( \begin{bmatrix}
        \bm x \\ 2
    \end{bmatrix} \right).
\end{equation}
Here, we can see the generative nature of our SymNN introduced in \cite{symanntex_paper}, where each operational layer $l$ in stack $k$, represented as $\bm f^{k,l}$ in \cref{layer_form}, is made up of primitives from \cref{layer_form_1}-\cref{layer_form_4}.

\begin{figure}[h]
    \centering
    \resizebox{\textwidth}{!}{%
            \begin{tikzpicture}

\node[state, xshift=-70] (input) {$\x$};

\node[data] (k_0) {$\begin{bmatrix} \bm x \\ 2 \end{bmatrix}$};

\draw[connect arrow] (input) -- (k_0);

\node[annot, above = 1, align = center, at = (input)] {Input\\$\R^n$};

\node[layer, xshift=80, yshift=-50, at=(k_0)] (k_1_l_1) {$l = L$};
% \node[layer, xshift=80, yshift=-50, at=(k_0)] (k_1_l_2) {$l = 2$};
% \node[layer, xshift=80, yshift=50, at=(k_0)] (k_1_l_L_1) {$l = L-1$};
\node[layer, xshift=80, yshift=50, at=(k_0)] (k_1_l_L) {$l = 1$};

\path (k_1_l_1) -- (k_1_l_L) node [myred, font=\Huge, midway, sloped] {$\dots$};

\draw[connect] (k_0.east) -- (k_1_l_1.west);
% \draw[connect] (k_0.east) -- (k_1_l_2.west);
% \draw[connect] (k_0.east) -- (k_1_l_L_1.west);
\draw[connect] (k_0.east) -- (k_1_l_L.west);

\node[data, xshift=160, yshift=0, at=(k_0)] (k_1) {$\z^{(k=1)}$};

\draw[rounded corners, myblue, very thick, dashed] (1.55, -3) rectangle (4, 3.5) {};

\coordinate (k_1_merge) at ([xshift=-5]k_1.west);
\coordinate (k_0_under) at (0, -3.5);

\coordinate (perp_point) at (k_0_under -| k_1_merge);
\draw[connect arrow] (k_0.south) -- (k_0_under) -- (perp_point) -- (k_1_merge);

\draw[connect arrow] (k_1_l_1.east)  .. controls (3.9, -1.5) and (4.1,0) .. (k_1_merge) -- (k_1.west);
% \draw[connect arrow] (k_1_l_2.east) .. controls (3.9, -1.5) and (4.1,0) .. (k_1_merge) -- (k_1.west);
% \draw[connect arrow] (k_1_l_L_1.east) .. controls (3.9, 1.5) and (4.1,0) .. (k_1_merge) -- (k_1.west);
\draw[connect arrow] (k_1_l_L.east) .. controls (3.9, 1.5) and (4.1,0) .. (k_1_merge) -- (k_1.west);

\node[annot, above = 1, align= center, at=(k_0)] {$\R^{n+1}$};

\node[annot2, myblue, above = 1, align= center, at=(k_1_l_L)] {Stack $k = 1$};

\node[annot, above = 1, align = center, at = (k_1)] {$\R^{4L+n+1}$};

%stack 2
\node[layer, xshift=80, yshift=-50, at=(k_1)] (k_2_l_1) {$l = L$};

\node[layer, xshift=80, yshift=50, at=(k_1)] (k_2_l_L) {$l = 1$};

\path (k_2_l_1) -- (k_2_l_L) node [myred, font=\Huge, midway, sloped] {$\dots$};

\draw[connect] (k_1.east) -- (k_2_l_1.west);

\draw[connect] (k_1.east) -- (k_2_l_L.west);

\node[annot2, myblue, above = 1, align= center, at=(k_2_l_L)] {Stack $k = 2$};

\draw[rounded corners, myblue, very thick, dashed, xshift= 160] (1.55, -3) rectangle (4, 3.5) {};

\coordinate (k_1_merge) at ([xshift=-5]k_1.west);

\node[data, draw = black!0, fill=black!0, xshift=320, yshift=0, at=(k_0)] (k_dots) {\Huge \textcolor{myblue}{$\cdots$}};

\coordinate (k_dots_merge) at ([xshift=-5]k_dots.west);

\coordinate (perp_point_dots) at (k_0_under -| k_dots_merge);
\draw[connect arrow] (k_1.south) -- ([xshift=160]k_0_under) -- (perp_point_dots) -- (k_dots_merge);

\draw[connect arrow] (k_2_l_1.east)  .. controls ([xshift=160]3.9, -1.5) and ([xshift=160]4.1,0) .. (k_dots_merge) -- (k_dots.west);

\draw[connect] (k_2_l_L.east) .. controls ([xshift=160]3.9,1.5) and ([xshift=160]4.1,0) .. (k_dots_merge) -- (k_dots.west);

%%%% last layer
\node[layer, xshift=80, yshift=-50, at=(k_dots)] (k_K_l_1) {$l = L$};

\node[layer, xshift=80, yshift=50, at=(k_dots)] (k_K_l_L) {$l = 1$};

\path (k_K_l_1) -- (k_K_l_L) node [myred, font=\Huge, midway, sloped] {$\dots$};

\draw[connect] (k_dots.east) -- (k_K_l_1.west);

\draw[connect] (k_dots.east) -- (k_K_l_L.west);

\node[annot2, myblue, above = 1, align= center, at=(k_K_l_L)] {Stack $k = K$};

\draw[rounded corners, myblue, very thick, dashed, xshift= 320] (1.55, -3) rectangle (4, 3.5) {};

\coordinate (k_dots_merge) at ([xshift=-5]k_dots.west);

\node[data, xshift=170, yshift=0, at=(k_dots)] (k_K) {$\z^{(k=K)}$};

\coordinate (k_K_merge) at ([xshift=-15]k_K.west);

\coordinate (perp_point_K) at (k_0_under -| k_K_merge);
\draw[connect arrow] (k_dots.east) -- ([xshift=343]k_0_under) -- (perp_point_K) -- (k_K_merge);

\draw[connect arrow] (k_K_l_1.east)  .. controls ([xshift=320]3.9, -1.5) and ([xshift=320]4.1,0) .. (k_K_merge) -- (k_K.west);

\draw[connect arrow] (k_K_l_L.east) .. controls ([xshift=320]3.9, 1.5) and ([xshift=320]4.1,0) .. (k_K_merge) -- (k_K.west);

\node[annot, above = 1, align = center, at = (k_K)] {$\R^{4LK+n+1}$};

\node[neuron, xshift=70, at=(k_K)] (dense_out) {$\ftilde(\x)$};

\draw[connect arrow] (k_K) -- (dense_out);

\node[annot, above = 1, align = center, at = (dense_out)] {Output\\$\R^n$};

\node[annot2, magenta, above = 2, align= center, at=(k_2_l_L)] {SymANNTEx};

\draw[rounded corners, magenta, ultra thick, dashed, xshift= 160] (-6.8, -4) rectangle (12.75, 3.75) {};
% \node[data, yshift=-70, at=(dense_out)] (output) {$\dot{\bm x}_\sigma$};
% \draw[connect arrow] (output) -- (dense_out);
% \node[annot, left = 0.7, align= center, at=(output)] {Training data $\R^n$};

\end{tikzpicture}
        }
        \caption{Architecture of our SymNN, showing $K$ stacks with $L$ layers within each stack.}
    \label{model_architecture}
\end{figure}
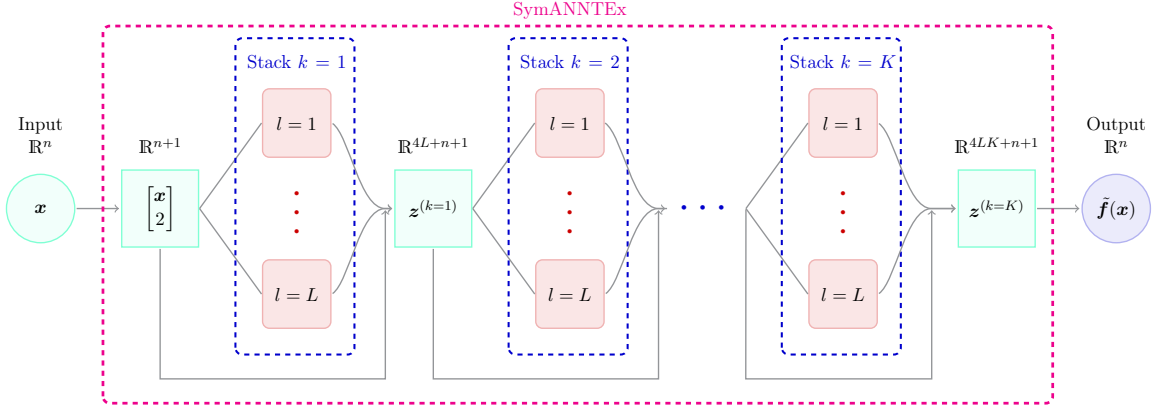

\subsection{Neural ODE approach}

We previously applied our SymNN to data from vector fields and discrete-maps~\cite{symanntex_paper, iyer2024expressive, wccm2024}, and we extend it to state data from continuous-time systems in this paper. To extend the above as a neural ODE formulation, we numerically integrate the SymNN -- thus SymNODE. We use a simple fixed timestep Runge-Kutta numerical scheme, as was done in \cite{fronk23}. This is a fourth-order explicit scheme, i.e., the numerically integrated solution can be explicitly written as a function of the state. Starting with state data $\X_0$, the estimated state $\xtilde$ at the next time step can be explicitly written as
\begin{equation}\label{RK4}
    \begin{split}
        \xtilde[h+1] &= \xtilde[h] + \frac{1}{6}(\bm k_1 + 2 \bm k_2 + 2 \bm k_3 + \bm k_4) \equiv \Stilde(\xtilde[h]),\\
        \text{where}&\\
        \bm k_1 = \delt \ftilde(\xtilde[h]), \ \bm k_2 &= \delt \ftilde(\xtilde[h] + \bm k_1/2), \ \bm k_3 = \delt \ftilde(\xtilde[h] + \bm k_2/2), \ \bm k_4 = \delt \ftilde(\xtilde[h] + \bm k_3).\\
        \text{with }&\text{initial condition} \ \xtilde[0] = \x[0].
    \end{split}
\end{equation}
Here, $\x[0]$ refers to each point in the dataset $\X_0$ which has subsequent time-series data available as training data. 
% Then the MAE after one time-step with respect to the true data $\X_1$ is
% \begin{equation*}\label{error_1}
%     \text{MAE}_1 = \sum_{\x[1]\in\X_1}  \|\x[1] - \xtilde[1]\|_1,
% \end{equation*}
By using the above, we numerically estimate the state after $h$ time-steps, denoted as $\xtilde[h]$, from $\x[0]$ as  shown schematically in \cref{neural_ode}. The Mean Absolute Error (MAE), after $h$ time-steps, between the training data $\x[h]$ and estimated $\xtilde[h]$ computed from the $m$ data points $\x[0] \in \X_0$ is
% , numerical integration in \cref{RK4} is performed \textit{recursively} to give $\xtilde[h]$ and the MAE is
\begin{equation}\label{MAE}
    \text{MAE}_h = \frac{1}{m}\sum_{\substack{\x[0] \in \X_0}}  \|\Stilde^h(\x[0]) - \S^h(\x[0])\|_1.
\end{equation}
The error $\text{MAE}_H$ at the end of the forecasting horizon $H$ could be minimized, but that can be achieved even by models that don't track the trajectory until $H$.  Thus, we seek to minimize the average MAE over the entire forecasting horizon.%, which we call the trajectory MAE:

\begin{figure}[h]
    \centering
    \resizebox{0.875\columnwidth}{!}{%
            \begin{tikzpicture}

\node[state, xshift=-90] (input) {$\x[0]$};

\node[annot, below = 0.75, align = center, at = (input)]{$\forall \x[0] \in \X_0$};

\node[NN] (NN) {$~$ SymANNTEx $~$};

\draw[connect arrow] (input) -- (NN);

% \node[derivative, xshift=90, at=(NN)] (f_out) {$\ftilde(\x)$};

% \draw[active arrow] (NN) -- (f_out);

\node[ODE, xshift=90, at=(NN)] (solver) {RK4};

\draw[active arrow] (NN) -- (solver);

\node[derivative, xshift=70, at=(solver)] (output) {$\tilde \x[h]$};

\draw[active arrow](solver) -- (output);

\coordinate (NN_under) at (0, -2);
\coordinate (output_under) at ([yshift=-36.5]output.south);
\draw[active arrow] (output.south) -- (output_under) -- (NN_under) -- (NN.south);

\node[annot, blue, above = 0.75, align = center, at = (output)]{$h = 1,2, \cdots, H$};

\node[state, xshift = 70, at = (output)] (ref_data) {$\x[h]$};

\node[annot, above = 0.05, xshift = -36, align = center, at = (ref_data)]{\textcolor{red}{Loss}};

\node[annot2, below = 0.75, align = center, at = (ref_data)]{$\forall \x[h] \in \X_h$};

\draw[bi arrow](output.east) -- (ref_data);

\end{tikzpicture}
        }
        \caption{Schematic of the prediction of the future states $\x[h]$ (far right teal circle) from the initial states $\x[0]$ (far left teal circle) using the SymNN from \cref{model_architecture}. The blue arrows denote the forward computation of the predicted states $\xtilde[h]$ \textit{up to} $H$ times (purple circle). The closed loop denotes a recurrence relationship forward in time and the gradients of the loss (red arrow) are propagated backwards in time.} 
    \label{neural_ode}
\end{figure}
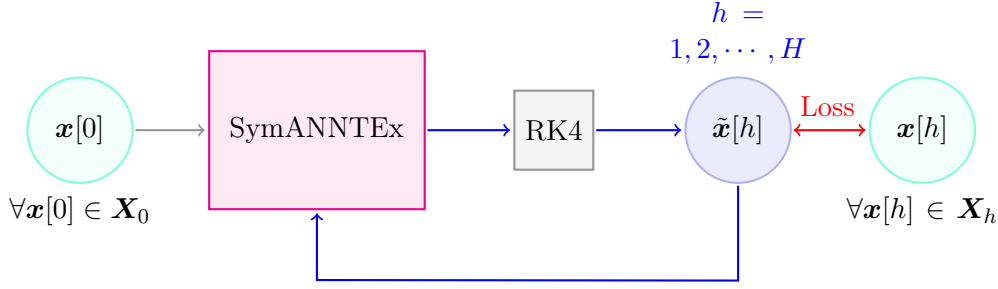

\subsection{Regularization}
%In this work, we use what has been called the smooth-$\lhalf$ regularization.
 Regularization is often used in machine learning to facilitate generalization of the learned model. This is the case in conventional NNs with localized activation functions where the network of weights can only be tailored to the training dataset, which can lead to overfitting. In contrast for our algorithm, which uses activation functions with global support, the purpose of regularization is to promote sparsity of weights. The most widely-used sparsity-promoting regularization in the literature is $\ell_1$ regularization:
\begin{equation}\label{l_1}
    \ell_1 = \sum_i |w_i|.
\end{equation}
Although we found this useful for smaller NNs, say, $L=3$ layers wide, such as those used for multi-step predictions, we found that this falls short of desired sparsity in larger NNs such as $L=10$ layers. In our previous work~\cite{symanntex_paper}, we used the $\lhalf$ regularization~\cite{l_half_regularization}, which provided much sparser results than $\ell_1$ regularization.  However, the gradients of the $\lhalf$ regularizer for small values can be very large. This is alleviated by smoothly increasing the value of regularization for weights close to zero, as is done in smooth-$\lhalf$ regularization~\cite{l_half_smooth_regularization}:

\begin{equation}\label{l_half}
\lhalf \equiv \sum_i 
\begin{cases}
\left|w_i\right|^{\nicefrac{1}{2}} \quad \text{ if } \left|w_i\right| \geq \epsilon\\
\left| -\frac{1}{8\epsilon^3}w_i^4 + \frac{3}{4\epsilon}w_i^2 + \frac{3\epsilon}{8}\right|^{\nicefrac{1}{2}} \quad \text{ if } \left|w_i\right| < \epsilon,
\end{cases}
\end{equation}
where $\epsilon = 0.01$ is the threshold for smoothing.  In this work, we will use smooth-$\lhalf$ regularization wherever the $\ell_1$ regularizer does not give desired sparsity.

\subsection{Optimization}
The loss function to be minimized is the average MAE~\cref{MAE} over the NODE horizon of $H$ time-steps, along with the sparsity promoting regularizer~\cref{l_half}:
\begin{equation} \label{loss}
    % \text{Loss } = \text{TMAE} + \alpha \lhalf,
    \text{Loss}_H = \frac{1}{H}\sum_{h=1}^H \text{MAE}_h + \alpha \lhalf
\end{equation}
where $\alpha$ is a hyperparameter that weights the relative importance of the regularizer.  
% In this work, we use $\alpha = \{0.01, 0.02, 0.04\}$.  
Note that we use mean absolute error (MAE) rather than mean squared error (MSE) in the definition of the loss. In a multi-step setting, small one-step errors can accumulate under repeated composition, so it is important to control even small deviations. Unlike MSE, which downweights small errors due to its quadratic scaling, MAE maintains sensitivity to these errors and therefore provides better control over long-horizon predictions. In addition, MAE is more robust to noise and outliers, and we have found empirically that it leads to more accurate and stable model identification.

%We use Adam with a learning rate scheduler. 

We use the discretize-then-optimize approach rather than the optimize-then-discretize approach \cite{discretize-optimize}, also known as the adjoint method found in \cite{neural_odes}. This is mainly done because for the variety of time-scales that we deal with, spanning fast-slow systems and particularly the small time-steps in chaotic systems, the adjoint method is known to be sensitive. And given the small memory requirements of our NN, there is little to be gained by using the adjoint method \cite{discretize-optimize}, \cite{colby_2}. We use the Adam optimizer~\cite{kingma2015adam} to minimize the loss function \cref{loss}. Its adaptive, parameter-wise step-size is well suited for effectively navigating non-convex loss landscapes encountered in deep learning. We found that the learning rate hyperparameter has the most significant impact on our use cases. However, we could not find a specific learning rate that works well for \textit{all} examples throughout the full training so we use a learning rate scheduler to set up varying learning rates, without feedback. In this work, we use a triangular learning rate scheduler. Some miscellaneous details on implementation are given in \cref{appendix:compute}.

\subsection{Effect of Optimization Horizon on Learned Dynamics} \label{sec:horizon}

We study the effect of the NODE horizon $H$ on the learned dynamics $\tilde f$ obtained by minimizing the loss \cref{loss}, which depends on the Mean Absolute Errors \cref{MAE}.  For each $H$, let
\begin{equation}
\tilde \f_H = \arg\min_{\tilde \f} \mathrm{Loss}_H(\tilde \f),
\label{fH}
\end{equation}
and denote by $\tilde S_H$ the corresponding learned stroboscopic map.

\begin{proposition}\label{optimality_horizon}
\textbf{(Optimality under extended horizon)}
Let $\tilde \f_H$ and $\tilde \f_{H+1}$ be minimizers of $\mathrm{Loss}_H$ and $\mathrm{Loss}_{H+1}$, respectively. Then
\begin{equation}
\mathrm{Loss}_{H+1}(\tilde \f_{H+1})
\le
\mathrm{Loss}_{H+1}(\tilde \f_H).
\end{equation}
\end{proposition}
\begin{proof}
For any admissible $\tilde \f$,
\begin{equation}
\mathrm{Loss}_{H+1}(\tilde \f)
=
\frac{1}{H+1}\sum_{h=1}^{H+1} \mathrm{MAE}_h(\tilde \f)
+ \alpha \lhalf(\tilde \f).
\end{equation}
Since $\tilde f_{H+1}$ minimizes $\mathrm{Loss}_{H+1}$, the result follows by evaluating at $\tilde \f = \tilde \f_H$.
\end{proof}

\noindent
Note that it is {\it not} generally true that 
$\mathrm{Loss}_{H+1}(\tilde \f_{H+1}) \le \mathrm{Loss}_{H}(\tilde \f_H)$.  This is because $\mathrm{Loss}_H$ and $\mathrm{Loss}_{H+1}$ are different objective functions that involve a different number of terms and different normalization factors.

\medskip

\cref{optimality_horizon} shows that optimizing over a longer NODE horizon yields a model that is optimal when evaluated over that same horizon. We now provide additional insight into why increasing the horizon improves the quality of the learned dynamics.  Recall from \cref{loss} that the loss is defined as an average over prediction steps.
As a result, the loss reflects the \emph{typical} prediction error per time step. However, this averaging does not eliminate sensitivity to error growth over time. In particular, if the trajectory error increases with $h$, then the later terms in the sum contribute increasingly large values, leading to a higher overall loss.  Consequently, increasing the NODE horizon $H$ makes the optimization problem more sensitive to the long-term behavior of the model. Models whose errors grow over time (for example, due to unstable dynamics) incur a larger loss when evaluated over longer horizons, while models with bounded errors remain competitive.
Therefore, in addition to incorporating more data, increasing the horizon implicitly favors models that exhibit stable long-term behavior. This provides an additional explanation for why $\tilde \f_{H+1}$ often yields improved performance compared to $\tilde \f_H$, beyond the comparison given by \cref{optimality_horizon}.

\begin{proposition}\label{accumulating_error}
\textbf{(Accumulating prediction errors increase the multi-step loss)}
Let $\mathrm{MAE}_h$ and $\mathrm{Loss}_H$ be defined as in
\cref{MAE} and \cref{loss}. Suppose that
\[
\mathrm{MAE}_{H+1}
>
\frac1H \sum_{h=1}^H \mathrm{MAE}_h.
\]
Then
\[
\mathrm{Loss}_{H+1}
>
\mathrm{Loss}_H.
\]
In particular, if the sequence $\{\mathrm{MAE}_h\}$ is increasing,
then extending the optimization horizon increases the loss.
\end{proposition}

\begin{proof}
For a fixed model $\ftilde$, the regularization term
$\alpha \ell_{1/2}(\ftilde)$ is independent of $H$.
Thus,
\[
\mathrm{Loss}_H
=
\frac1H\sum_{h=1}^H \mathrm{MAE}_h
+
\alpha \ell_{1/2}(\ftilde).
\]
Therefore,
\[
\mathrm{Loss}_{H+1}-\mathrm{Loss}_H
=
\frac1{H+1}\sum_{h=1}^{H+1}\mathrm{MAE}_h
-
\frac1H\sum_{h=1}^H\mathrm{MAE}_h.
\]
Rearranging gives
\[
\mathrm{Loss}_{H+1}-\mathrm{Loss}_H
=
\frac{
\mathrm{MAE}_{H+1}
-
\frac1H\sum_{h=1}^H \mathrm{MAE}_h
}{H+1}.
\]
Hence,
\[
\mathrm{Loss}_{H+1}>\mathrm{Loss}_H
\iff
\mathrm{MAE}_{H+1}
>
\frac1H\sum_{h=1}^H \mathrm{MAE}_h.
\]
If $\{\mathrm{MAE}_h\}$ is increasing, then the newest term exceeds
the average of the previous terms, proving the result.
\end{proof}

\begin{remark}
\cref{accumulating_error} shows that when multi-step prediction errors grow with the forecasting horizon, the corresponding multi-step training loss also increases with the horizon.  Consequently, during optimization, models whose errors compound under repeated composition become increasingly unfavorable as the prediction horizon grows.
As discussed heuristically in
\cref{appendix:trajectory_tracking}, repeated composition of an
imperfect learned flow map can amplify local modeling errors over
time, particularly in unstable or chaotic systems where nearby
trajectories separate rapidly. Minimizing $\mathrm{Loss}_H$ over
longer horizons therefore biases training toward models whose
prediction errors remain controlled under iteration, rather than
models that exhibit rapidly accumulating trajectory error.
\end{remark}

\begin{remark}
By \cref{l2_norm_eqv} below, control of the averaged
trajectory MAE also provides control of the corresponding averaged trajectory MSE over the training horizon. Consequently,
models whose multi-step prediction errors compound over time incur
increasing penalties as the horizon $H$ is extended. This provides
additional motivation for minimizing trajectory errors over multiple
time steps rather than only matching the vector field locally.
\end{remark}

As an illustrative example, consider a linear system $\f(\x) = U\x$ where $U$ has purely imaginary eigenvalues. Then the true dynamics are neutrally stable. If we consider perturbed models of the form $(U \pm \epsilon I)\x$, the model with positive $\epsilon$ exhibits exponential growth in trajectories, while the model with negative $\epsilon$ exhibits decay. Over longer horizons, the growing trajectories lead to rapidly increasing prediction errors, and hence a larger averaged loss. In contrast, the decaying model produces bounded errors and a smaller loss. This illustrates how longer horizons penalize unstable models more strongly.

\cref{optimality_horizon} showed that increasing the optimization horizon improves performance when evaluated over that same horizon. We now complement this result by showing that, for a fixed horizon $H$, minimizing $\mathrm{Loss}_H$ also provides control over the trajectory error of the learned model.  While the loss is defined using $\ell^1$ (mean absolute) errors, it is often more natural to measure trajectory accuracy using the $\ell^2$ norm. The following result shows that minimizing $\mathrm{Loss}_H$ also controls the averaged $\ell^2$ trajectory error.

\begin{proposition}\label{l2_norm_eqv}
\textbf{(Control of $\ell^2$ trajectory error via norm equivalence)}
Define $\tilde \f_H$ as in \cref{fH}, and let $\tilde \S_H$ be the corresponding learned stroboscopic map. Assume that the true and predicted trajectories remain uniformly bounded over the training horizon, i.e., there exists constants $M > 0$ and $\tilde{M} > 0$ such that
\begin{equation}
\|\S^h(\x[0])\|_2 \le M,
\qquad
\|\tilde \S_H^h(\x[0])\|_2 \le \tilde{M}
\end{equation}
for all $\x[0] \in \X_0$ and $h = 1,\dots,H$. Such bounds are natural when trajectories remain in a compact region of state space over the training horizon. Then there exists a constant $C = (M + \tilde{M}) > 0$, such that
\begin{equation}
\frac{1}{mH}
\sum_{h=1}^H
\sum_{\x[0] \in \X_0}
\|\tilde \S_H^h(\x[0]) - \S^h(\x[0])\|_2^2
\;\le\;
C \cdot \mathrm{Loss}_H(\tilde \f_H).
\end{equation}
\end{proposition}

\begin{proof}
\textit{Relating the loss to MAE terms.}
\begin{equation}
\mathrm{Loss}_H(\tilde \f_H)
=
\frac{1}{H}\sum_{h=1}^H \mathrm{MAE}_h(\tilde \f_H)
+ \alpha \lhalf(\tilde \f_H)
\;\ge\;
\frac{1}{H}\sum_{h=1}^H \mathrm{MAE}_h(\tilde \f_H).
\end{equation}

\medskip
\noindent
\textit{Expanding MAE.}

\begin{equation}
\mathrm{MAE}_h(\tilde \f_H)
=
\frac{1}{m}
\sum_{\x[0] \in \X_0}
\|\tilde \S_H^h(\x[0]) - \S^h(\x[0])\|_1.
\end{equation}
Thus,
\begin{equation}
\frac{1}{H}\sum_{h=1}^H \mathrm{MAE}_h(\tilde \f_H)
=
\frac{1}{mH}
\sum_{h=1}^H
\sum_{\x[0] \in \X_0}
\|\tilde \S_H^h(\x[0]) - \S^h(\x[0])\|_1.
\end{equation}

\medskip
\noindent
\textit{Relating $\ell^2$ and $\ell^1$ errors.}

For any ${\bf v} \in \mathbb{R}^n$, from norm equivalence we have
\begin{equation}
\|{\bf v}\|_2 \le \|{\bf v}\|_1.
\end{equation}
Therefore,
\begin{equation}
\|{\bf v}\|_2^2 \le \|{\bf v}\|_2 \cdot \|{\bf v}\|_1.
\end{equation}
Let
\[
{\bf v}
=
\tilde \S_H^h(\x[0]) - \S^h(\x[0]).
\]
By the triangle inequality and the uniform boundedness assumption,
\begin{equation}
\|{\bf v}\|_2
\le
\|\tilde \S_H^h(\x[0])\|_2
+
\|\S^h(\x[0])\|_2
\le
M + \tilde{M}.
\end{equation}
Hence,
\begin{equation}
\|{\bf v}\|_2^2 \le (M+\tilde{M}) \|{\bf v}\|_1.
\end{equation}
Applying this pointwise to the trajectory errors and summing over all samples,
\begin{equation}
\sum_{h=1}^H
\sum_{\x[0] \in \X_0}
\|\tilde \S_H^h(\x[0]) - \S^h(\x[0])\|_2^2
\;\le\;
(M + \tilde{M})
\sum_{h=1}^H
\sum_{\x[0] \in \X_0}
\|\tilde \S_H^h(\x[0]) - \S^h(\x[0])\|_1.
\end{equation}

\medskip
\noindent
\textit{To conclude,}

Dividing both sides by $mH$, we obtain
\begin{equation}
\frac{1}{mH}
\sum_{h=1}^H
\sum_{\x[0] \in \X_0}
\|\tilde \S_H^h(\x[0]) - \S^h(\x[0])\|_2^2
\;\le\;
(M + \tilde{M})
\frac{1}{H}\sum_{h=1}^H \mathrm{MAE}_h(\tilde \f_H).
\end{equation}
Using Step 1, this implies
\begin{equation}
\frac{1}{mH}
\sum_{h=1}^H
\sum_{\x[0] \in \X_0}
\|\tilde \S_H^h(\x[0]) - \S^h(\x[0])\|_2^2
\;\le\;
C \cdot \mathrm{Loss}_H(\tilde \f_H),
\end{equation}
where $C = (M + \tilde{M})$.
\end{proof}

% {\color{blue}
\begin{remark}
The proof of \cref{l2_norm_eqv} uses the uniform trajectory bounds
$M$ and $\tilde M$ to obtain the estimate
\[
\|{\bf v}\|_2^2 \le (M+\tilde M)\|{\bf v}\|_1,
\]
where ${\bf v} = \tilde \S_H^h(\x[0])-\S^h(\x[0])$.
Alternatively, if one has access to the maximum trajectory error
\[
M_\infty = \max_{\x[0]\in\X_0} \max_{1\le h\le H} \| \tilde \S_H^h(\x[0]) - \S^h(\x[0]) \|_\infty,
\]
then Hölder's inequality gives
\[
\|{\bf v}\|_2^2 \le \|{\bf v}\|_\infty \|{\bf v}\|_1 \le M_\infty \|{\bf v}\|_1.
\]
Repeating the proof yields the alternative bound
\[
\frac{1}{mH} \sum_{h=1}^H \sum_{\x[0]\in\X_0} \| \tilde \S_H^h(\x[0]) - \S^h(\x[0]) \|_2^2 \le M_\infty \,\mathrm{Loss}_H(\tilde \f_H).
\]
Since $M_\infty$ measures the largest trajectory error over the training
horizon, this bound may be sharper than that obtained from
$M+\tilde M$ when the learned and true trajectories remain close.
\end{remark}
% }

\noindent
\cref{optimality_horizon} establishes that increasing the optimization horizon yields a model that is optimal for longer-term prediction. \cref{accumulating_error} further shows that extending the NODE horizon increasingly penalizes models whose trajectory errors compound over time, thereby favoring models that remain stable under repeated composition. Meanwhile, \cref{l2_norm_eqv} shows that, for any fixed horizon $H$, minimizing $\mathrm{Loss}_H$ ensures that the learned dynamics accurately reproduce trajectories over that horizon, as measured by the average squared $\ell^2$ error. In other words, minimizing $\mathrm{Loss}_H$ controls the average multi-step trajectory error over the training horizon.
Together, these results provide a theoretical justification for multi-step training: increasing the NODE horizon discourages unstable error accumulation and promotes models that remain accurate over longer time intervals, while minimizing the corresponding loss guarantees accurate trajectory reconstruction over the training data.

Although the intention of our NODE framework is to give accurate analytic models, often the resulting models will not be exact, in particular when the data is noisy. \cref{sec:bounds} shows that for systems with stable fixed points and stable periodic orbits, in the asymptotic limit the solutions of approximately accurate models will stay within bounds of the stable solutions for the exact models. Moreover, in the chaotic examples below, we will find that (i) \cref{l2_norm_eqv} explains short-term trajectory agreement, (ii) sensitivity to initial conditions explains eventual divergence, and (iii) the results from \cref{sec:bounds_statistics}, which show that pointwise trajectory error provides short-time control over statistical properties of the solutions, a step toward explaining why long-term statistics remain accurate despite trajectory mismatch.

\subsection{Modified Akaike Information Criterion}

The theoretical results above characterize properties of models that minimize the training loss $\mathrm{Loss}_H$. Post-training, we distill models by considering an additional trade-off between goodness-of-fit and model complexity. To do this, we modify the Akaike Information Criterion (AIC) \cite{akaike1974new} to consider numerical integration of the estimated model, over the horizon $H$ -- thus a modified AIC (mAIC):
\begin{equation}\label{aic}
    \text{mAIC} =  \left[ 2k + m\log_e\Big(\frac{1}{H}\sum_{h=1}^H \text{MSE}_h\Big) \right] + 2\frac{(k+1)(k+2)}{m - k - 2},
\end{equation}
where
\begin{equation}\label{MSE}
    \text{MSE}_h = \frac{1}{m}\sum_{\substack{\x[0] \in \X_0}}  \|\Stilde^h(\x[0]) - \S^h(\x[0])\|_2^2.
\end{equation}
This is not the same as the standard AIC~\cite{akaike1974new} because our metric is different. Post-training, we truncate parameters of the trained model $\ftilde$ to each tolerance in $[0.001, 0.01, 0.1, 1]$ using Sympy's \verb|nsimplify| and distill sparse models $\ftilde_s$. We then select the distilled model that minimizes mAIC which is averaged over a horizon that could in practice be taken to be \textit{longer} in post-training than the horizon used in training:
\begin{equation}
    \ftilde_A = \argmin \text{mAIC}(\ftilde_s)
\end{equation}
While the mAIC is not identical to the training loss, it is constructed from the averaged mean squared error over the NODE horizon, and therefore preserves the key dependence on long-term prediction accuracy. As a result, the qualitative conclusions regarding the effect of increasing the horizon -- namely, the preference for models with stable long-term behavior -- continue to hold in practice, although the theoretical guarantees no longer apply exactly.

\section{Examples}\label{sec:Examples}

\subsection{Noisy data}\label{noise}

To test our algorithm on datasets with noise, we corrupt the state data from simulated deterministic dynamical systems with random disturbances. Adding an independent disturbance drawn from a normal distribution could impact each of our examples and their states differently based on their dynamic range (e.g. Chua double scroll in~\cref{example:chua} vs Lorenz equations in~\cref{example:lorenz} or $x$ vs $y$ in Van der Pol oscillator in~\cref{example:van_der_pol}). For consistency, we define our disturbance using a state-wise Noise-to-Signal Ratio (NSR):
\begin{equation}\label{noise_relative}
    x_i + x^{rms}_i \eta_i,
\end{equation}
which scales the noise added to the $i^{th}$ state with an independently drawn $\eta_i \in \mathcal{N}(0, \sigma^2)$ for each component $x_i$ of each data point $\x \in \X$. Here, $x_i^{rms}$ is the Root Mean Square (RMS) of $x_i$ in the available data:
\begin{equation}\label{rms}
x_i^{rms} = \sqrt{\frac{1}{m} \sum_{\x \in \X} x_i^2},
\end{equation}
which is the $2$-norm of the $i^{th}$ state averaged over the $m$ data points. In the following examples we consider $\sigma \in \{0, 0.01, 0.05, 0.1, 0.15, 0.2, 0.25, 0.3\}$, which when multiplied by 100 are denoted as ``$\%$ NSR''. 

For comparison, the more common alternative is adding an independently drawn disturbance $\eta^{abs}_i \in \mathcal{N}(0, \sigma_{abs}^2)$ is
 \begin{equation}\label{noise_absolute}
    x_i + \eta^{abs}_i.
\end{equation}

\subsection{Takens Bogdanov normal form}

We begin with the Takens--Bogdanov normal form \cite{guckenheimer2013nonlinear}, given by
\begin{subequations}
\begin{align}
    \dot{x} &= y, \label{tb_x}\\
    \dot{y} &= \mu_1 + \mu_2 \, y + x^2 + x \, y, \label{tb_y}
\end{align}
\end{subequations}
with parameters $\mu_1 = -4.41$ and $\mu_2 = 1.5$. For these values, the system exhibits a stable spiral sink at $(x,y) = (-\sqrt{-\mu_1},0) = (-2.1,0)$ and a saddle point at $(x,y) = (\sqrt{-\mu_1},0) = (2.1,0)$.  The parameters were chosen so that, to leading order, there is a homoclinic orbit connecting the saddle point to itself~\cite{guckenheimer2013nonlinear}, giving a homoclinic bifurcation:  for $\mu_1 \gtrsim -4.41$ there is an unstable periodic orbit, and for $\mu_1 \lesssim -4.41$ there are no periodic orbits.  Some trajectories will asymptotically approach the stable fixed point, while those outside its basin of attraction will escape to infinity. 
%As a result, the data is inherently non-stationary and sensitive to perturbations.
This example provides a useful test case for model identification from time-series data for systems that are sensitive to both state and parametric perturbations. 

We generate a dataset consisting of $m^* = 600$ data points sampled from $p = 4$ trajectories with randomly chosen initial conditions over the time interval $T = [0,4]$, and corrupt the data with $1\%$ NSR as described in \cref{noise}. The NN is $K=1$ stack deep and $L=10$ layers wide (236 parameters), and is trained using a NODE horizon of $H = 16$ time steps.
Using the proposed SymNODE framework, we obtain the identified model
\begin{subequations}
\begin{align}
    \dot{x} &= y, \label{tb_main_x}\\
    \dot{y} &= -4.5 + 1.5 \, y + x^2 + x \, y, \label{tb_main_y}
\end{align}
\end{subequations}
which closely matches the true system, with only a small discrepancy in the constant term.  We note that the identified model has a saddle fixed point at $(x,y) \approx (2.121,0)$ and a stable spiral sink at $(x,y) \approx (-2.121,0)$. The dynamics near this stable fixed point are expected to be consistent with the results from \cref{appendix:fixed_point}: asymptotically they will stay wihin a close neighborhood of the fixed point of the true system. 

A comparison between the true and identified dynamics is shown in \cref{tb_space}--\cref{tb_time}. In the phase portrait \cref{tb_space}, the training data (teal dots) lies primarily along trajectories influenced by the unstable manifold of the saddle point, with visible perturbations due to noise. The trajectory generated by the identified model (dashed red) closely follows that of the true system (solid blue) from the same initial condition. This agreement is also evident in the time-series comparison shown in \cref{tb_time}. This example demonstrates that the proposed approach can accurately recover governing equations from noisy data in systems which are sensitive to state and parametric perturbations.
In particular, the use of a multi-step prediction horizon enforces consistency under repeated composition, allowing the learned model to remain accurate even as trajectories evolve along unstable directions.

\begin{figure*}[h]
    \centering
    \begin{subfigure}[t]{0.5\textwidth}
        \centering
        \includegraphics[height=1.9in]{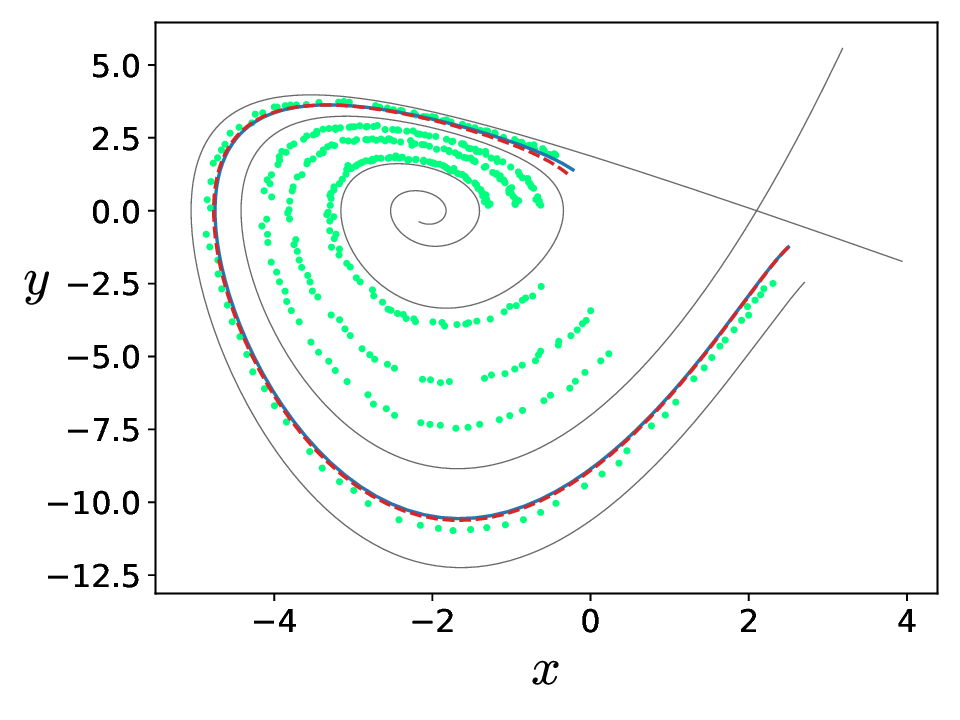}
        \caption{Space}
        \label{tb_space}
    \end{subfigure}%
    ~ 
    \begin{subfigure}[t]{0.49\textwidth}
        \centering
        \includegraphics[height=1.9in]{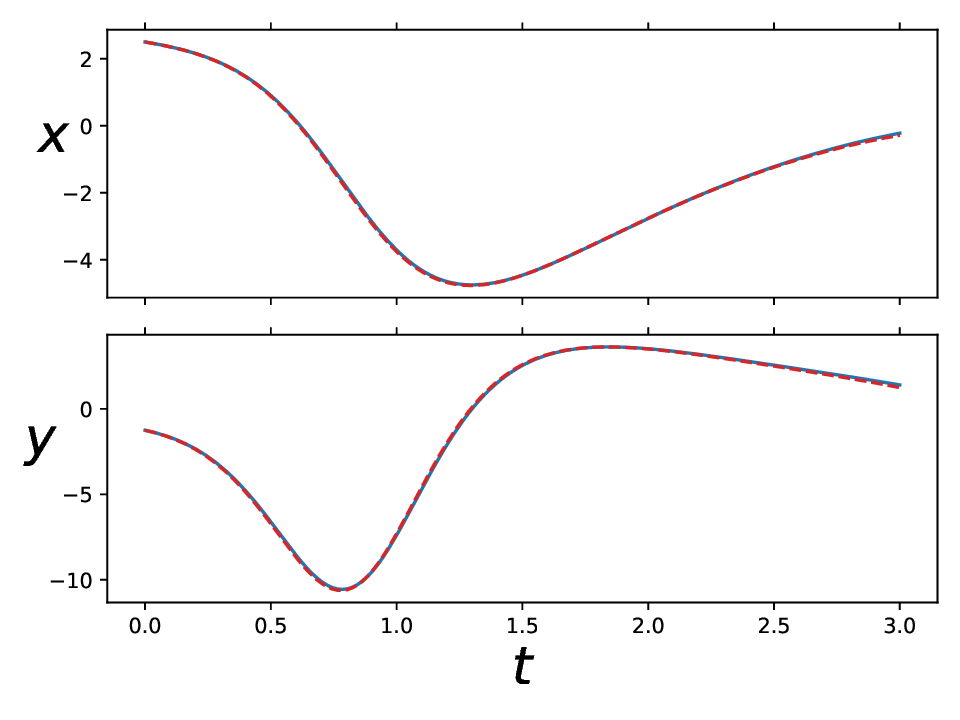}
        \caption{Time}
        \label{tb_time}
    \end{subfigure}
    \caption{Takens--Bogdanov normal form. (a) Phase portrait showing training data (teal), true trajectory (solid blue), and trajectory generated by the identified model (dashed red).  The data are predominantly influenced by the unstable manifold of the saddle point.  (b) Time-series comparison from the same initial condition.  The close agreement demonstrates accurate recovery of both the vector field and transient dynamics despite both state and parametric sensitivity.}
    \label{tb_fig}
\end{figure*}

\subsection{Hopf normal form}

We next consider the Hopf normal form in polar coordinates \cite{guckenheimer2013nonlinear}
\begin{subequations}
\begin{align}
    \dot{r} &= a r + b r^3, \label{hopf_r}\\
    \dot{\theta} &= c + d r^2, \label{hopf_theta}
\end{align}
\end{subequations}
with parameters $a = 1$, $b = -1$, $c = 1$, and $d = 0$. This system exhibits a stable limit cycle at $r = 1$, with trajectories spiraling toward the periodic orbit.

This example highlights a fundamental challenge in data-driven model discovery: \emph{non-uniqueness of models consistent with observed trajectories}. In particular, many dynamical systems can reproduce a circular limit cycle with approximately constant angular velocity. For example, linear systems with an elliptic fixed point generate periodic orbits, and may closely match observed trajectories when the data lies near the limit cycle. As a result, trajectory data alone - especially when dominated by post-transient behavior - may be insufficient to uniquely identify the governing equations.
To illustrate this, we generate data from $p = 5$ trajectories initialized outside the unit circle and sampled over the interval $T = [0,10]$, yielding $m^* = 600$ data points. The data is corrupted with $1\%$ NSR as described in \cref{noise}. The NN is $K=1$ stack deep and $L=10$ layers wide (236 parameters), and trained using a horizon of $H = 16$ time steps.

When the dataset contains sufficient transient behavior (i.e., radial decay toward the limit cycle), the proposed method successfully recovers the exact governing equations \cref{hopf_r}--\cref{hopf_theta} after $800$ epochs. A comparison between the true and identified dynamics is shown in \cref{hopf_space}-\cref{hopf_time}. In the phase portrait \cref{hopf_space}, the training data (teal dots) includes trajectories converging toward the limit cycle. The identified model reproduces both the transient dynamics and the asymptotic periodic orbit. The time-series representation \cref{hopf_time} shows exponential decay in $r$ and linear growth in $\theta$, both of which are captured accurately.
In contrast, when the data is restricted to trajectories near the limit cycle, we observe that simpler models, such as linear systems, can fit the data with low short-term error, but fail to capture the correct radial dynamics. This reflects the fact that the limit cycle alone does not uniquely determine the underlying vector field.

\begin{figure*}[h]
    \centering
    \begin{subfigure}[t]{0.49\textwidth}
        \centering
        \includegraphics[height=1.9in]{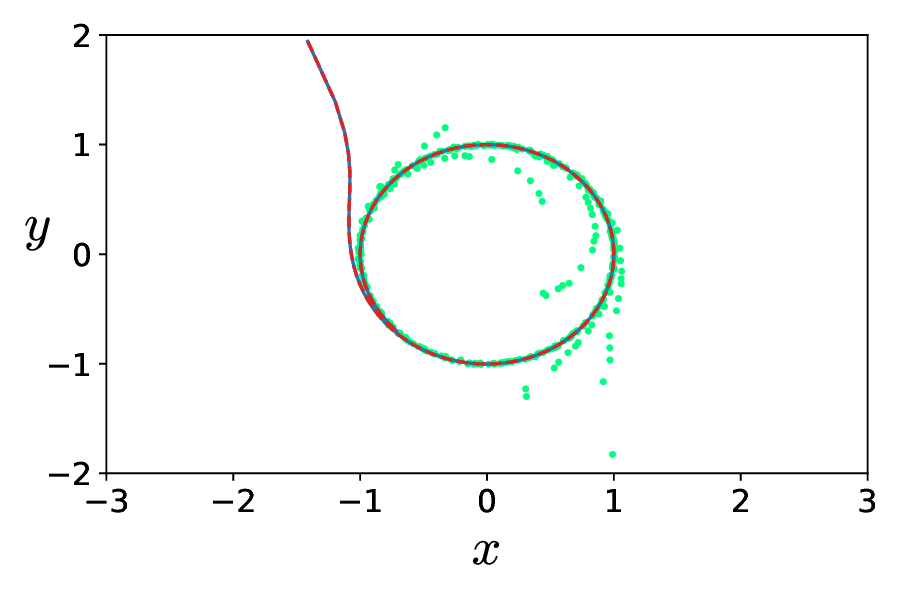}
        \caption{Space}
        \label{hopf_space}
    \end{subfigure}%
    ~ 
    \begin{subfigure}[t]{0.49\textwidth}
        \centering
        \includegraphics[height=1.9in]{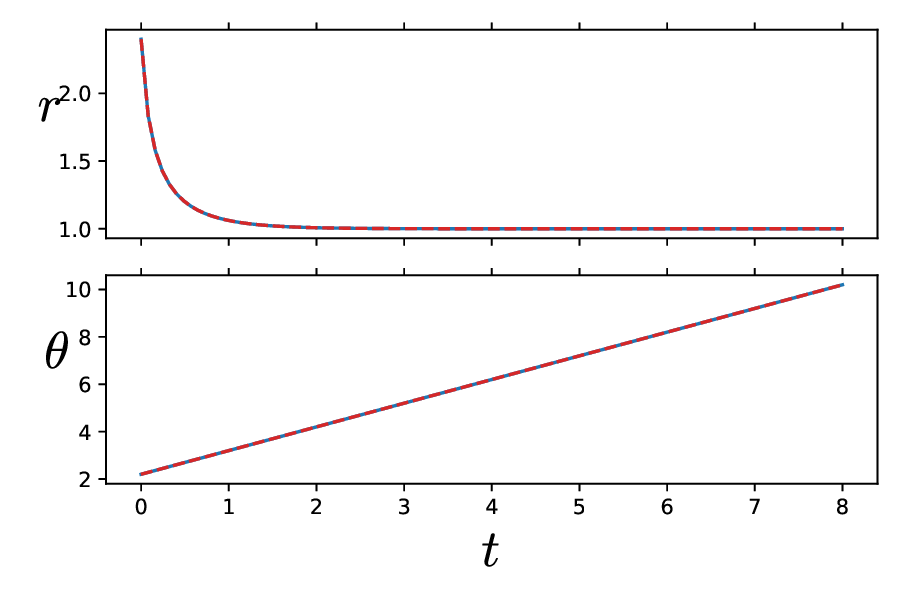}
        \caption{Time}
        \label{hopf_time}
    \end{subfigure}
    \caption{Hopf normal form with $x = r\cos\theta$, $y = r\sin\theta$.  (a) Phase portrait illustrating convergence to the limit cycle; training data (teal) are drawn primarily from post-transient dynamics.  (b) Time evolution of $r$ and $\theta$, showing exponential radial convergence and linear phase growth. Despite the non-uniqueness of circular limit-cycle dynamics, the identified model recovers the correct normal form structure.}
    \label{hopf_fig}
\end{figure*}

This example demonstrates that accurate identification of dynamical systems with attracting invariant sets requires sufficiently informative data, particularly capturing transient behavior.  The use of a multi-step optimization horizon further aids in distinguishing between candidate models by penalizing those that fail to reproduce the full trajectory evolution under repeated composition.

Finally, we note that even when our system identification algorithm gives a model which is structurally different from the true one, the asymptotic behavior is often close to the true periodic orbit, as expected from \cref{appendix:limit_cycle} for systems with a stable periodic orbit.  For example, \cref{approx_hopf_1}-\cref{approx_hopf_2} show trajectories from such models, \cref{hopf_1} and \cref{hopf_2} respectively, whose modeling error is small. The stable periodic orbit of this model is nearly -- but not exactly -- the same in phase space. Consequently, the Loss \cref{loss} over increasingly long horizons is also small.

\begin{figure*}[h]
    \centering
    \begin{subfigure}[t]{0.49\textwidth}
        \centering
        \includegraphics[height=1.75in]{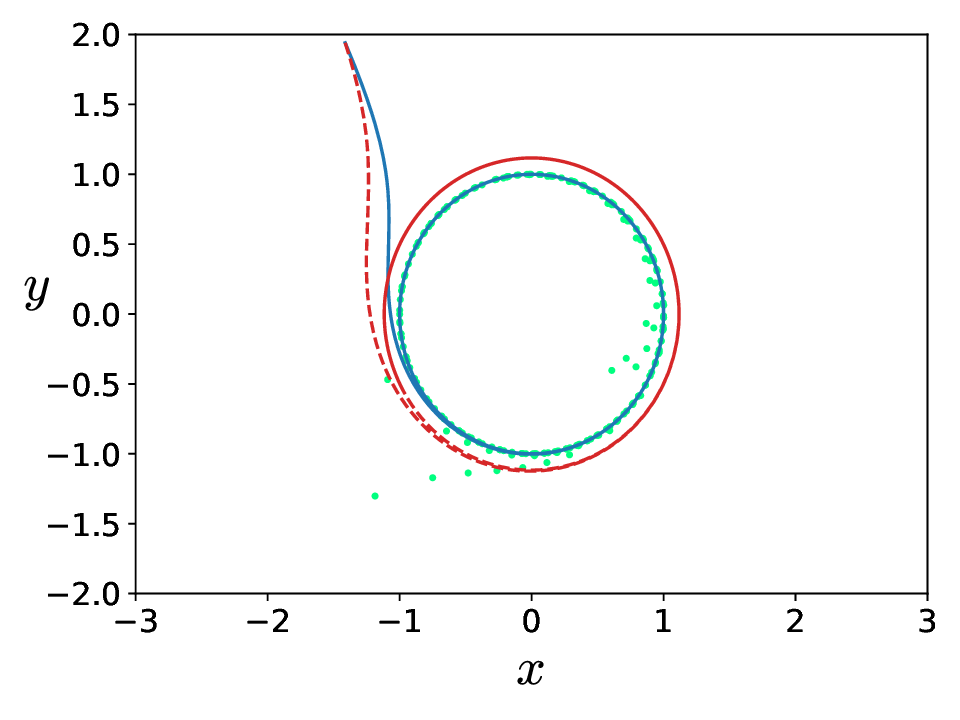}
        \caption{Phase portrait of \cref{hopf_1}}
        \label{approx_hopf_1}
    \end{subfigure}%
    ~ 
    \begin{subfigure}[t]{0.49\textwidth}
        \centering
        \includegraphics[height=1.75in]{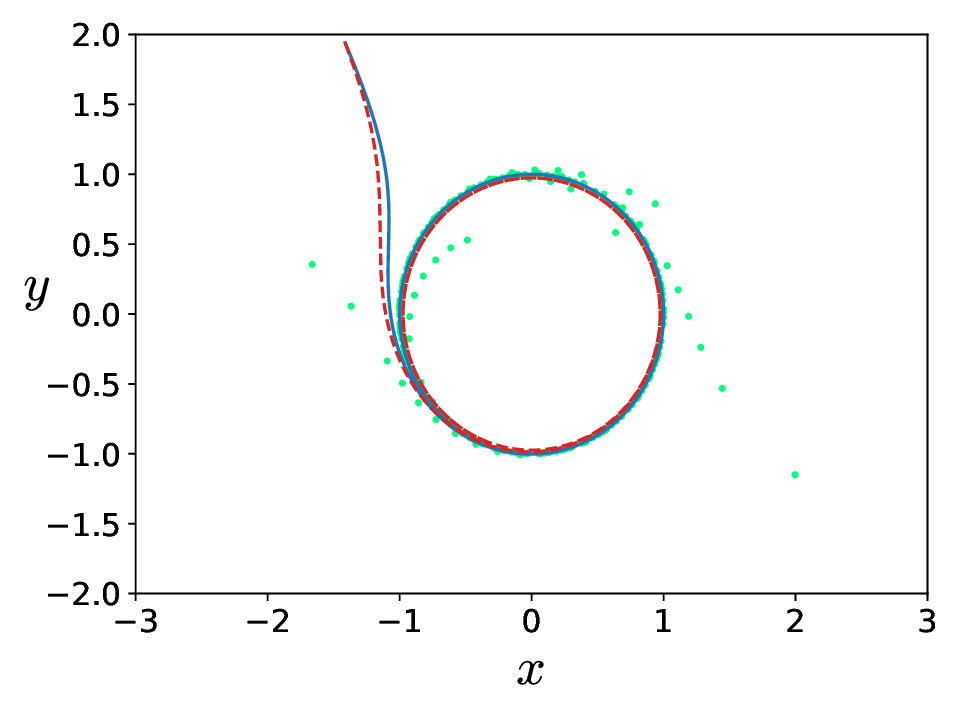}
        \caption{Phase portrait of \cref{hopf_2}}
        \label{approx_hopf_2}
    \end{subfigure}

    \begin{subfigure}[t]{0.49\textwidth}
        \centering
        \includegraphics[height=1.75in]{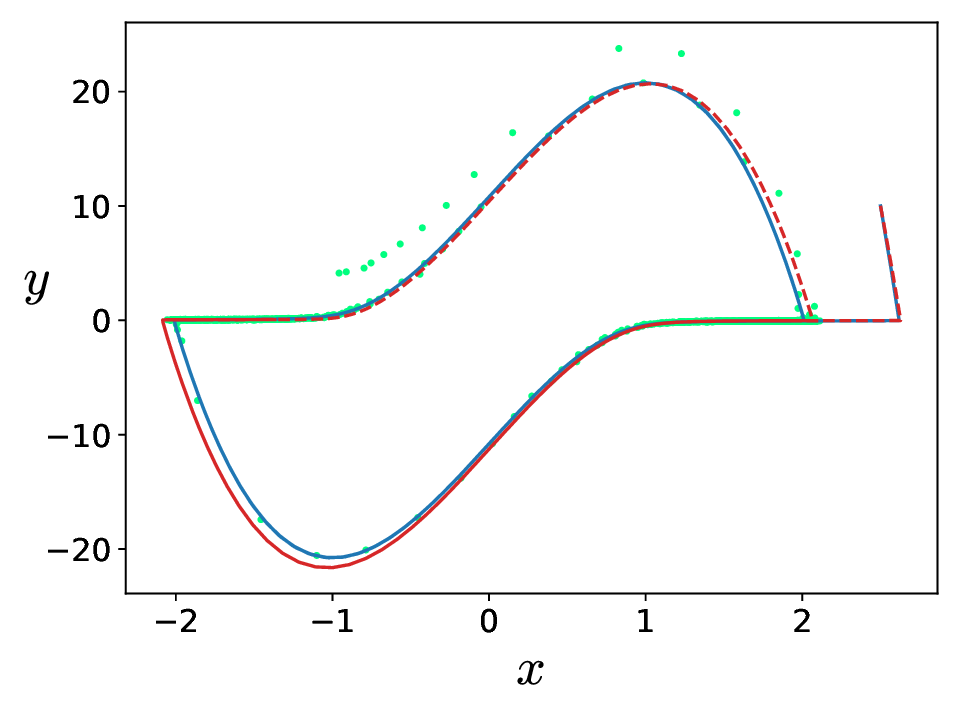}
        \caption{Phase portrait of \cref{vdp_1}}
        \label{approx_VDP_1}
    \end{subfigure}%
    ~ 
    \begin{subfigure}[t]{0.49\textwidth}
        \centering
        \includegraphics[height=1.75in]{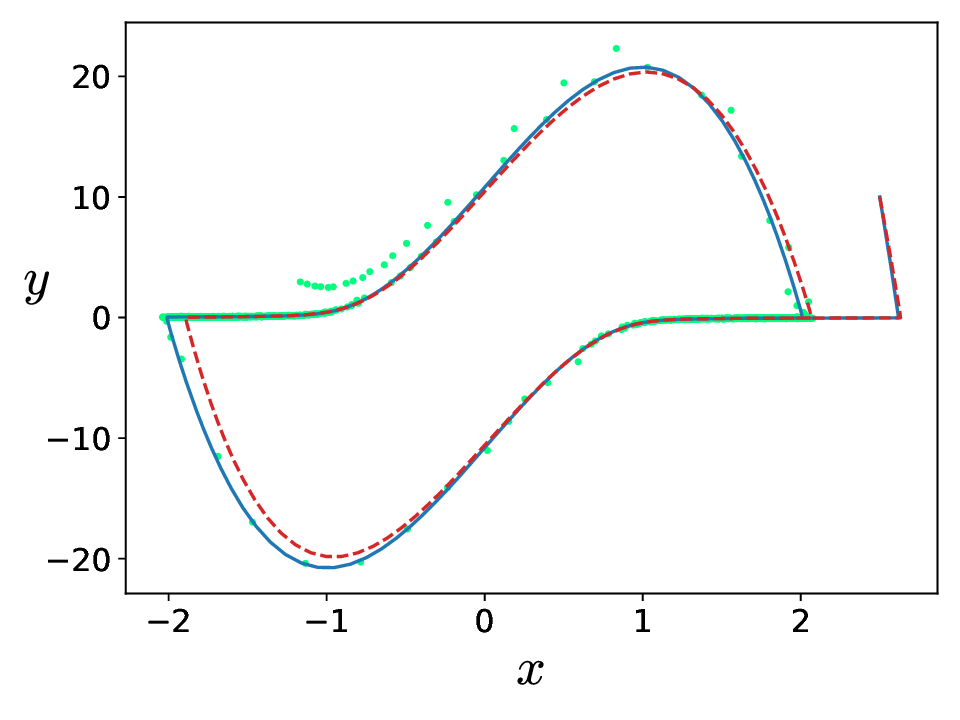}
        \caption{Phase portrait of \cref{vdp_2}}
        \label{approx_VDP_2}
    \end{subfigure}
    \caption{Phase portraits of model approximations for \cref{hopf_r}-\cref{hopf_theta} (top) and \cref{vdp_x}-\cref{vdp_y} (bottom). Notice how the periodic orbits of the approximate models (dashed red) seem like a slightly perturbed version of the exact model (solid blue). Such models have similar stability properties, dynamic range, and time-periods as explained in \cref{appendix:limit_cycle}.}
    \label{approx_figs_2}
\end{figure*}

\subsection{Simple Pendulum}

We next consider the dynamics of a simple pendulum
\begin{equation}\label{pendulum_eqn}
    \ddot{\theta} = -\frac{g}{l} \sin \theta,
\end{equation}
with parameters $g = 9.81$ and $l = 2$. Writing this as a first-order system with $x = \theta$ and $y = \dot{\theta}$, we obtain a nonlinear oscillator with a continuum of periodic orbits whose periods depend on the amplitude.

This example highlights the challenge of distinguishing between locally accurate but globally incorrect models. In particular, for data consisting of a single trajectory, the dynamics can be well-approximated by a linear system corresponding to small oscillations. Such a model yields periodic behavior and can achieve low short-term prediction error, despite failing to capture the amplitude-dependent nonlinear dynamics of the true system.

To address this ambiguity, we generate a dataset consisting of $m^* = 201$ data points sampled from $p = 3$ trajectories with different initial conditions over the time interval $T = [0,5]$, and add $1\%$ NSR as described in \cref{noise}. The use of multiple trajectories introduces variation in oscillation amplitudes and periods, providing information necessary to identify the nonlinear dependence in \cref{pendulum_eqn}. The NN is $K=1$ stack deep and $L=10$ layers wide (236 parameters), and trained using a NODE horizon of $H = 4$ time steps.
Using the proposed method, we obtain the identified model
\begin{subequations}\label{identified_pendulum}
\begin{align}
    \dot{\theta}_1 &= \theta_2, \label{pendulum_x}\\
    \dot{\theta}_2 &= -5 \sin \theta_1, \label{pendulum_y}
\end{align}
\end{subequations}
which closely matches the true system, differing only slightly in the coefficient.

A comparison between the true and identified dynamics is shown in \cref{pendulum_fig}. In the phase portrait \cref{pendulum_space}, the training data (teal dots) spans multiple trajectories with different amplitudes. The trajectory generated by the identified model (dashed red) closely overlaps with that of the true system (solid blue). This agreement is also evident in the time-series representation \cref{pendulum_time}.
We emphasize that when the training data is restricted to a single trajectory, the method often identifies a linear model that reproduces the observed motion but fails to generalize to other initial conditions. In contrast, the inclusion of multiple trajectories leads to sufficiently large discrepancies for such linear approximations, allowing the nonlinear sinusoidal term to be identified.

This example demonstrates that accurate recovery of nonlinear dynamics requires both sufficiently informative data and a loss function that penalizes errors over multiple time steps. In particular, the multi-step optimization horizon helps distinguish between models that fit individual trajectories and those that correctly capture the underlying dynamical structure.

\begin{figure*}[h]
    \centering
    \begin{subfigure}[t]{0.49\textwidth}
        \centering
        \includegraphics[height=1.9in]{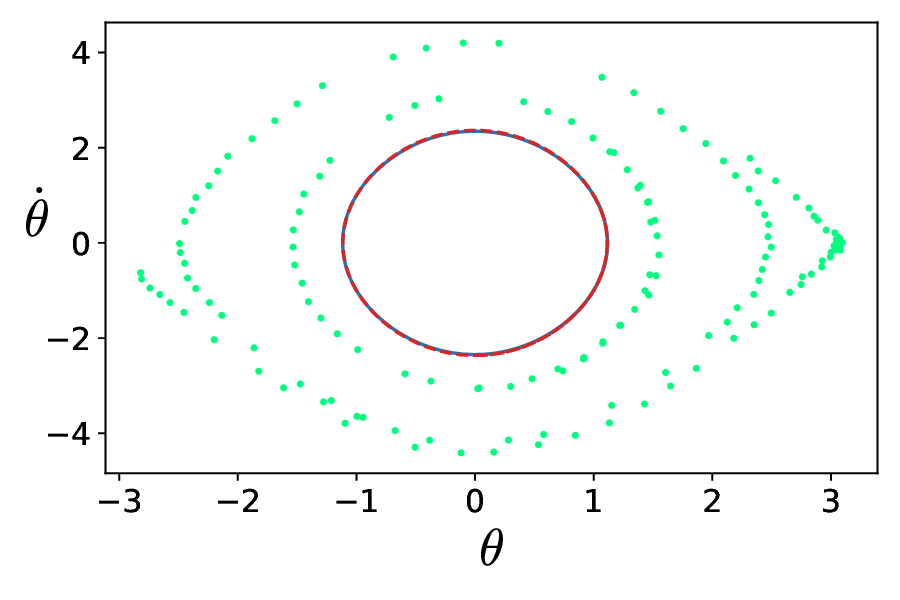}
        \caption{Space}
        \label{pendulum_space}
    \end{subfigure}%
    ~ 
    \begin{subfigure}[t]{0.49\textwidth}
        \centering
        \includegraphics[height=1.9in]{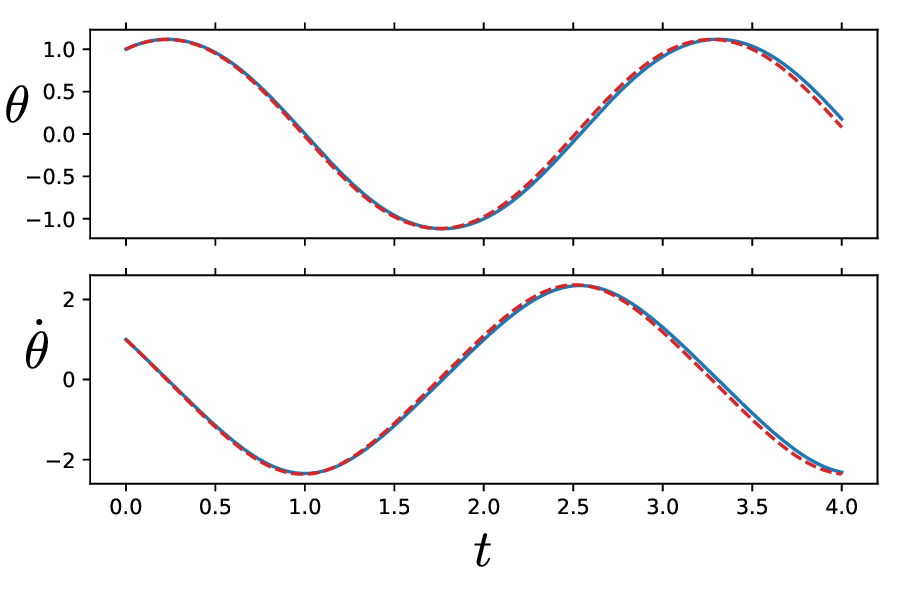}
        \caption{Time}
        \label{pendulum_time}
    \end{subfigure}
    \caption{Simple pendulum. (a) Phase portrait showing training data (teal), true trajectory (solid blue), and trajectory generated by the identified model (dashed red). The data span multiple energy levels, illustrating nonlinear periodic motion across trajectories of differing amplitudes. (b) Time-series comparison from the same initial condition. The close agreement demonstrates accurate recovery of the underlying Hamiltonian structure and preservation of periodic dynamics over the NODE horizon.}
    \label{pendulum_fig}
\end{figure*}

\subsection{Van der Pol oscillator} 
\label{example:van_der_pol}

We next consider the Van der Pol oscillator \cite{van_der_pol}
\begin{subequations}
\begin{align}
    \dot{x} &= y, \label{vdp_x}\\
    \dot{y} &= \mu (1 - x^2) y - x, \label{vdp_y}
\end{align}
\end{subequations}
which provides a canonical example of a nonlinear system with multiple time scales. As the parameter $\mu$ increases, the system exhibits relaxation oscillations characterized by slow evolution along attracting manifolds and rapid transitions between them. This separation of time scales leads to stiffness for large values of $\mu$, posing significant challenges for both numerical integration and data-driven model identification.

We investigate the performance of the proposed method for $\mu \in [3,15]$. For moderate values of $\mu$ (e.g., $\mu \in [3,6]$), the dynamics are only mildly stiff, and the governing equations can be accurately identified from relatively small datasets. In particular, we find that approximately $m^* = 600$ data points sampled from $1$--$3$ trajectories with $1\%$ NSR are sufficient to recover the exact model.
As $\mu$ increases, the separation of time scales becomes more pronounced. The $y$-component exhibits increasingly rapid variation, while the $x$-component evolves more slowly along the slow manifold. This results in a significant disparity in the dynamic range and temporal scales of the state variables, making parameter estimation more difficult. For $\mu \approx 10$, we find that approximately $m^* =900$ data points are required to reliably recover the governing equations.
In the stiff regime ($\mu \approx 15$), accurate identification becomes substantially more challenging. The rapid transitions require sufficiently fine temporal resolution in the data, while numerical integration errors accumulate over the NODE horizon \cite{colby_3}. In this regime, we find that $m^* = 1300$--$1900$ data points from $p = 1-2$ trajectories spanning $T=25-45$ time-units with $1\%$ NSR are sufficient to recover the correct model, even when using a standard (non-stiff) RK4 integration scheme.

One such result for $\mu = 15$ is shown in \cref{vdp_space}. The $m^* = 1200$ points of training data with $1\%$ NSR are sampled from $p = 1$ trajectory over $T=[0,30]$. Despite this limited transient information, the exact model is identified after $800$ epochs using $K=2$ stacks and $L=1$ layer ($68$ parameters) using a horizon of $H = 16$ time-steps. The agreement between the true and identified trajectories is evident in both the phase portrait and time-series representations.
We note that, in contrast to the Hopf normal form, accurate identification is possible even when the data is concentrated near the limit cycle. This reflects the richer nonlinear structure of the Van der Pol dynamics, which imposes stronger constraints on the admissible vector fields.

\begin{figure*}[h]
    \centering
    \begin{subfigure}[t]{0.49\textwidth}
        \centering
        \includegraphics[height=1.9in]{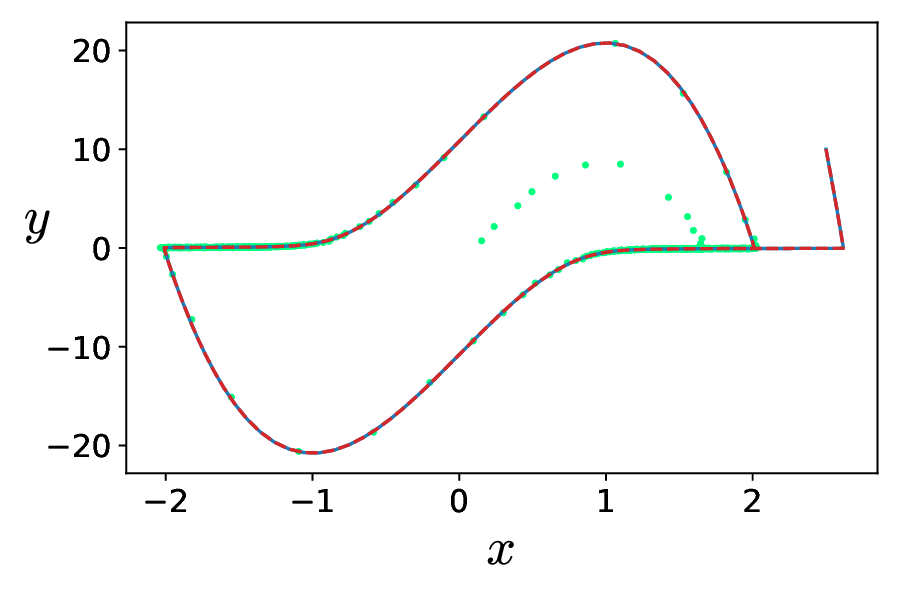}
        \caption{Space}
        \label{vdp_space}
    \end{subfigure}%
    ~ 
    \begin{subfigure}[t]{0.49\textwidth}
        \centering
        \includegraphics[height=1.9in]{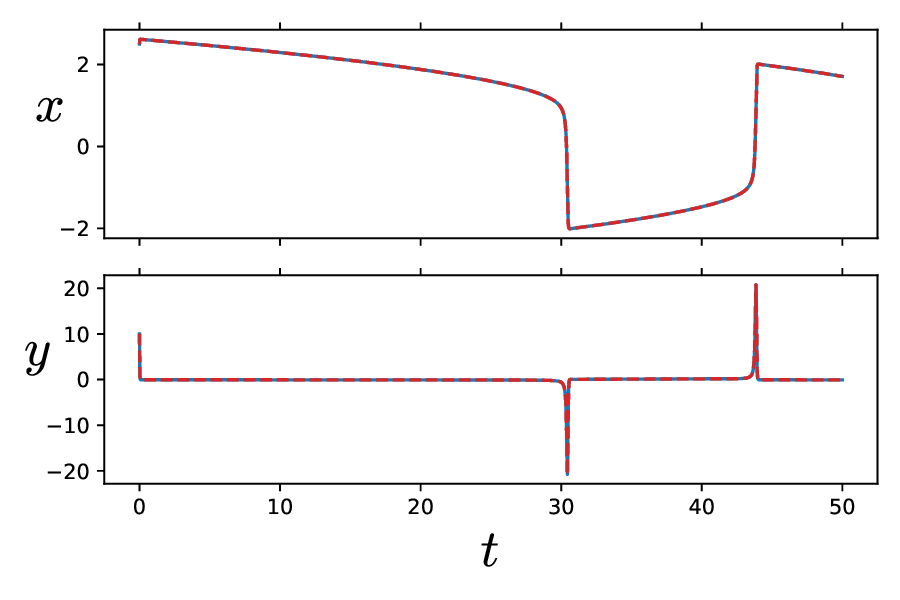}
        \caption{Time}
    \end{subfigure}
    \caption{Van der Pol oscillator with $\mu = 15$, demonstrating scale separation over length and time. (a) Phase portrait showing training data (teal), true trajectory (solid blue), and trajectory generated by the identified model (dashed red). The trajectories converge to a stable limit cycle, capturing both transient dynamics and nonlinear damping. (b) Time-series comparison from the same initial condition. The agreement illustrates accurate recovery of relaxation oscillations and the stable periodic orbit.}
\end{figure*}

This example demonstrates that the proposed method remains effective in the presence of strong nonlinearities and multiple time scales, provided that the data resolution and NODE horizon are chosen appropriately. 
It also highlights the interplay between numerical integration error and model identification, as discussed in \cref{sec:problem_formulation}, particularly in stiff regimes where errors can accumulate rapidly under repeated composition.

Our system identification approach doesn't always give models with the structure of the true model.  But since the periodic orbit for the true model is stable, the results from Appendix~\cref{appendix:limit_cycle} are expected to hold for small modeling errors. \cref{approx_VDP_1} and \cref{approx_VDP_2} show the stable periodic orbits found for such models, \cref{vdp_1} and \cref{vdp_2} respectively, which are different from the stable periodic orbit for \cref{vdp_x}-\cref{vdp_y} but with nearly the same dynamic range, stability characteristics, and time-period.

\subsection{Chaotic dynamics: sensitivity, horizon length, and statistical fidelity}\label{chaotic_examples}

We now consider three canonical chaotic systems: the R\"ossler attractor, Chua's double scroll, and the Lorenz equations. These examples highlight several key challenges for data-driven discovery:
(i) sensitivity to initial conditions,
(ii) rapid error growth under composition,
and (iii) the distinction between short-term trajectory accuracy and long-term statistical fidelity.
From the perspective of the Loss$_H$ framework developed earlier, chaotic systems provide a stringent test: even small one-step errors can amplify exponentially, so minimizing a multi-step loss is essential for identifying dynamically consistent models.

\subsubsection{R\"ossler equations}

We first consider the R\"ossler system~\cite{rossler1976chaotic}, which finds relevance in modeling the chaotic behavior of chemical reactions:
\begin{subequations}
\begin{align}
  \dot{x} &= -y - z, \label{rossler_x}\\
  \dot{y} &= x + a y,\\
  \dot{z} &= z(x - c) + b, \label{rossler_z}
\end{align}
\end{subequations}
with parameters $a = 0.5$, $b = 2$, and $c = 4$, as considered in~\cite{rossler_parameters}. This system exhibits a chaotic attractor alongside a different region of state-space where trajectories grow unboundedly. This example highlights the challenge of error accumulation in chaotic regimes; without a multi-step optimization horizon to enforce consistency, small modeling errors can easily perturb generated trajectories into unbounded regions as proposed in \ref{accumulating_error}.

%\paragraph{Exact recovery and dependence on sampling.}
We find that when the data is noise-free, the exact model \cref{rossler_x}--\cref{rossler_z} can be recovered using $H = 1$ time-step from $m^* = 1000$ data-points using a larger $\Delta t$ such as $0.2$; see \cref{rossler_exact}.  However, for noisy data, the model estimate using $H=1$ is only approximate. 
A key observation is that increasing the NODE horizon $H$ significantly improves robustness to noise. Using longer horizons, we are able to exactly identify the model even with up to $20\%$ NSR:
\begin{subequations}
\begin{align}
  \dot{x} &= -y - z,\\
  \dot{y} &= x + 0.5 y,\\
  \dot{z} &= z(x - 4) + 2.
\end{align}
\end{subequations}

\begin{figure*}[h]
    \centering
    \begin{subfigure}[t]{0.49\textwidth}
        \centering
        \includegraphics[height=1.9in]{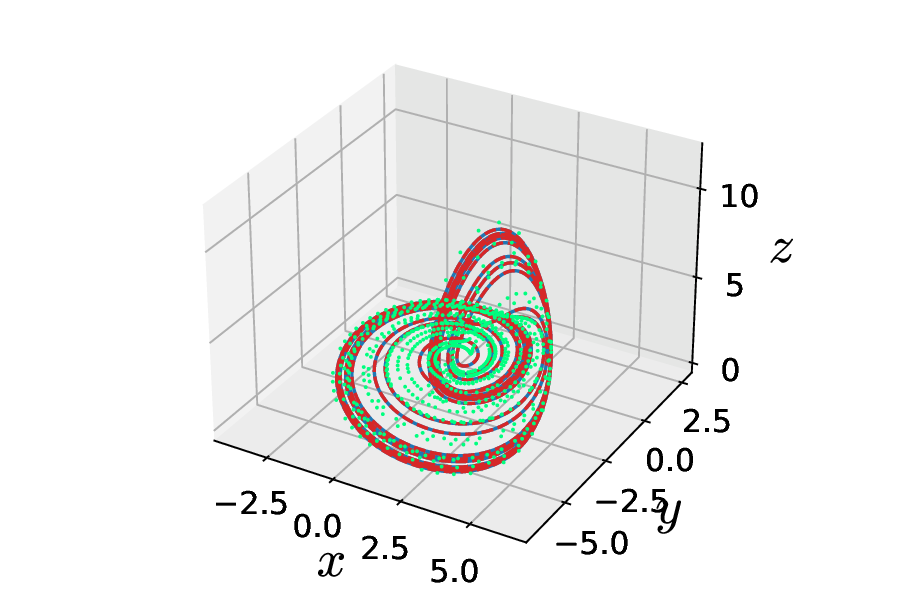}
        \caption{Space}
    \end{subfigure}%
    ~ 
    \begin{subfigure}[t]{0.49\textwidth}
        \centering
        \includegraphics[height=1.9in]{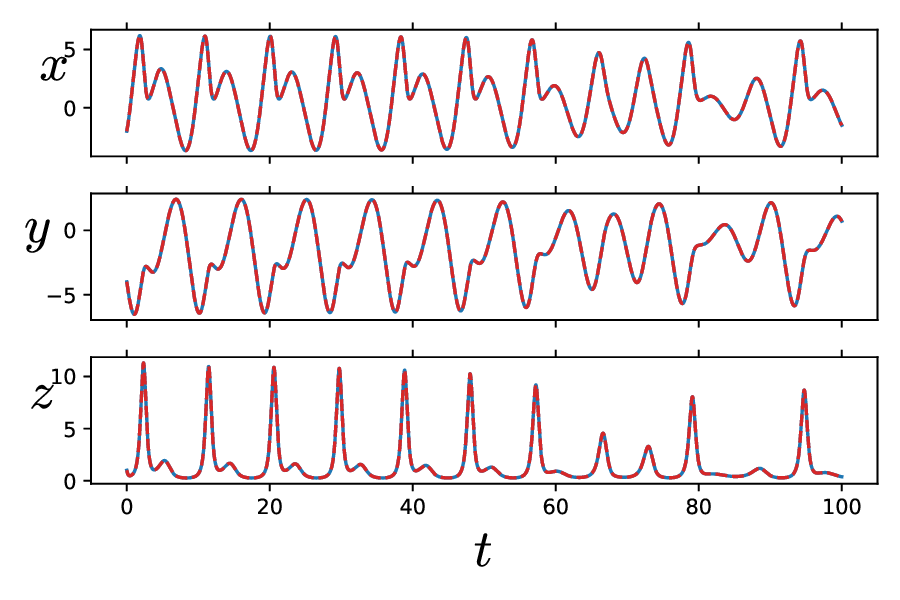}
        \caption{Time}
    \end{subfigure}
    \caption{R\"ossler system. (a) Phase portrait showing noise-free training points (teal), true trajectory (solid blue), and trajectory generated by the identified model (dashed red).  (b) Time-series comparison, showing exact agreement between the exact and estimated model.} 
    \label{rossler_exact}
\end{figure*}

In \cref{rossler_noise_figs}, we show figures for the estimated trajectories (dashed red) which overlap with the true trajectories (solid blue) because the model is identified exactly, and the corresponding training datasets (teal dots) at various levels of noise used to corrupt the training data. These models are learned from $m = 2560$ points ($m^* = 2560+H$ total using $H=1-64$) sampled uniformly over $T = [0, 200]$ from $p=1$ trajectory. To alleviate the increase in computations required with longer horizons (up-to $H = 256$ time-steps in this study), we also down-sized the NN from $L=10$ layers wide ($322$ trainable parameters) to $3$ layers ($105$ parameters). Notice how in \cref{rossler_space_0}, which is noise-free, you can visualize the attractor from the training data. The visualization gets increasingly more scattered from \cref{rossler_space_5} with $5\%$ NSR (corresponding to additive $\sigma_{abs} \approx [0.13, 0.14, 0.12]$) to \cref{rossler_space_20} which has $20\%$ NSR ($\sigma_{abs} \approx [0.49, 0.55, 0.48]$) where the training data seems just clumped up in state-space. So we find it noteworthy that our algorithm can estimate the exact model for the \rossler attractor even with this much  noise - and consequently large loss even after converging upon the exact model - over the training dataset.  This illustrates a central point of the Loss$_H$ framework: multi-step consistency provides a strong regularization effect that filters out noise. 

\begin{figure*}[h!]
    \centering
    \begin{subfigure}[t]{0.495\textwidth}
        \centering
        \includegraphics[height=1.95in]{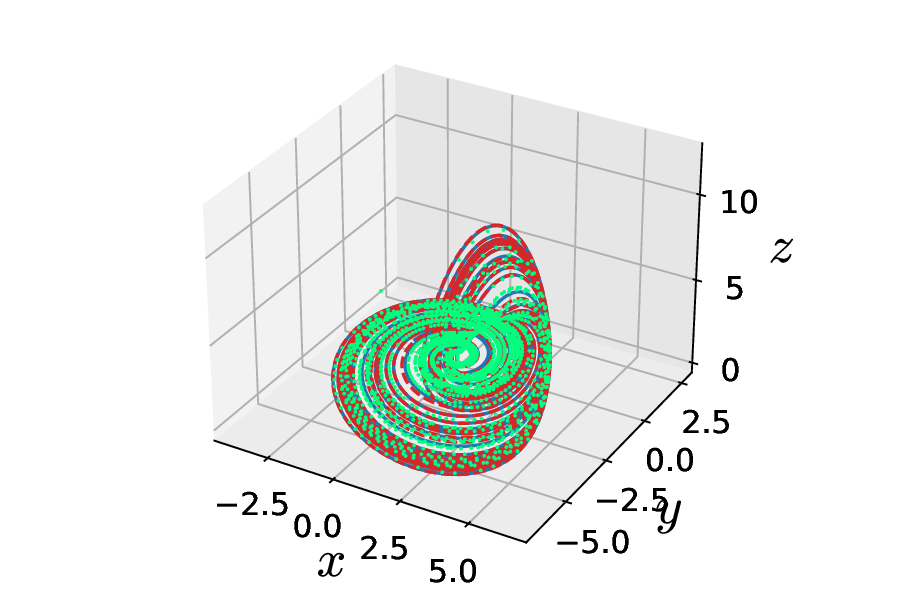}
        \caption{noise-free, $H = 4$}\label{rossler_space_0}
    \end{subfigure}%
    ~ 
    \begin{subfigure}[t]{0.495\textwidth}
        \centering
        \includegraphics[height=1.95in]{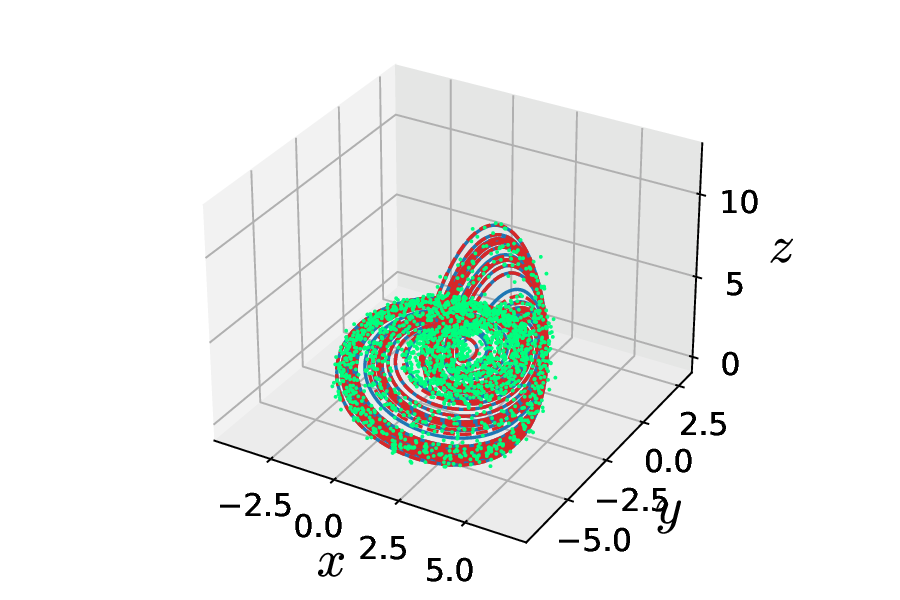}
        \caption{$5\%$ NSR, $H = 8$}\label{rossler_space_5}
    \end{subfigure}
    
    \begin{subfigure}[t]{0.495\textwidth}
        \centering
        \includegraphics[height=1.95in]{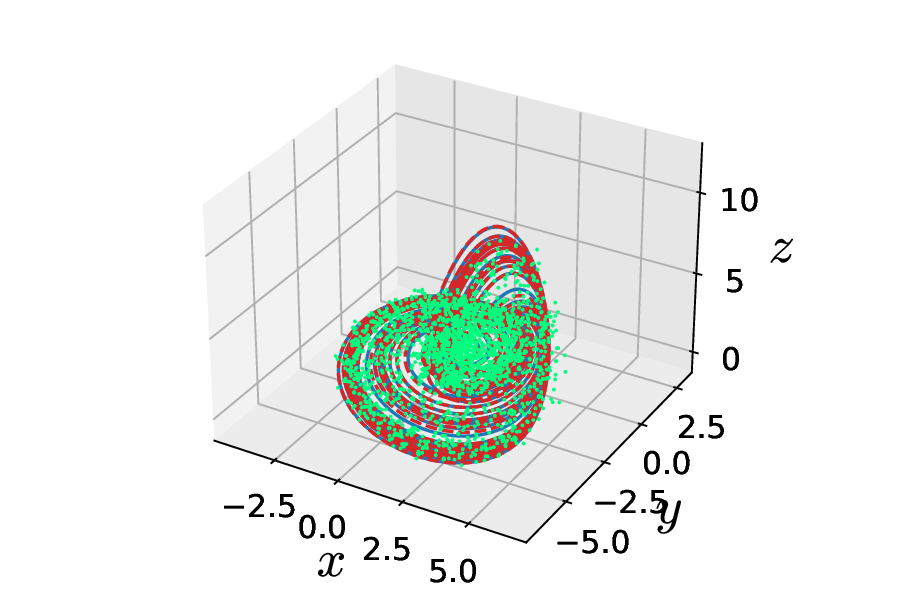}
        \caption{$10\%$ NSR, $H = 16$}\label{rossler_space_10}
    \end{subfigure}%
    ~ 
    \begin{subfigure}[t]{0.495\textwidth}
        \centering
        \includegraphics[height=1.95in]{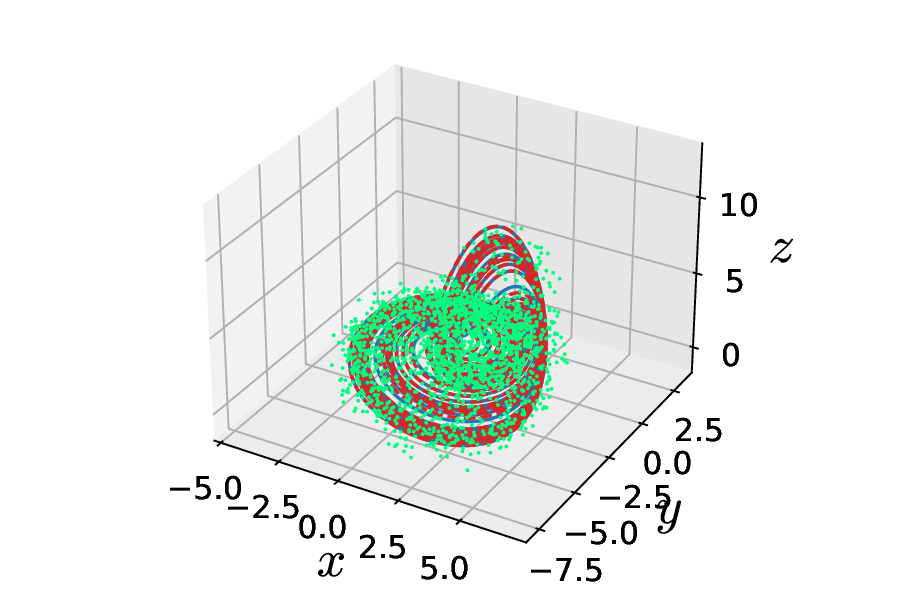}
        \caption{$15\%$ NSR, $H = 32$}\label{rossler_space_15}
    \end{subfigure}
    
    \begin{subfigure}[t]{0.495\textwidth}
        \centering
        \includegraphics[height=1.95in]{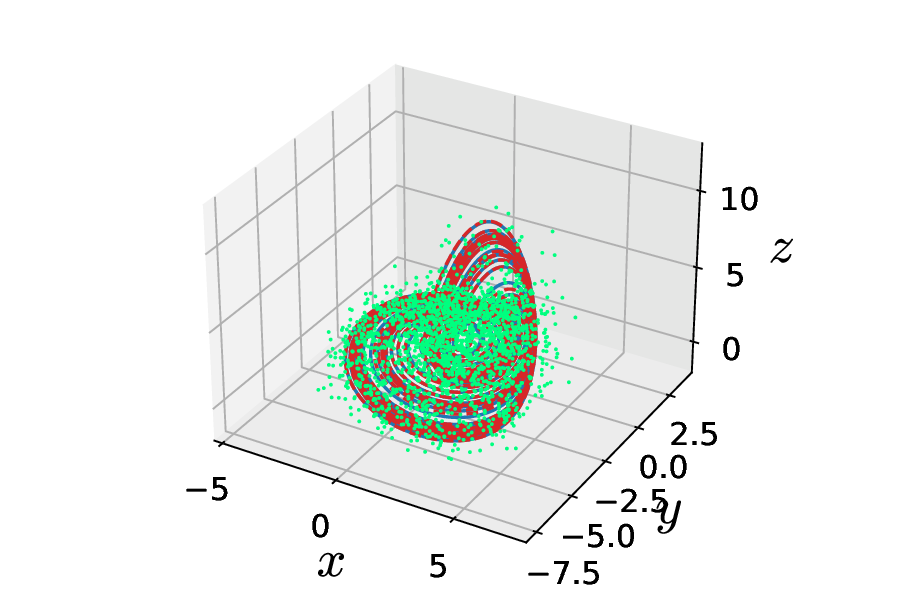}
        \caption{$20\%$ NSR, $H = 64$}\label{rossler_space_20}
    \end{subfigure}
    
    \caption{\small{Comparison of $m^* = 2560+H$ data (green dots) from the \rossler attractor with $0-20\%$ NSR. All models are exact as evident from the overlapping true (blue) and estimated (red) trajectories.}}\label{rossler_noise_figs}%%%%%%%%%%%%%%%%%%%%%%%%%%%%%% Adding to the above description lengthens it and sends text to the next page
\end{figure*}

%\paragraph{Approximate models and bias under noise.}
For higher noise levels ($25\%$ and $30\%$ NSR), the respective recovered models for $H = 128$ and $256$, respectively, deviate slightly:
\begin{subequations}\label{ros_eq_25}
\begin{align}
  \dot{x} &= -y - z,\\
  \dot{y} &= x + 0.5 y,\\
  \dot{z} &= z(x - 4) + 1.667,
\end{align}
\end{subequations}
and
\begin{subequations}\label{ros_eq_30}
\begin{align}
  \dot{x} &= -y - 0.75 z - 0.2,\\
  \dot{y} &= x + 0.4 y,\\
  \dot{z} &= z(x - 3.75) + 1.5 + 0.2 x,
\end{align}
\end{subequations}
with dynamics shown in \cref{rossler__bad_figs}.  For $25\%$ NSR, the only difference is in the constant term in $\dot{z}$, and we can see in \cref{rossler_25_time} how the the estimated model tracks the true models for 2 entire Lyapunov times unit after which it slowly diverges.  For $30\%$ NSR, the model has additional spurious terms.  This is reflected visually where neither the attractor \cref{rossler_30_space} nor the time-series \cref{rossler_30_time} is as close to the true model as \cref{rossler_25_space} and \cref{rossler_25_time}, respectively. Accordingly, the mean $\tilde{\bm \mu} = [ 0.75,  -1.5, 1.5]$, standard deviation $\tilde{\bm \sigma} = [2.55, 2.53,  2.44]$ and Lyapunov exponents $\tilde{\bm \lambda}_e = [0.17, 0, -2.91]$ of \cref{ros_eq_25} are much closer to that of the true model's: $\bm \mu = [0.72, -1.44, 1.43]$, $\bm \sigma = [2.39, 2.33, 2]$, and $\bm \lambda_e = [0.13, 0, -2.91]$ -- than those of \cref{ros_eq_30}: $\tilde{\bm \mu} = [ 0.42,  -1.06, 1.14]$, $\tilde{\bm \sigma} = [2.2, 2.03,  1.5]$, and $\tilde{\bm \lambda}_e = [0.07, 0, -2.99]$.

\begin{figure*}[h!]
    \centering
    \begin{subfigure}[t]{0.52\textwidth}
        \centering
        \includegraphics[height=2.1in]{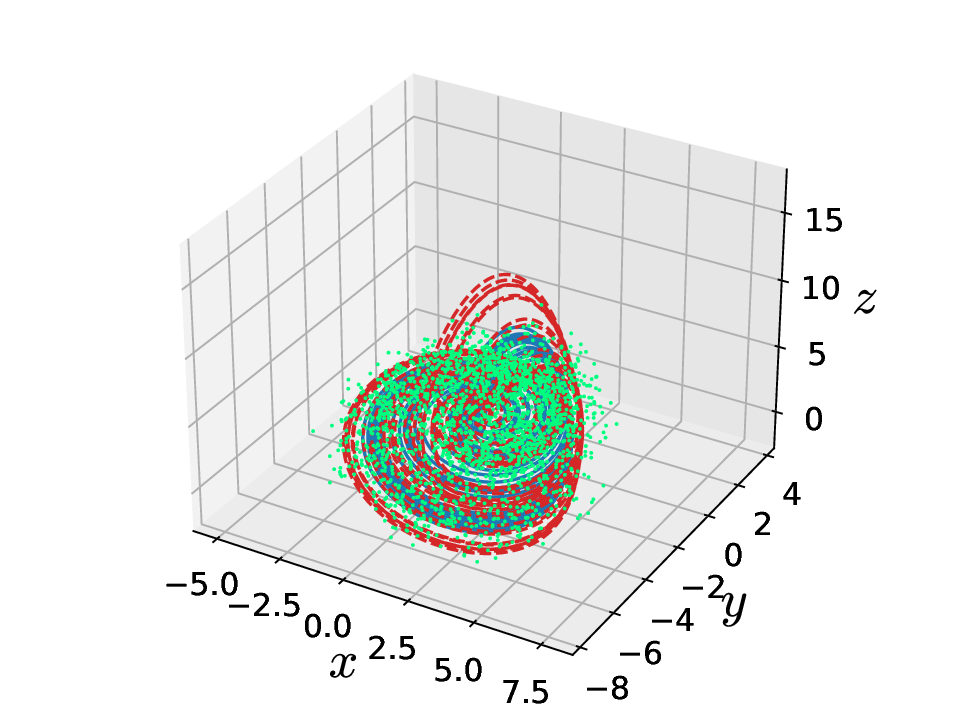}
        \caption{$25\%$ Space }\label{rossler_25_space}
    \end{subfigure}%
    ~ 
    \begin{subfigure}[t]{0.47\textwidth}
        \centering
        \includegraphics[height=2.075in]{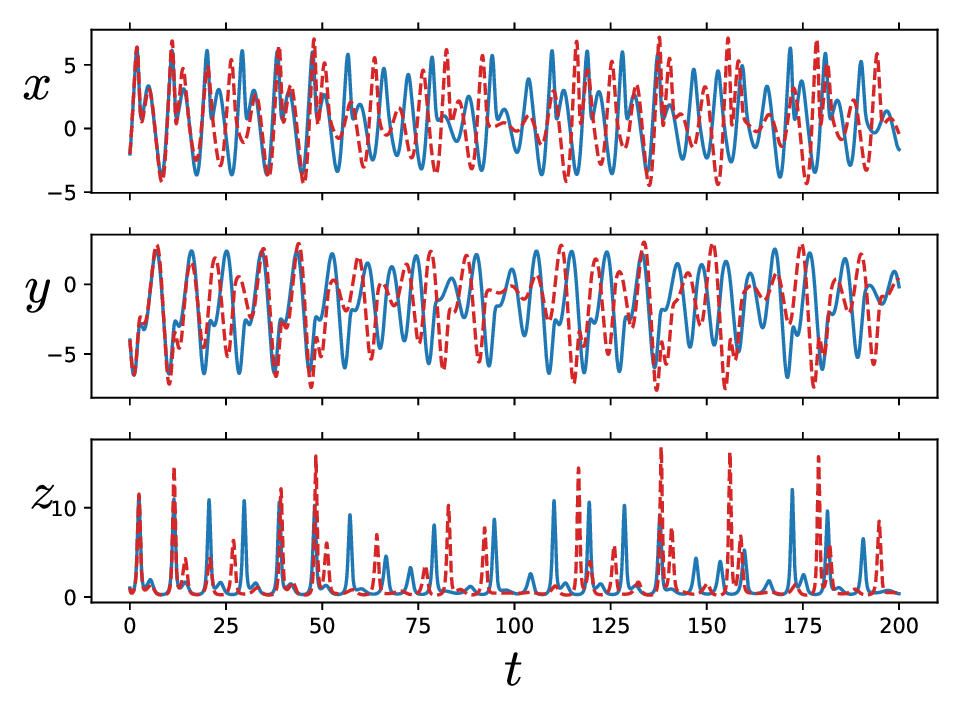}
        \caption{Time-series}\label{rossler_25_time}
    \end{subfigure}
    
    \begin{subfigure}[t]{0.52\textwidth}
        \centering
        \includegraphics[height=2.1in]{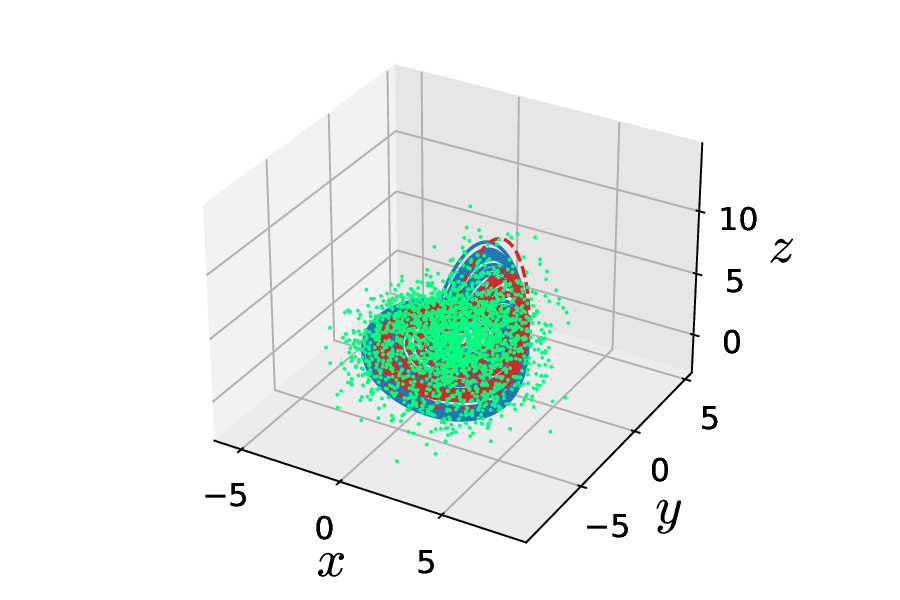}
        \caption{$30\%$ Space }\label{rossler_30_space}
    \end{subfigure}%
    ~ 
    \begin{subfigure}[t]{0.47\textwidth}
        \centering
        \includegraphics[height=1.9in]{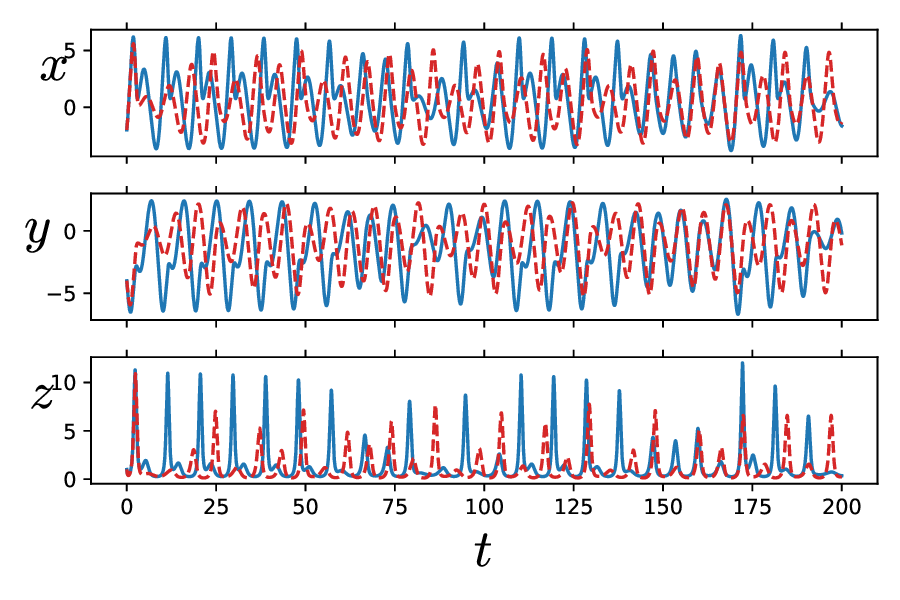}
        \caption{Time-series}\label{rossler_30_time}
    \end{subfigure}
    \caption{\small{Phase portraits of models estimated for the \rossler attractor with (a) $25\%$ NSR (corresponding additive $\sigma_{abs} \approx [0.67, 0.75, 0.7]$)} and (c) $30\%$ NSR (corresponding additive $\sigma_{abs} \approx [0.75, 0.85, 0.78]$) using horizons of $H=128$ and $256$ time-steps respectively. Corresponding time-series are shown in (b) and (d)}\label{rossler__bad_figs}.
\end{figure*}

These deviations are consistent with \cref{l2_norm_eqv}, which shows that minimizing $\mathrm{Loss}_H$ controls the average multi-step trajectory error over the training horizon.
Thus, models with small multi-step loss necessarily produce trajectories 
that remain close to the true trajectories over the training horizon. 
In chaotic systems, however, small errors at each step can accumulate, 
leading to eventual divergence beyond the training horizon.  Thus, even when the identified model is only approximate, it must reproduce trajectories accurately over $h = 1,\dots,H$, with deviations that accumulate gradually beyond this window in chaotic systems.

Finally, \cref{horizon_effect_fig} shows that larger horizons ($H \geq 32$) lead to convergence to the exact model, while smaller horizons get trapped in suboptimal minima.  
These results directly reflect the theoretical role of $H$: increasing $H$ enforces consistency of repeated compositions and reduces the effective model error.

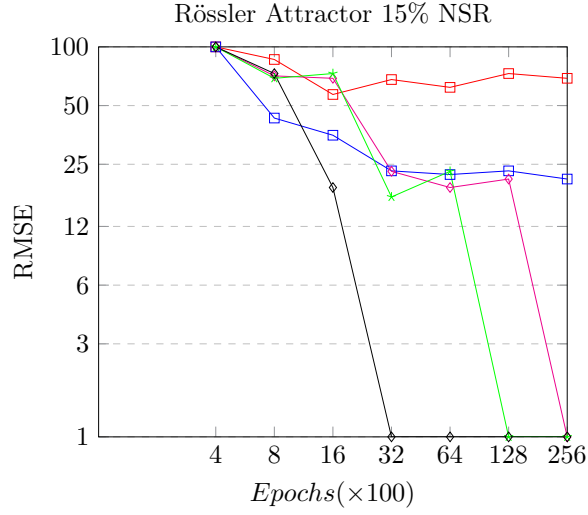
\begin{figure}[t]
    \centering
    \resizebox{0.5\columnwidth}{!}{%
\begin{tikzpicture}
\begin{axis}[
    title={\rossler Attractor $15\%$ NSR},
    xmode = log, log ticks with fixed point,
    ymode = log, log ticks with fixed point,
    xlabel={$Epochs (\times 100)$},
    ylabel={RMSE},
    xmin=1, xmax=256,
    ymin=1, ymax=100,
    xtick={4, 8, 16, 32, 64, 128, 256},
    ytick={1,3,6,12,25,50,100},
    legend pos=north west,
    ymajorgrids=true,
    grid style=dashed,
]

\addplot[
    color=red,
    mark=square,
    ]
    coordinates {
    (4, 100)(8, 86.17)(16, 57)(32, 68)(64, 62)(128, 73)(256, 69)
    };
% \addlegendentry{H = 8 time-steps}
\addplot[
    color=blue,
    mark=square,
    ]
    coordinates {
    (4, 100)(8, 43.15)(16, 35.23)(32, 23.13)(64, 22.15)(128, 23.13)(256, 21)
    };
\addplot[
color=magenta,
mark=diamond,
]
coordinates {
(4, 100)(8, 71)(16, 69)(32, 23)(64, 19)(128, 21)(256, 1)
};

\addplot[
color=green,
mark=star,
]
coordinates {
(4, 100)(8, 69)(16, 73)(32, 17)(64, 23)(128, 1)(256, 1)
};

\addplot[
color=black,
mark=diamond,
]
coordinates {
(4, 100)(8, 73)(16, 19)(32, 1)(64, 1)(128, 1)(256, 1)
};

\end{axis}
\end{tikzpicture}
}
        \caption{RMSE versus training epochs for $H=$ 8 (red), 16 (blue), 32 (magenta), 64 (green), 128 (black) time-steps. For the same training dataset size of $m = 2560$ in $\X_0$ ($m^* = 2560 + H$ in $\X$ \cref{m_total}), we can observe how the shorter horizons (8, 16) do not converge at all although they might initially have lower RMSE, while the longer horizons (32, 64, 128) not only converge to the true model but also converge earlier (epochs for convergence using $H=128$ are fewer than for  $H=64$, which in turn are fewer than for $H=32$).\label{horizon_effect_fig}
        }
\end{figure}

Note that to make longer forecasting horizons computationally tractable, the computational graph was made smaller by reducing to $L=3$ layers.  While we trade in expressivity of a larger NN, the payoff is that a smaller netowrk gives desired sparsity with an $\ell_1$ regularizer, which is convex and the gradients are better behaved.  For smaller model sizes, we find that $\ell_1$ regularization gives desired sparsity, so we use it instead of the non-convex $\lhalf$ regularizer. We find that longer NODE horizons are also more robust to noise with both the \rossler model and the Lorenz model (discussed below) identified from data with $1\%$ NSR which was not possible with shorter NODE horizons.

\subsubsection{Chua double scroll} \label{example:chua}

We next consider Chua's circuit~\cite{chua1986double}
\begin{subequations}
\begin{align}
\dot{x} &= 15.6 y - \frac{31.2}{7} x + 3.343 (|x+1|-|x-1|), \label{chua_x}\\
\dot{y} &= x - y + z, \label{chua_y}\\
\dot{z} &= -28 y. \label{chua_z}
\end{align}
\end{subequations}
Chua's circuit models a physical electronic oscillator exhibiting a chaotic double-scroll attractor uniquely characterized by a non-differentiable term representing a nonlinear resistor. Identifying this system challenges our framework to approximate non-smooth dynamics using smooth primitive functions, testing whether the estimated equations can represent the attractor's geometry and the characteristic switching behavior. This attractor also has a larger leading positive Lyapunov exponent than the \rossler attractor does, making short term-tracking and reproducing long-term statistics challenging with modeling errors. The model identified using $K=1$ stack and $L=10$ layers ($322$ parameters), from $m^* = 1000$ points without noise over $T=[0,140]$ using $H = 1$ time-step replaces the piecewise-linear nonlinearity with smooth approximations:\begin{subequations}
\begin{align}
\dot{x} &= 15.333 y - \tfrac{1}{3} x - 2\sin(2x - 3) - \tfrac{1}{3}\sin(4.333x + 6.333), \label{chua_main_x}\\
\dot{y} &= x - y + z,\\
\dot{z} &= -28 y. \label{chua_main_z}
\end{align}
\end{subequations}
We see that \cref{chua_y}-\cref{chua_z} are identified exactly and \cref{chua_main_x} has terms different from \cref{chua_x}. While these terms are not identified exactly, this shows the power of compositions in generating sinusoidal terms with arbitrary frequencies and phases.  While not structurally identical, this model reproduces the attractor geometry and dynamics accurately; see \cref{fig:chua}. This highlights the expressive power of the function class under composition: smooth basis functions can approximate nonsmooth nonlinearities when trained over multiple steps.

\begin{figure*}[h]
    \centering
    \begin{subfigure}[t]{0.49\textwidth}
        \centering
        \includegraphics[height=1.9in]{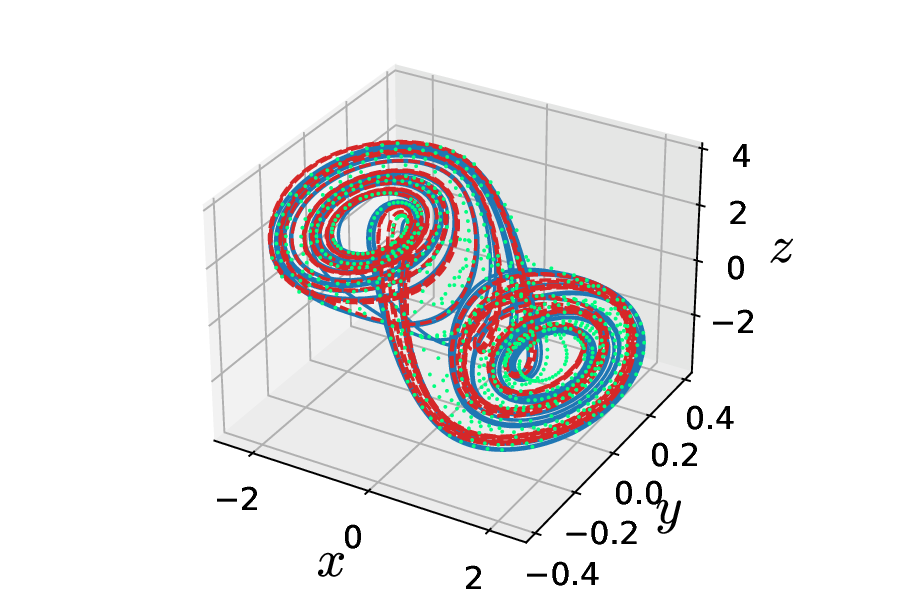}
        \caption{Space}
        \label{chua_space}
    \end{subfigure}%
    ~ 
    \begin{subfigure}[t]{0.49\textwidth}
        \centering
        \includegraphics[height=1.9in]{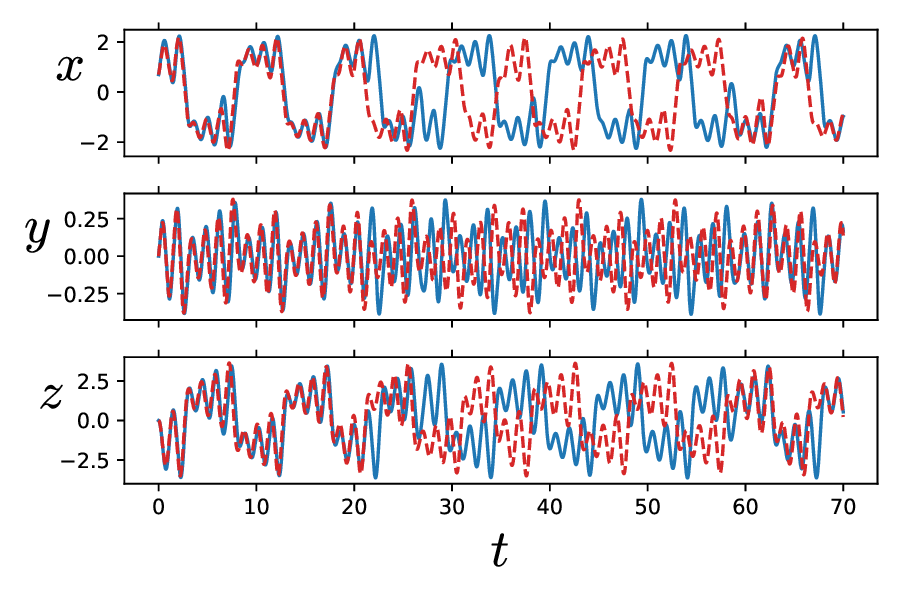}
        \caption{Time}
        \label{chua_time}
    \end{subfigure}
    \caption{Chua double-scroll. (a) Phase portrait showing training data (teal), true trajectory (solid blue), and trajectory generated by the identified model (dashed red).  (b) Time-series comparison from the same initial condition. There is good agreement for short times, 
    and the subsequent divergence is consistent with chaotic dynamics.}
    \label{fig:chua}
\end{figure*}

In particular, this model shows great tracking accuracy for about 10 Lyapunov time units as seen in \cref{chua_time} but diverges eventually. The statistics $\tilde{\bm \mu} = [-0.05,  0, 0.05]$ and $\tilde{\bm \sigma} = [1.39, 0.17, 1.75]$ are close to that of the true models: $\bm \mu = [0.02,  0, -0.02]$ and $\bm \sigma = [1.41, 0.17, 1.78].$ This agreement is partially explained by \cref{mu_error}--\cref{std_error} in ~\cref{sec:bounds_statistics}, which show that small trajectory error implies accurate estimation of statistical quantities over finite time horizons. Combined with \cref{l2_norm_eqv}, this ensures that the learned model reproduces short-term statistics accurately over the training window. In chaotic systems, trajectory errors grow exponentially beyond this window, so the RMSE is no longer small.  However, statistical quantities depend on the invariant measure of the dynamics rather than pointwise trajectory alignment. The observed agreement therefore indicates that the learned model approximates the invariant distribution of the attractor, leading to accurate long-time averages despite trajectory divergence.  
The Lyapunov exponents for the fit model, $\tilde{\bm \lambda}_e = [0.26, 0, -3.85]$ show decent agreement with those for the true model, $\bm \lambda_e = [0.44,  0,  -4.24]$.

For this example, the models seem to be estimated consistently even from noisy data ($5\%$ NSR) using longer NODE horizons. The following model:
\begin{subequations}
\begin{align}
\dot{x} &= 14 y + 1.9 \sin(\tfrac{5}{3} x + 0.143 y - 0.111 z) - \tfrac{1}{3} x + \tfrac{1}{8} z, \label{chua_5_x}\\
\dot{y} &= 0.9 x - 0.8 y + 0.9 z + 0.1 \sin(\tfrac{5}{3} x + 0.143 y - 0.111 z),\label{chua_5_y}\\
\dot{z} &= -27.125 y - \tfrac{1}{8} z - 0.2 (x - \sin(\tfrac{5}{3} x + 0.143 y - 0.111 z)).\label{chua_5_z}
\end{align}
\end{subequations}
was identified from $m^* = 2560+H$ data points using a longer horizon, $H=64$ time-steps, but a narrower network ($L=3$ with $105$ parameters vs $L=10$ with $322$ parameters). The statistics $\tilde{\bm \mu} = [0.01, 0, -0.01]$ and $\tilde{\bm \sigma} = [1.45,  0.183, 1.88]$ are in good agreement with the true model. Although the positive Lyapunov exponent in $\tilde{\bm \lambda}_e = [0.23, 0, -2.84]$ is smaller, it is similar to \cref{chua_main_x}-\cref{chua_main_z}, and a trajectory is illustrated in \cref{fig:chua_5}.

\begin{figure*}[h]
    \centering
    \begin{subfigure}[t]{0.49\textwidth}
        \centering
        \includegraphics[height=1.9in]{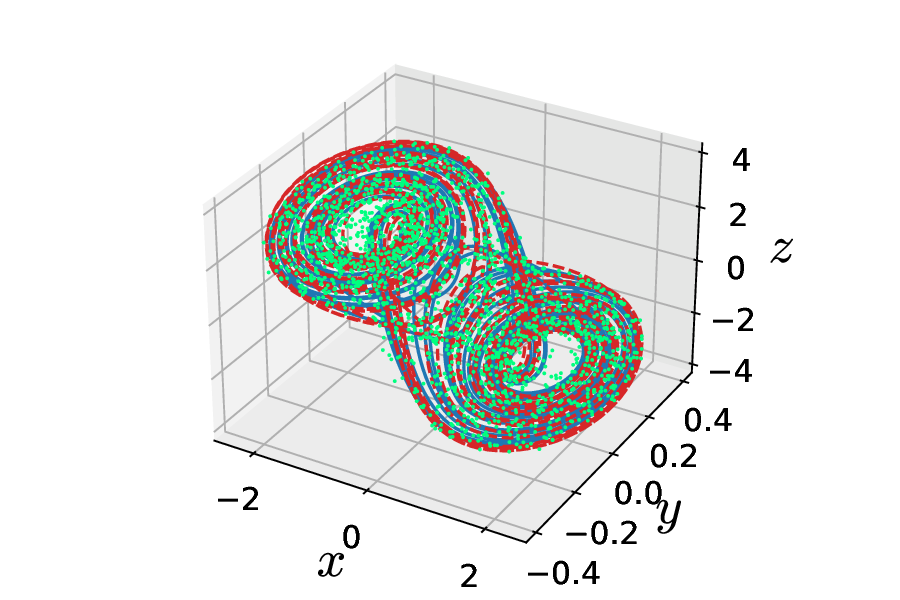}
        \caption{Space}
        \label{chua_5_space}
    \end{subfigure}%
    ~ 
    \begin{subfigure}[t]{0.49\textwidth}
        \centering
        \includegraphics[height=1.9in]{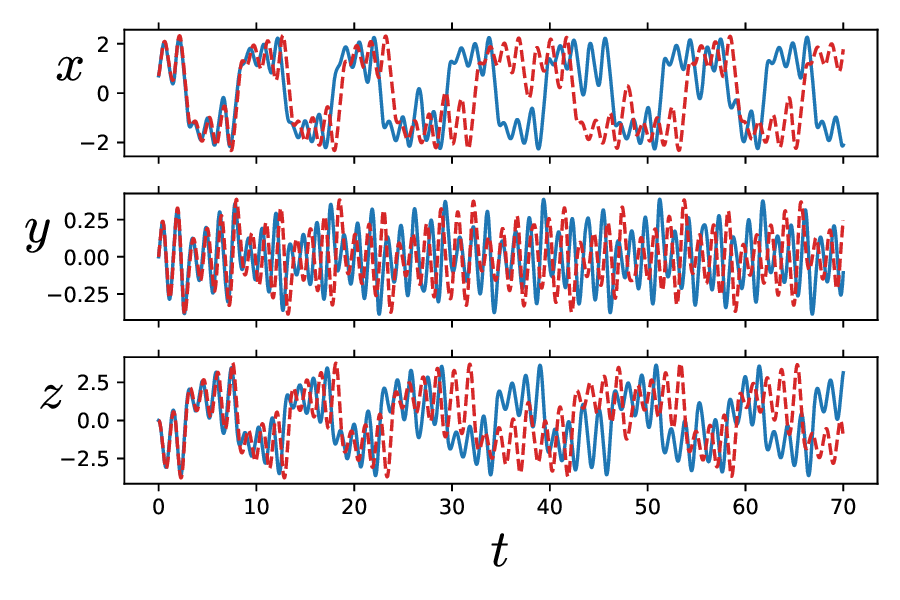}
        \caption{Time}
        \label{chua_5_time}
    \end{subfigure}
    \caption{Chua system with $5\%$ NSR (corresponding to additive $\sigma_{abs} \approx [0.07, 0.01, 0.09]$). (a) Training data in phase portrait (teal dots) seem a lot more scattered due to noise, but trajectory generated by the identified model (dashed red) and the true trajectory (solid blue) are close. (b) Time-series comparison from the same initial condition. There is good agreement for short times, 
    and the subsequent divergence is consistent with chaotic dynamics.}
    \label{fig:chua_5}
\end{figure*}

\subsubsection{Lorenz equations} \label{example:lorenz}

Finally, we consider the Lorenz system~\cite{lorenz1963deterministic}, originally derived to model two-dimensional atmospheric convection:
\begin{subequations}
\begin{align}
\dot{x} &= \sigma (y - x), \label{lorenz_x}\\
\dot{y} &= x(\rho - z) - y, \label{lorenz_y}\\
\dot{z} &= xy - \beta z, \label{lorenz_z}
\end{align}
\end{subequations}
with $(\sigma,\rho,\beta) = (10,28,8/3)$. Compared to previous chaotic examples, the Lorenz attractor has a much larger positive Lyapunov exponent. Consequently, this system is more sensitive to initial conditions -- thus errors accumulate faster -- and also happens to be quite sensitive to parametric changes   making making time-series forecasting and model estimation more challenging. When trained with $H=1$, the identified models can include non-Lipschitz terms such as inverse powers of the state variables. These models achieve low one-step error but are dynamically ill-posed, exhibiting singular behavior when states approach zero.
This illustrates a limitation of one-step losses: \textit{they do not penalize instability under composition} (\cref{accumulating_error}).

When the horizon is increased, the algorithm recovers the correct structure:
\begin{subequations}
\begin{align}
\dot{x} &= 10(y - x),\\
\dot{y} &= 28x - xz - y,\\
\dot{z} &= xy - 2.667 z, \label{lorenz_fit_z}
\end{align}
\end{subequations}
for example when we use $m^* = 2560+H$ data points, with $H=8$ time-steps, sampled from $p = 1$ trajectory over $T=[0,40]$, using $L = 3$ layers ($105$ parameters) after $3200$ epochs.  Also noteworthy is that despite the increased computational cost, the algorithm estimates these models in $25-50\%$ of the epochs in this case - and generally converging to equally good models in fewer epochs (although those epochs may take longer time to compute due to backpropagation-through-time (BPTT) over multiple time steps). This model is compared to the true model in \cref{Lorenz_exact}. The training data is noise free so and the estimated model is \textit{almost} exact - differing only in one parameter $8/3$ in \cref{lorenz_z} getting rounded off to $2.667$ in \cref{lorenz_fit_z} - so the estimated trajectory diverges from the true trajectory gradually over time.

\begin{figure*}[h]
    \centering
    \begin{subfigure}[t]{0.49\textwidth}
        \centering
        \includegraphics[height=1.9in]{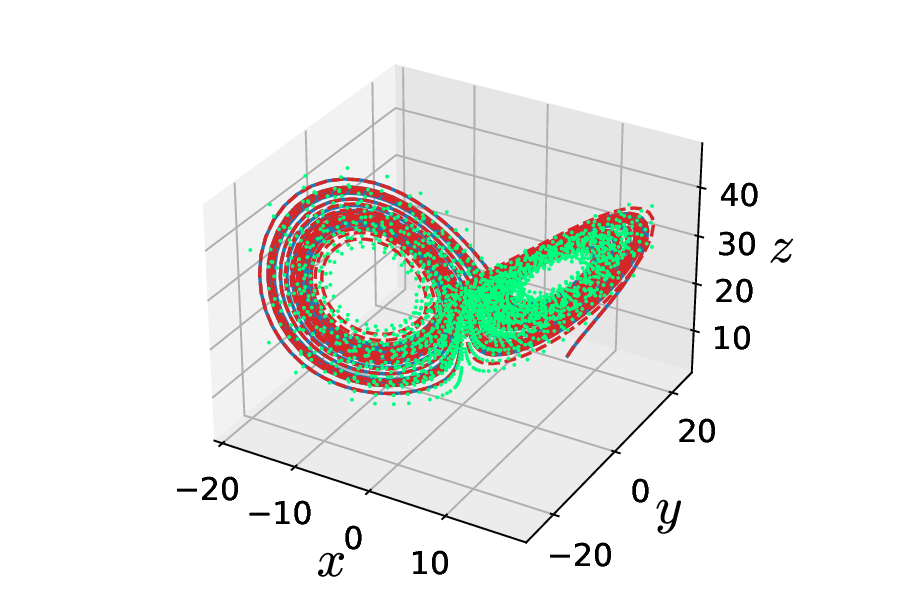}
        \caption{Space}\label{lorenz_exact_space}
    \end{subfigure}%
    ~ 
    \begin{subfigure}[t]{0.49\textwidth}
        \centering
        \includegraphics[height=1.9in]{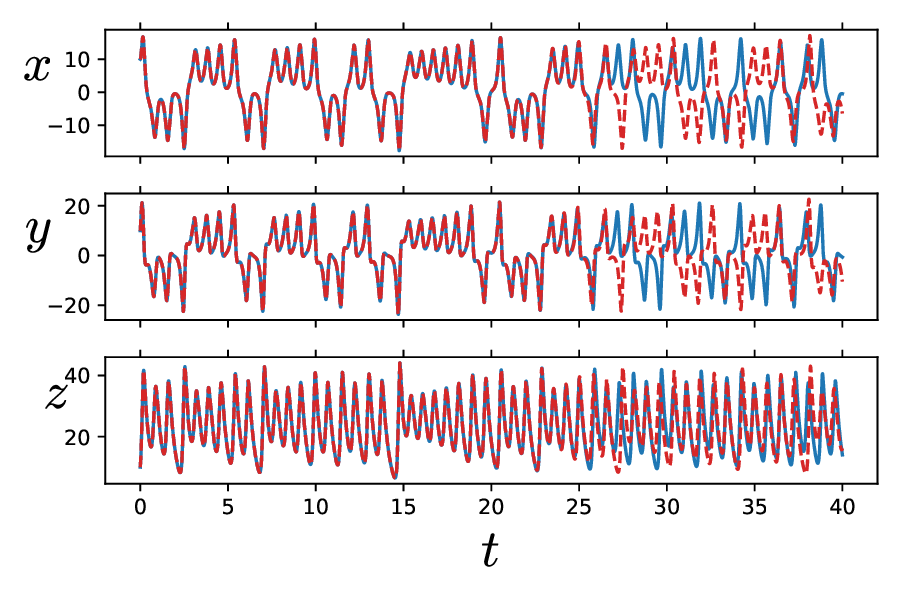}
        \caption{Time}\label{lorenz_exact_time}
    \end{subfigure}
    \caption{Lorenz system without noise. (a) Phase portrait showing training data (teal), true trajectory (solid blue), and trajectory generated by the identified model (dashed red).  (b) Time-series comparison, showing exact agreement between the exact and fit model.}
    \label{Lorenz_exact}
\end{figure*}
Consequently, the statistics $\tilde{\bm \mu} = [-0.03,  -0.03, 23.5]$ and $\tilde{\bm \sigma} = [7.91, 9.01, 8.65]$ are nearly identical to the true $\bm \mu = [ 0.07,  0.07, 23.5]$ and $\bm \sigma = [7.92, 9.01, 8.64]$, and so are the Lyapunov exponents $\tilde{\bm \lambda}_e = [0.91, 0, -14.6]$ vs $\bm \lambda_e = [0.91,  0,  -14.6].$ This behavior is partially explained by \cref{mu_error}--\cref{std_error}: over finite time horizons where the trajectory error remains small (as ensured by \cref{l2_norm_eqv}), the time-averaged quantities are also accurately captured. In chaotic systems, however, trajectory errors grow exponentially beyond this horizon, so the RMSE is no longer small and these bounds no longer directly apply. The continued agreement in $\bm \mu$ and $\bm \sigma$ therefore reflects a different mechanism: the learned model accurately reproduces the invariant measure of the attractor. Consequently, time-averaged quantities remain stable even as individual trajectories diverge. These results illustrate an important point:
minimizing a multi-step loss promotes models that are stable under composition and statistically accurate, even when exact trajectory matching is impossible due to chaos.

We found that model estimation for the Lorenz equations from noisy data (as defined in \cref{noise}) is challenging. We believe this is partly due to the relative noise used that's taken proportional to the dynamic range in the state data which is higher for the Lorenz attractor than the other examples presented. Additionally, the Lorenz attractor has larger positive Lyapunov exponent ($\lambda_e \approx 0.91$) than the \rossler ($\lambda_e \approx 0.13$) and Chua ($\lambda_e \approx 0.44$) attractors. For instance, consider a similar dataset as used before but with $5\%$ NSR. We found that at least $H=32$ time steps are required to estimate models that give desirable errors. For example, our algorithm estimated
\begin{subequations}
\begin{align}
  \dot{x} &= 10  y - 9.5  x - 1.667, \label{lorenz_5_x}\\
    \dot{y} &= 27  x - x  z - 0.667 y + 0.667, \label{lorenz_5_y}\\
    \dot{z} &= x  y - 2.667  z + 0.333  y + 0.667, \label{lorenz_5_z}
\end{align}
\end{subequations}
from $m^* = 2560+H$ data points using $H=128$ time-steps sampled from a single time-series over $T=40$ time units, corrupted by $5\%$ NSR, after $19600$ epochs. The statistics and Lyapunov exponents are: $\tilde{\bm \mu} = [ 0.08, 0.07, 23.3]$, $\tilde{\bm \sigma} = [7.97,  8.71, 8.16]$, and $\tilde{\bm \lambda}_e = [0.92, 0, -13.7]$. Time evolution of this estimated model is compared with that of the true model in \cref{lorenz_5}. From \cref{lorenz_5_space}, we can see that the estimated attractor (dashed red) is visually close to the true attractor (solid blue). The training data (teal dots) lying on it looks pretty scattered because of the noise, hinting at the difficulty in system identification. We can see that the estimated time-series (dashed red) tracks the true trajectory (solid blue) for more than 2 Lyapunov time-units. The trajectory diverges earlier than the aforementioned models due to larger difference in the vector field.

However, we would like to bring the readers' attention to the terms in \cref{lorenz_5_x}-\cref{lorenz_5_z}. They have the same dependencies on the state variables $x$, $y$ and $z$ as that in \cref{lorenz_x}-\cref{lorenz_z} but vary in the coefficients and additional constants. The constants here are particularly noteworthy since we found that such small spurious constants appear in most of the models estimated from data from the Lorenz equations corrupted with higher noise. These deviations can be interpreted as a small model perturbation $\g(\x)$ in the sense of \cref{sec:bounds}. 
While the bounds there apply to systems with stable fixed points or periodic orbits, they illustrate a general principle: small perturbations in the vector field lead to controlled deviations in trajectories over finite times.
The appearance of small constant terms is also consistent with the structure of $\mathrm{Loss}_H$, which averages errors over time steps. 
Small biases can produce only modest increases in the averaged multi-step loss, especially when compared to errors that grow exponentially under composition. 
As a result, such terms are weakly penalized unless explicitly regularized.
We posit that models estimated by minimizing the point-wise error averaged over a time-horizon are invariant to such constants as long as the other terms in the estimated model are exact.
\begin{figure*}[h]
    \centering
    
    \begin{subfigure}[t]{0.49\textwidth}
        \centering
        \includegraphics[height=1.9in]{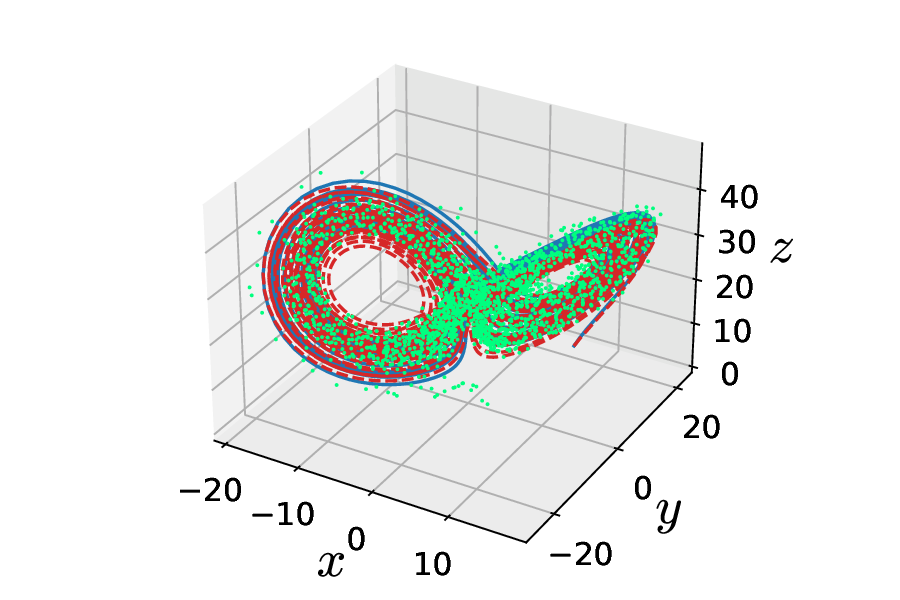}
        \caption{Space}
        \label{lorenz_5_space}
    \end{subfigure}%
    ~ 
    \begin{subfigure}[t]{0.49\textwidth}
        \centering
        \includegraphics[height=1.9in]{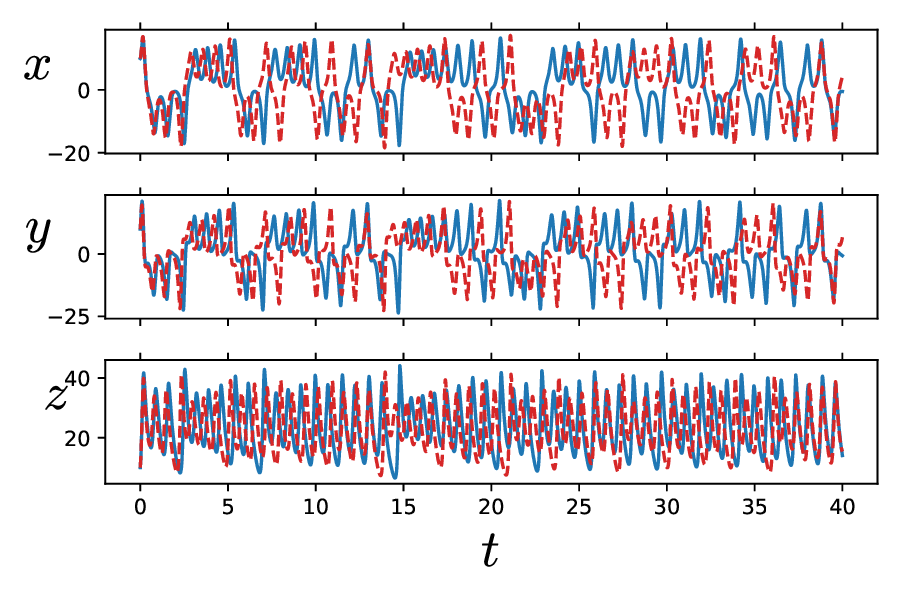}
        \caption{Time}
        \label{lorenz_5_time}
    \end{subfigure}
    \caption{Lorenz $5\%$ NSR (corresponding to additive $\sigma_{abs} \approx [0.41, 0.45, 1.28]$). (a) Phase portrait showing training data (teal), true trajectory (solid blue), and trajectory generated by the identified model (dashed red).  (b) Time-series comparison from the same initial condition. There is good agreement for short times, 
    and the subsequent divergence is consistent with chaotic dynamics.}
    \label{lorenz_5}
\end{figure*}

As we corrupt the training data with more noise, we find that the spurious terms increase even more. For example, we estimated the following model from $m^* = 2560+H$ data points using $H=256$ time-steps with $10\%$ NSR using similar hyperparameters after $12800$ epochs:
\begin{subequations}
\begin{align}
  \dot{x} &= 10 \; y - 9.667 \; x - 2, \label{lorenz_10_x}\\
    \dot{y} &= 24.5 \; x - x \z - 0.333\; y - 0.333, \label{lorenz_10_y}\\
    \dot{z} &= x \; y - 2.667 \; z - 0.333 \; y + 0.667 \;x + 0.333. \label{lorenz_10_z}
\end{align}
\end{subequations}
The statistics and Lyapunov exponents are: $\tilde{\bm \mu} = [ 0.03, 0.05, 21.2]$, $\tilde{\bm \sigma} = [7.59,  8.37, 7.78]$, and $\tilde{\bm \lambda}_e = [0.89, 0, -13.6]$, and a trajectory is illustrated in \cref{lorenz_10}.

\begin{figure*}[h]
    \centering
    \begin{subfigure}[t]{0.49\textwidth}
        \centering
        \includegraphics[height=1.9in]{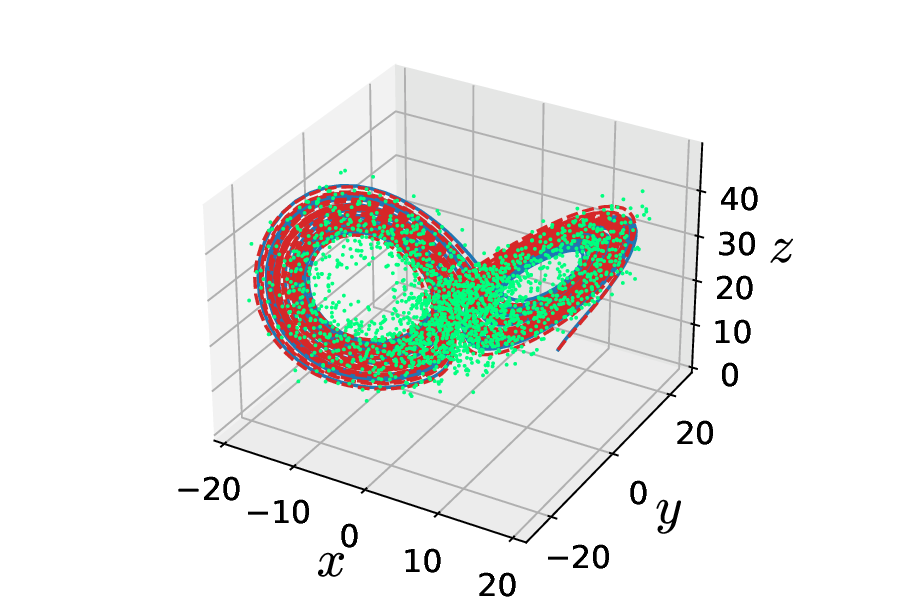}
        \caption{Space}
        \label{lorenz_10_space}
    \end{subfigure}%
    ~ 
    \begin{subfigure}[t]{0.49\textwidth}
        \centering
        \includegraphics[height=1.9in]{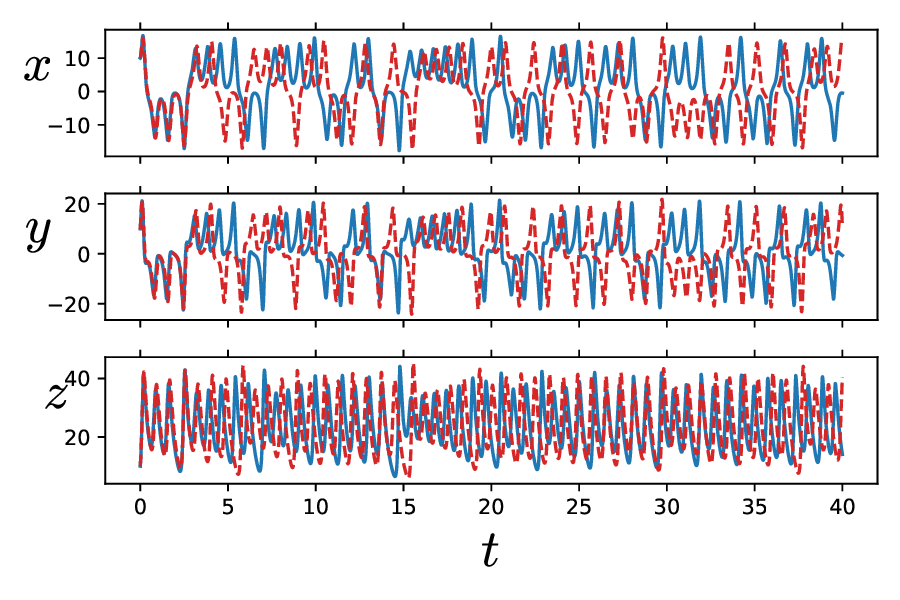}
        \caption{Time}
        \label{lorenz_10_time}
    \end{subfigure}

    \caption{Lorenz $10\%$ NSR (corresponding to additive $\sigma_{abs} \approx [0.79, 0.92, 2.46]$). (a) Phase portrait showing training data (teal), true trajectory (solid blue), and trajectory generated by the identified model (dashed red).  (b) Time-series comparison from the same initial condition. There is good agreement for short times, 
    and the subsequent divergence is consistent with chaotic dynamics.}
    \label{lorenz_10}
\end{figure*}

We would like the readers to note that estimating models that fit datasets with more noise is possible but our algorithm performs more akin to a conventional perceptron or dictionary based NODE without the desired sparsity. The estimated trajectories from those models are satisfactory but the models lose their sparsity. However, in that case we redirect the readers \cref{sec:bounds_statistics} where we show that the statistics of the estimated models are ``close'' to the true statistics.  This is not just for the optimization, but also for the models selected by the mAIC.

\medskip

Across the R\"ossler, Chua, and Lorenz examples, we observe a consistent pattern:
(i) minimizing $\mathrm{Loss}_H$ enforces accurate trajectory tracking over finite horizons (\cref{l2_norm_eqv}),
(ii) chaotic sensitivity leads to divergence beyond this horizon,
and (iii) statistical quantities such as $\bm \mu$ and $\bm \sigma$ remain accurate: over finite horizons this is guaranteed by \cref{mu_error}--\cref{std_error}, while over longer times it indicates that the learned models capture the invariant measure of the underlying attractor.
This illustrates how the Loss$_H$ framework balances short-term trajectory accuracy with long-term statistical fidelity.

\section{Discussion and Summary}
\label{sec:conclusions}

The central theme of this work is that learning dynamical systems from data requires enforcing consistency under composition, rather than merely fitting local, one-step behavior. While many existing approaches focus on minimizing one-step prediction error or directly regressing the vector field, such objectives do not guarantee that the learned dynamics behave correctly when iterated. In contrast, the multi-step loss considered here explicitly penalizes errors that accumulate over time, thereby favoring models that remain accurate under repeated composition.

This perspective leads to a fundamental distinction between local approximation and dynamical fidelity. A model may achieve small one-step error while still producing qualitatively incorrect trajectories when iterated, particularly in nonlinear or chaotic regimes. By incorporating multiple prediction steps into the training objective, the Loss$_H$ framework enforces agreement between predicted and true trajectories over a finite horizon. \cref{l2_norm_eqv} formalizes this connection by showing that minimizing the multi-step loss controls the average trajectory error over that horizon. At the same time, the comparison between $\mathrm{Loss}_H$ and $\mathrm{Loss}_{H+1}$ highlights how increasing the horizon penalizes models whose errors grow under iteration, thereby promoting stability.

These ideas are especially important for chaotic systems, where sensitivity to initial conditions makes long-term trajectory prediction inherently impossible. In this setting, the goal of model identification must be carefully interpreted. Rather than expecting trajectory-level agreement for all time, it is more appropriate to require (i) accurate reproduction of trajectories over a finite horizon and (ii) faithful recovery of long-term statistical properties. The results in this work show that multi-step training achieves both objectives: it controls trajectory error over the training window and produces models whose invariant statistics - such as means, standard deviations, and Lyapunov exponents - are in good agreement with those of the true system.
A key insight is that these two forms of accuracy are linked but not identical. The bounds in \cref{sec:bounds_statistics} show that small trajectory error implies accurate statistical estimates. However, in chaotic systems, trajectory error inevitably grows beyond the training horizon. The empirical results demonstrate that, despite this divergence, the learned models can still reproduce the correct statistical behavior. This reflects the fact that the multi-step loss enforces accuracy over finite windows distributed along trajectories, which is sufficient to capture invariant measures even when long-term trajectories diverge pointwise.

Another important observation is that multi-step training acts as an implicit regularizer. By penalizing error accumulation, it suppresses dynamically unstable or non-physical models that may fit the data locally but fail under iteration. This effect is particularly evident in the presence of noise, where longer optimization horizons improve robustness and lead to models that better capture the underlying structure of the dynamics. When combined with sparsity-promoting regularization, this yields parsimonious models that are both interpretable and dynamically consistent.

From a computational perspective, these benefits come at a cost. Multi-step training requires repeated numerical integration and backpropagation through compositions of the learned dynamics, leading to deeper computational graphs and increased training time. This trade-off between computational efficiency and dynamical accuracy is intrinsic to the problem. In practice, it can be managed through choices of horizon length, network architecture, and optimization strategy, suggesting that adaptive or staged training procedures may be beneficial.

Overall, the results of this work support the view that multi-step training provides a principled framework for learning dynamical systems, particularly in regimes where long-term behavior matters. By aligning the training objective with the compositional nature of dynamical systems, it is possible to obtain models that not only fit data locally, but also reproduce the qualitative and statistical features that define the underlying dynamics.

Several directions for future work naturally emerge from this framework. One important avenue is the development of adaptive training strategies in which the NODE horizon is selected dynamically, balancing computational cost with control of error growth. Extending the approach to settings with partial observations, including latent-variable and delay-embedding formulations, would further broaden its applicability. We also found that models estimated from noisy data may violate symmetries that the underlying dynamics possess, so equivariant or invariant neural network architectures could respect the symmetries of the system corrupted by noise. Finally, on the theoretical side, it would be valuable to strengthen the connection between multi-step training objectives and the accurate recovery of invariant quantities such as Lyapunov exponents and stationary distributions.

\appendix

\section{Bounds on Dynamics for Perturbed Systems}
\label{sec:bounds}

\subsection{Systems with Stable Fixed Points}\label{appendix:fixed_point}

Suppose \cref{sys} has a stable hyperbolic fixed point $\x^*$, so that the real parts of the eigenvalues of $A \equiv D \f(\x^*)$ are all negative.  Now consider 
\begin{equation}\label{sys_pert}
    \dot \x = \f(\x) + \g(\x),
\end{equation}
with $\g(\x)$ small, bounded, and smooth.  Let $\z(t)$ be the difference between solutions to \cref{sys_pert} and $\x^*$.   Then, to leading order
\begin{equation}
    \frac{d\z}{dt} = A z + \g(\x^*) + \cdots.
\end{equation}
Ignoring the higher-order terms,
\begin{equation}
    \z(t) = e^{A t} \z(0) + \int_0^t e^{A(t-\tau)} \g(\x^*) d \tau,
\end{equation}
so
\begin{equation}
    \|z(t) \| \le \left\|e^{A t} \right\| \|\z(0) \| + \int_0^t \left\|e^{A (t - \tau)} \right\| \| \g(\x^*) \| d \tau.
\end{equation}
We now claim that
\begin{equation}
    \left\| e^{A t} \right\| \le C e^{-\alpha t},
\end{equation}
where
\begin{equation}
    \alpha = - \max_j {\rm Re}(\lambda_j) > 0
\end{equation}
corresponds to the least unstable eigenvalue.
This follows by considering $\dot{\x} = A \x$.  Because $A$ is Hurwitz, there exist positive definite matrices $P$ and $Q$ such that
\[
A^T P + P A = -Q.
\]
Let $V(\x) = \x^T P \x$ be a Lyapunov function.  We find
\[
\dot{V} = -\x^T Q \x \le -\lambda_{\min} (Q) \| \x \|^2.
\]
Thus, $\dot{V} <0$ for all $\x$, except when $\x = 0$ where $\dot{V} = 0$.  Therefore, the system $\dot{\x} = A \x$ is asymptotically stable, so
\[
\| \x(t) \| \le C e^{-\alpha t} \| \x(0) \|.
\]
Absorbing the norm of the initial condition into the constant,
\[
\| e^{A t} \| \le C e^{-\alpha t},
\]
as desired.  Thus,
\[
\| \z(t) \| \le C e^{-\alpha t} \| \z(0) \| + C \int_0^t e^{-\alpha (t - \tau)} \| \g(\x^*) \| d \tau.
\]
Finally, supposing $\| \g (\x) \| \le G$ for all $\x$,
\[
\| \z(t) \| \le C e^{-\alpha t} \| \z(0) \| + C G \underbrace{\int_0^t e^{-\alpha (t - \tau)} d \tau}_H.
\]
But
\[
H = \frac{1}{\alpha} (1 - e^{-\alpha t}) \le \frac{1}{\alpha},
\]
so
\[
\| \z(t) \| \le C e^{-\alpha t} \| \z(0) \| + \frac{C G}{\alpha}.
\]
That is, solutions to \cref{sys_pert} asymptotically stay close to the fixed point of the unperturbed system \cref{sys}. 

\subsection{Systems with Stable Periodic Orbits}\label{appendix:limit_cycle}

We can make a similar argument regarding systems with a stable periodic orbit, but it will be useful to change to more convenient coordinates.
Suppose \cref{sys} has a stable hyperbolic periodic orbit $\x_\gamma(t)$ with period $T$.  Let $\lambda_1, \lambda_2, \cdots, \lambda_{n-1}$ be the non-trivial Floquet exponents of $\x_\gamma$, all with negative real parts.  For the perturbed vector field \cref{sys_pert}, we can do an augmented phase reduction to phase coordinates $\theta(\x)$ and isostable coordinates $y_j(\x)$, $j = 1,\cdots,n-1$, giving~\cite{wils16}
\begin{eqnarray}
    \frac{d \theta}{dt} &=& \omega + \Z(\theta) \cdot \g (\x_\gamma(\theta)) + \cdots, \label{phase_reduction} \\
    \frac{ d y_j}{dt} &=& \lambda_j y_j + \I(\theta) \cdot \g(\x_\gamma(\theta)) + \cdots, \quad j = 1,\cdots, n-1, \label{isostable_reduction}
\end{eqnarray}
where $\omega = 2 \pi / T$ and 
\begin{equation}
    \Z(\theta) = \left. \nabla \theta \right|_{\x_\gamma}, \qquad \I_j(\theta) = \left. \nabla \y_j \right|_{\x_\gamma}
\end{equation}
are the phase response curve and isostable response curves, respectively.
The isostable coordinates are useful for describing the dynamics transverse to the periodic orbit. 

Now, defining $\y = (y_1,\cdots,y_{n-1})$ and $\Lambda = {\rm diag}(\lambda_1,\cdots,\lambda_{n-1})$, \eqref{isostable_reduction} becomes
\begin{equation}
\dot{\y} = \Lambda \y + \b(\theta) + \cdots, \qquad b_j(\theta) \equiv \I_j(\theta) \cdot \g(\x_\gamma(\theta)),\qquad  j = 1,\cdots,n-1.
\end{equation}
To leading order, we see that
\begin{equation}
    \y(t) = e^{\Lambda t} \y(0) + \int_0^t e^{(t - \tau)} \; \b(\theta(\tau)) \; d \tau.
\end{equation}
We note that
\[
e^{\Lambda t} = {\rm diag} \left( e^{\lambda_1 t}, \cdots,e^{\lambda_{n-1} t} \right),
\]
and let
\[
\alpha = -\max_j {\rm Re}(\lambda_j) > 0
\]
correspond to the least stable eigenvalue.  Since $\| e^{\lambda_j t}\| \le e^{-\alpha t}$ for all $j = 1,2,\cdots,n-1$, $\| e^{\Lambda t}\| \le C e^{-\alpha t}$ for some constant $C>0$.  Thus
\[
\| e^{\Lambda t} \y(0) \| \le \| e^{\Lambda t}\| \;\|\y(0)\| \le C e^{-\alpha t} \|\y(0)\|.
\]
Next, we find the bound
\begin{eqnarray}
    \left\| \int_0^t e^{\Lambda (t - \tau)} \; \b(\theta(\tau) \; d \tau \right\| &\le& \int_0^t \left\| e^{\Lambda (t-\tau)} \; \b(\theta(\tau)) \right\| \; d \tau \nonumber \\
    & \le & \int_0^t \left\| e^{\Lambda (t-\tau)} \right\| \cdot \| \b(\theta(\tau) \| d \tau \nonumber \\
    &\le& C \int_0^t e^{-\alpha (t - \tau)} \; \| \b(\theta(\tau)) \| \; d \tau.
\end{eqnarray}
Now, suppose $\| \b(t) \| \le B$ for all $t$, which is a measure of the (bounded) size of the perturbation $\g(\x)$.  Then
\[
\int_0^t e^{-\alpha (t - \tau)} \; \| \b(\theta(\tau)) \| \; d \tau \le B \int_0^t e^{-\alpha (t - \tau)} d \tau = \frac{B}{\alpha} (1 - e^{-\alpha t}) \le \frac{B}{\alpha}.
\]
Putting this all together,
\begin{equation}
    \|y(t) \| \le C e^{-\alpha t} \| \y(0) \| + \frac{C B}{\alpha}.
\end{equation}
This bound shows that the asymptotic solution to the perturbed system \cref{sys_pert} remains in an ${\cal O}(\| g \|)$ tube close to the periodic orbit for the unperturbed system \cref{sys}.  However, the solutions for \cref{sys_pert} can drift along the neutrally stable phase direction of $\x_\gamma$; from \cref{phase_reduction}, this will be at a rate ${\cal O}(\| g\|)$. 
We note that the bounds that we have derived are consistent with the notion of persistence of hyperbolic sets under perturbations, namely that a small perturbation to a vector field leads to a small change in the attractor, provided non-degeneracy conditions are valid. 

These results are illustrated in \cref{approx_figs_2} for some models estimated for the Hopf normal form:
\begin{equation}\label{hopf_1}
    \dot{r} = 0.5 - (\nicefrac{1}{3}) r^{11/3}, \quad
    \dot{\theta} = 1,
\end{equation}
visualized in \cref{approx_hopf_1} or,
\begin{equation}\label{hopf_2}
    \dot{r} = 0.36 - 0.1 r^4(1 + r^{-0.1} + 2 r^{0.1}), \quad
    \dot{\theta} = r/4 + 0.8,
\end{equation}
visualized in \cref{approx_hopf_2}, and the Van der Pol oscillator:
\begin{equation}\label{vdp_1}
    \dot{x} = y, \quad
    \dot{y} = 15y - x - 14 x^2y + \sin(x-1),
\end{equation}
visualized in \cref{approx_VDP_1} or,
\begin{equation}\label{vdp_2}
\begin{split}
    \dot{x} = y, \quad \dot{y} &=  14.5 y - 14 x^2 y + 0.5 ((x^2 y - 1)(\sgn(x) - e^{-x}) -  \sin(x^2/2))\\
    & \quad - 0.125 |x|^{1.5} |x + 1| |\sgn(x) - e^{-x}|,
\end{split}
\end{equation}
visualized in \cref{approx_VDP_2}. We have verified numerically that these perturbed vector fields remain close to the true vector fields in a neighborhood of the periodic orbit, except for short-lived deviations for the van der Pol equations when the trajectory has very fast dynamics.  This is consistent with $\|\g(\x)\|$ remaining small, which means that the asymptotic solutions to the perturbed systems \cref{sys_pert} remain close to the periodic orbit for the unperturbed system.

\section{Heuristic analysis of trajectory error growth}
\label{appendix:trajectory_tracking}

The following derivation is heuristic and intended to illustrate how
modeling errors can accumulate under repeated composition of the learned
dynamics. In particular, it provides intuition for why multi-step
prediction errors may grow with the forecast horizon in unstable or
chaotic systems.
Let
\[
\dot{\x} = \f(\x),
\qquad
\dot{\tilde{\x}} = \tilde{\f}(\tilde{\x}),
\]
where
\[
\tilde{\f}(\x) = \f(\x) + \g(\x),
\]
and $\g$ represents the modeling error.
Assume that the true and predicted trajectories are initialized from the
same initial condition:
\[
\x(0) = \tilde{\x}(0) = \x_0,
\]
and define the trajectory error
\[
\delta(t) := \tilde{\x}(t) - \x(t).
\]
Then
\begin{equation}
\delta(t) = \int_0^t \left( \tilde{\f}(\tilde{\x}(\tau)) - \f(\x(\tau)) \right) \,d\tau.
\end{equation}
Substituting $\tilde{\f} = \f + \g$ gives
\begin{equation}
\delta(t) = \int_0^t \left( \f(\tilde{\x}(\tau)) - \f(\x(\tau)) + \g(\tilde{\x}(\tau)) \right) \,d\tau.
\end{equation}
Using $\tilde{\x} = \x + \delta$ and formally linearizing for small
$\delta$,
\begin{equation}
\f(\x+\delta) \approx \f(\x) + \nabla \f(\x)\delta,
\end{equation}
\begin{equation}
\g(\x+\delta) \approx \g(\x) + \nabla \g(\x)\delta.
\end{equation}
Neglecting higher-order terms yields the approximate error dynamics
\begin{equation}
\delta(t) \approx \int_0^t \left( \nabla \f(\x(\tau)) + \nabla \g(\x(\tau)) \right)\delta(\tau) + \g(\x(\tau)) \,d\tau.
\label{difference_norm}
\end{equation}
Taking norms and applying the triangle inequality,
\begin{equation}
\|\delta(t)\| \lesssim \int_0^t \|\nabla \f(\x(\tau))\|\,\|\delta(\tau)\| \,d\tau
+
\int_0^t \|\nabla \g(\x(\tau))\|\,\|\delta(\tau)\| \,d\tau
+
\int_0^t \|\g(\x(\tau))\| \,d\tau.
\end{equation}
This expression illustrates two mechanisms contributing to trajectory
divergence:
\begin{enumerate}
\item direct forcing from the modeling error $\g$, and
\item amplification of existing trajectory discrepancies through the
linearized dynamics.
\end{enumerate}
In unstable or chaotic systems, repeated composition of the learned flow
map can amplify small local modeling errors, leading to growth of
multi-step prediction errors over longer horizons.

\section{Bounds on Statistics} \label{sec:bounds_statistics}

In this appendix, we show that pointwise trajectory error, measured via RMSE, provides explicit control over first and second order statistics.

We denote the true data points as $\x(j) \in \R^n$ and the data points generated by the identified model as $\tilde{\x}(j) \in \R^n$. Then we consider the error between the true and predicted trajectories by averaging the point-wise $\ell_2$ distance over the number of data points using the Root Mean Squared Error (RMSE) as
\begin{equation}\label{rmse}
    \mathrm{RMSE} = \sqrt{\frac{1}{m} \sum_{j = 1}^m \| \x (j) - \tilde \x(j)\|^2_2},
\end{equation}
where $\|\cdot\|_2$ is the point-wise $\ell^2$ norm defined as 
\begin{equation}\label{2_norm}
    \| \x \|_2 = \sqrt{\sum_{i = 1}^n x_i^2}.
\end{equation}

A common way to quantify the performance of estimated models for chaotic dynamics is through first and second order statistics. We define the state-wise mean $\bm \mu$ of a dataset $\{\x(j)\}_{j = 1}^m$ as
\begin{equation}\label{mean}
    \mu_i = \frac{1}{m} \sum_{j = 1}^m x_i(j),
\end{equation}
and the state-wise standard deviation $\bm \sigma$ as
\begin{equation}\label{stdev}
    \sigma_i = \sqrt{\frac{1}{m} \sum_{j=1}^m (x_i(j) - \mu_i)^2},
\end{equation}
with corresponding predicted quantities $\tilde{\bm \mu}$ and $\tilde{\bm \sigma}$.

\begin{proposition}\label{mu_error}
The RMSE between a trajectory generated by the actual model and a trajectory generated by an estimated model gives an upper bound on the $\ell^2$ error between the averages of those trajectories:
\begin{equation}
\|\bm \mu - \tilde{\bm \mu}\|_2
\le
\mathrm{RMSE}.
\end{equation}
\end{proposition}
\begin{proof}
We write
\begin{equation}
\bm \mu - \tilde{\bm \mu}
=
\frac{1}{m} \sum_{j=1}^m \big(\x(j) - \tilde{\x}(j)\big).
\end{equation}
Using the inequality
\begin{equation}
\left\| \frac{1}{m} \sum_{j=1}^m \mathbf{v}_j \right\|_2^2
\le
\frac{1}{m} \sum_{j=1}^m \|\mathbf{v}_j\|_2^2,
\end{equation}
which follows from Jensen's inequality applied to the convex function $\|\cdot\|_2^2$,
we obtain
\begin{equation}
\|\bm \mu - \tilde{\bm \mu}\|_2^2
\le
\frac{1}{m} \sum_{j=1}^m \|\x(j) - \tilde{\x}(j)\|_2^2.
\end{equation}
Taking square roots yields the result.
\end{proof}

% \medskip

% \medskip
\begin{proposition}\label{std_error}
\textbf{(Control of standard deviation under bounded trajectories).}
Suppose that the true and predicted trajectories satisfy a uniform bound
\[
\|\x(j)\|_2 \le M, \quad \|\tilde{\x}(j)\|_2 \le \tilde{M}
\quad \text{for all } j=1,\dots,m.
\]
Then the difference between the variances satisfies
\begin{equation}
\|\bm \sigma^2 - \tilde{\bm \sigma}^2\|_2 \le 2(M+\tilde{M}) \, \mathrm{RMSE}.
\end{equation}
Moreover, the difference between the standard deviations satisfies
\begin{equation}
\|\bm \sigma - \tilde{\bm \sigma}\|_2^2
\le
C \, \mathrm{RMSE},
\end{equation}
where $C = 2\sqrt{n}(M+\tilde{M}) >0$ depends only on $M$ and $n$ (and not on $m$).
\end{proposition}

% \medskip

% \noindent
\begin{proof}
\textit{Decomposing the variance.}

For each component $i = 1, \cdots, n$,
\begin{equation}
\sigma_i^2 = \frac{1}{m} \sum_{j=1}^m (x_i(j) - \mu_i)^2
=
\frac{1}{m}\sum_{j=1}^m x_i(j)^2 - \mu_i^2,
\end{equation}
\begin{equation}
\tilde{\sigma}_i^2 = \frac{1}{m} \sum_{j=1}^m (\tilde x_i(j) - \tilde \mu_i)^2
=
\frac{1}{m}\sum_{j=1}^m \tilde{x}_i(j)^2 - \tilde{\mu}_i^2.
\end{equation}
Thus,
\begin{equation}
\sigma_i^2 - \tilde{\sigma}_i^2
=
\underbrace{
\frac{1}{m}\sum_{j=1}^m \big(x_i(j)^2 - \tilde{x}_i(j)^2\big)
}_{A_i}
-
\underbrace{
\big(\mu_i^2 - \tilde{\mu}_i^2\big)
}_{B_i}.
\end{equation}

% \medskip
% 
\textit{Bounding the quadratic term $A_i$.}

\noindent
Using $a^2 - b^2 = (a-b)(a+b)$,
\begin{equation}
|x_i(j)^2 - \tilde{x}_i(j)^2|
=
|x_i(j) - \tilde{x}_i(j)| \, |x_i(j) + \tilde{x}_i(j)|.
\end{equation}
By boundedness,
\[
|x_i(j) + \tilde{x}_i(j)| \le (M+\tilde{M}),
\]
so
\begin{equation}
|x_i(j)^2 - \tilde{x}_i(j)^2|
\le
(M+\tilde{M})\, |x_i(j) - \tilde{x}_i(j)|.
\end{equation}
Therefore,
\begin{equation}
|A_i| = \left| \frac{1}{m} \sum_{j=1}^m (x_i(j)^2 - \tilde{x}_i(j)^2) \right| \le \frac{1}{m} \sum_{j=1}^m |x_i(j)^2 - \tilde{x}_i(j)^2| \le \frac{M+\tilde{M}}{m} \sum_{j=1}^m |x_i(j) - \tilde{x}_i(j)|.
\label{Ai}
\end{equation}
Now, we apply the Cauchy-Schwarz inequality $|\langle u,v \rangle| \le \| u \|_2 \| v \|_2$ with
\[
u_j = |x_i(j) - \tilde{x}_i(j)|, \qquad v_j = 1.
\]
Here
\begin{equation}
\langle u,v \rangle = \sum_{j=1}^m |x_i(j) - \tilde{x}_i(j)|, \qquad 
\| u \|_2 = \left( \sum_{j=1}^m |x_i(j) - \tilde{x}_i(j)|^2 \right)^{1/2}, \qquad \|v\|_2 = \sqrt{m},
\end{equation}
giving
\[
\sum_{j=1}^m |x_i(j) - \tilde{x}_i(j)|
\le
\sqrt{m} \left( \sum_{j=1}^m |x_i(j) - \tilde{x}_i(j)|^2 \right)^{1/2}.
\]
Substituting into \cref{Ai} gives
\begin{equation}
|A_i| \le (M+\tilde{M}) \left( \frac{1}{m} \sum_{j=1}^m (x_i(j) - \tilde{x}_i(j))^2 \right)^{1/2}.
\end{equation}

\medskip

% \noindent
\textit{Bounding the mean term $B_i$.}

\noindent
Using $a^2 - b^2 = (a-b)(a+b)$,
\begin{equation}
|\mu_i^2 - \tilde{\mu}_i^2|
=
|\mu_i - \tilde{\mu}_i| \, |\mu_i + \tilde{\mu}_i|.
\end{equation}
Since $|\mu_i| \le M, |\tilde{\mu}_i| \le \tilde{M}$,
\[
|\mu_i + \tilde{\mu}_i| \le (M+\tilde{M}),
\]
so
\begin{equation}
|B_i| \le (M+\tilde{M})\, |\mu_i - \tilde{\mu}_i|
= \frac{M+\tilde{M}}{m} \left| \sum_{j=1}^m (x_i(j) - \tilde{x}_i(j)) \right| \le \frac{{M+\tilde{M}}}{m} \sum_{j=1}^m |x_i(j) - \tilde{x}_i(j)|.
\end{equation}
Applying Cauchy-Schwarz again,
\begin{equation}
|B_i| \le (M+\tilde{M}) \left( \frac{1}{m} \sum_{j=1}^m (x_i(j) - \tilde{x}_i(j))^2 \right)^{1/2}.
\end{equation}

% \noindent
\textit{Combining bounds.}
We have,
\begin{equation}
|\sigma_i^2 - \tilde{\sigma}_i^2| \le |A_i| + |B_i|
\le
2(M+\tilde{M}) \left(\frac{1}{m} \sum_{j=1}^m (x_i(j) - \tilde{x}_i(j))^2\right)^{1/2}.
\end{equation}
Squaring both sides,
\begin{equation}
|\sigma_i^2 - \tilde{\sigma}_i^2|^2 \le \frac{4(M+\tilde{M})^2}{m} \sum_{j=1}^m (x_i(j) - \tilde{x}_i(j))^2,
\end{equation}
so, summing over $i=1,\dots,n$,
\begin{equation}
\|\bm \sigma^2 - \tilde{\bm \sigma}^2\|_2^2 \le \frac{4(M+\tilde{M})^2}{m} \sum_{j=1}^m \| \x(j) - \tilde{\x}(j) \|_2^2. 
\end{equation}
Thus, 
\begin{equation}
\|\bm \sigma^2 - \tilde{\bm \sigma}^2\|_2
\le
2(M+\tilde{M}) \, \mathrm{RMSE}.
\label{variance_bound}
\end{equation}

% \noindent
\textit{Pass from variance to standard deviation:}

\noindent
For nonnegative $a,b$, we have
\[
(\sqrt{a} + \sqrt{b})^2 = a + b + 2 \sqrt{a b} \ge a + b \ge |a - b|.
\]
Thus
\[
|\sqrt{a} - \sqrt{b}|^2
=
\frac{|a-b|^2}{(\sqrt{a} + \sqrt{b})^2} \le \frac{|a-b|^2}{|a-b|} = |a - b|.
\]
Applying this with $a=\sigma_i^2$, $b=\tilde{\sigma}_i^2$ gives
\begin{equation}
|\sigma_i - \tilde{\sigma}_i|^2
\le
|\sigma_i^2 - \tilde{\sigma}_i^2|.
\end{equation}
Summing over $i = 1, \cdots, n$,
\begin{equation}
\|\bm \sigma - \tilde{\bm \sigma}\|_2^2
\le
\|\bm \sigma^2 - \tilde{\bm \sigma}^2\|_1.
\end{equation}
Using the equivalence of norms result $\|v\|_1 \le \sqrt{n}\|v\|_2$,
\begin{equation}
\|\bm \sigma - \tilde{\bm \sigma}\|_2^2
\le
\sqrt{n}\,\|\bm \sigma^2 - \tilde{\bm \sigma}^2\|_2.
\end{equation}
Finally, applying \cref{variance_bound},
\begin{equation}
\|\bm \sigma - \tilde{\bm \sigma}\|_2^2
\le
2\sqrt{n}(M+\tilde{M}) \, \mathrm{RMSE},
\end{equation}
for a constant $C = 2\sqrt{n}(M+\tilde{M})$ depending only on $M$, $\tilde{M}$, and $n$.
This weaker bound arises because the mapping from variance to standard deviation involves a square root, which is not Lipschitz near zero. As a result, even linear control of the variance error yields only quadratic control of the standard deviation error without additional assumptions.
\end{proof}
% \hfill $\square$

% \medskip

% \noindent
The above two propositions show that minimizing the pointwise RMSE also minimizes the difference in first and second order statistics without having to explicitly minimize them. In particular, accurate trajectory reconstruction implies accurate estimation of time-averaged means and variances, which are commonly used to characterize chaotic dynamics. Alongside providing explicit upper bounds, these results also avoid the need to directly penalize statistical quantities such as means and standard deviations during training, which may otherwise be estimated less accurately in a batch-wise setting.

% \medskip

% \noindent
For the true model, ideally, increasing the NODE horizon does not increase the loss, as the $\ell^2$ error is averaged over the number of prediction steps. However, numerical integration errors also play a role. In particular, for systems that require smaller time steps for accurate numerical integration, increasing the NODE horizon can increase accumulated error even when the model is identified exactly. This effect is especially pronounced in stiff systems, such as the Van der Pol oscillator in the stiff regime, where shorter horizons may yield small errors while longer horizons introduce significant numerical artifacts that can affect the loss and resulting gradients.

\section{Computation details}\label{appendix:compute}
As is common in machine learning, due to the multiple local minima in parameter space of the optimization landscape, we use a K-fold training approach. The total training data is split into K(=5) folds with each fold having $80\%$ of the total available data for training and $20\%$ for testing. Given our small sizes of our NNs and datasets, these K-folds can be executed on a laptop (AMD Ryzen 9 4900 HS, Nvidia RTX 2060m) or desktop (Intel i9 10940X, Nvidia RTX 3080) with consumer grade CPUs and GPUs. More intense workloads such as using larger networks, optimizing over longer horizons, training on larger datasets, or even just post-training distillation from a large number of saved models, can benefit from \verb|PyTorch| and \verb|Python|'s parallelization libraries -- thus, the dataset sizes we use in \cref{chaotic_examples} are taken accordingly to split $80\%$ of $m$ points of training data into batches of $2^b$ data points. Accordingly, numerous small instantiations can be run simultaneously with parallelized K-folds or fewer larger instantiations can be run but with parallelized training and post-training including mAIC computations, visualizations, etc. 

\section*{Acknowledgments}
We would like to thank Colby Fronk and Jared Jonas for insightful discussions.

\vspace{-0.1in}

\bibliographystyle{siamplain}
\bibliography{references}

@article{wils16,
  title={Isostable reduction of periodic orbits},
  author={Wilson, D. and Moehlis, J.},
  journal={Physical Review E},
  volume={94},
  pages={052213},
  year={2016}
}

@article{fronk23,
  title={Interpretable polynomial neural ordinary differential equations},
  author={Fronk, Colby and Petzold, Linda},
  journal={Chaos},
  volume={33},
  number={4},
  year={2023},
  publisher={AIP Publishing}
}

@article{chua1986double,
  title={The double scroll family},
  author={Chua, LEONO and Komuro, Motomasa and Matsumoto, Takashi},
  journal={Transactions on Circuits and Systems},
  volume={33},
  number={11},
  pages={1072--1118},
  year={1986},
  publisher={IEEE}
}

@article{iyer2024expressive,
  title={Expressive Symbolic Regression for Interpretable Models of Discrete-Time Dynamical Systems},
  author={Iyer, Adarsh and Boddupalli, Nibodh and Moehlis, Jeff},
  journal={arXiv preprint arXiv:2406.06585},
  year={2024}
}

@inproceedings{wccm2024,
  title = {Using artificial neural networks for symbolic regression},
  author = {Boddupalli, N. and Moehlis, J. and Matchen, T.},
  booktitle={World Congress on Computational Mechanics (WCCM)},
  year = {2024}
}

@article{symanntex_paper,
  title={Symbolic regression via neural networks},
  author={Boddupalli, N and Matchen, T and Moehlis, J},
  journal={Chaos: An Interdisciplinary Journal of Nonlinear Science},
  volume={33},
  number={8},
  year={2023},
  publisher={AIP Publishing}
}

@article{POD,
  title={The proper orthogonal decomposition in the analysis of turbulent flows},
  author={Berkooz, Gal and Holmes, Philip and Lumley, John L},
  journal={Annual Review of Fluid Mechanics},
  volume={25},
  number={1},
  pages={539--575},
  year={1993},
  publisher={Annual Reviews 4139 El Camino Way, PO Box 10139, Palo Alto, CA 94303-0139, USA}
}

@article{colby_2,
  title={Training stiff neural ordinary differential equations with implicit single-step methods},
  author={Fronk, Colby and Petzold, Linda},
  journal={Chaos},
  volume={34},
  number={12},
  year={2024},
  publisher={AIP Publishing}
}

@article{colby_3,
  title={The vanishing gradient problem for stiff neural differential equations},
  author={Fronk, Colby and Petzold, Linda},
  journal={Chaos: An Interdisciplinary Journal of Nonlinear Science},
  volume={35},
  number={11},
  year={2025},
  publisher={AIP Publishing}
}

@inproceedings{takens2006detecting,
  title={Detecting strange attractors in turbulence},
  author={Takens, Floris},
  booktitle={Dynamical Systems and Turbulence, Warwick 1980: Proceedings of a symposium held at the University of Warwick 1979/80},
  pages={366--381},
  year={2006},
  organization={Springer}
}

@article{raissi2018multistep,
  title={Multistep neural networks for data-driven discovery of nonlinear dynamical systems},
  author={Raissi, Maziar and Perdikaris, Paris and Karniadakis, George Em},
  journal={arXiv preprint arXiv:1801.01236},
  year={2018}
}

@article{l_half_smooth_regularization,
  title={Batch gradient method with smoothing {L}$_{1/2}$ regularization for training of feedforward neural networks},
  author={Wu, Wei and Fan, Qinwei and Zurada, Jacek M and Wang, Jian and Yang, Dakun and Liu, Yan},
  journal={Neural Networks},
  volume={50},
  pages={72--78},
  year={2014},
  publisher={Elsevier}
}

@article{van_der_pol,
  title={On “relaxation-oscillations”},
  author={Van der Pol, Balth},
  journal={The London, Edinburgh, and Dublin Philosophical Magazine and Journal of Science},
  volume={2},
  number={11},
  pages={978--992},
  year={1926},
  publisher={Taylor \& Francis}
}

@article{discretize-optimize,
  title={Discretize-optimize vs. optimize-discretize for time-series regression and continuous normalizing flows},
  author={Onken, Derek and Ruthotto, Lars},
  journal={arXiv preprint arXiv:2005.13420},
  year={2020}
}

@article{sindy,

  title={Discovering governing equations from data by sparse identification of nonlinear dynamical systems},

  author={S. L. Brunton and J. L. Proctor and J. N. Kutz},

  journal={Proceedings of the National Academy of Sciences},

  volume={113},

  pages={3932-3937},

  year={2016}

}

@article{quad16,

  title={Prediction of dynamical systems by symbolic regression},

  author={M. Quade and M. Abel and K. Shafi and R. K. Niven and B. R. Noack},

  journal={Physical Review E},

  volume={94},

  pages={012214},

  year={2016}

}

@article{budi12,

  title={Applied {K}oopmanism},

  author={M. Budisic and R. Mohr and I. Mezic},

  journal={Chaos},

  volume={22},

  pages={047510},

  year={2012}

}

@article{arba17,

  title={Ergodic theory, dynamic mode decomposition, and computation of spectral properties of the {K}oopman operator},

  author={H. Arbabi and I. Mezic},

  journal={SIAM Journal on Applied Dynamical Sytems},

  volume={16},

  pages={2096-2126},

  year={2017}

}

@article{mezic2005spectral,
  title={Spectral properties of dynamical systems, model reduction and decompositions},
  author={Mezi{\'c}, Igor},
  journal={Nonlinear Dynamics},
  volume={41},
  number={1},
  pages={309--325},
  year={2005},
  publisher={Springer}
}

@article{DMD,
  title={Dynamic mode decomposition of numerical and experimental data},
  author={Schmid, Peter J},
  journal={Journal of Fluid Mechanics},
  volume={656},
  pages={5--28},
  year={2010},
  publisher={Cambridge University Press}
}

@Article{bong07,
  author = 	 {J. Bongard and H. Lipson},
  title = 	 {Automated reverse engineering of nonlinear dynamical systems},
  journal = 	 {Proc. Natl. Acad. Sci.},
  year = 	 {2007},
  volume = 	 {104},
  pages = 	 {9943},
}

@article{rossler1976chaotic,
  title={Chaotic behavior in simple reaction systems},
  author={R{\"o}ssler, Otto E},
  journal={Zeitschrift f{\"u}r Naturforschung A},
  volume={31},
  number={3-4},
  pages={259--264},
  year={1976},
  publisher={Verlag der Zeitschrift f{\"u}r Naturforschung}
}

@article{havok,
  title={Chaos as an intermittently forced linear system},
  author={Brunton, Steven L and Brunton, Bingni W and Proctor, Joshua L and Kaiser, Eurika and Kutz, J Nathan},
  journal={Nature Communications},
  volume={8},
  number={1},
  pages={19},
  year={2017},
  publisher={Nature Publishing Group UK London}
}

@article{quade2019glyph,
  title={Glyph: Symbolic Regression Tools},
  author={Quade, Markus and Gout, Julien and Abel, Markus},
  journal={Journal of Open Research Software},
  volume={7},
  number={1},
  article={19},
  year={2019},
  pages = {1-8},
  publisher={Ubiquity Press}
}

@article{lorenz1963deterministic,
  title={Deterministic nonperiodic flow},
  author={Lorenz, Edward N},
  journal={Journal of Atmospheric Sciences},
  volume={20},
  number={2},
  pages={130--141},
  year={1963}
}

@book{goodfellow2016deep,
  title={Deep Learning},
  author={Goodfellow, Ian and Bengio, Yoshua and Courville, Aaron},
  year={2016},
  publisher={MIT Press},
  address={Cambridge, Massachusetts}
}

@article{neural_odes,
  title={Neural ordinary differential equations},
  author={Chen, Ricky TQ and Rubanova, Yulia and Bettencourt, Jesse and Duvenaud, David K},
  journal={Advances in Neural Information Processing Systems},
  volume={31},
  year={2018}
}

@article{weak_sindy,
  title={Weak {SINDy}: Galerkin-based data-driven model selection},
  author={Messenger, Daniel A and Bortz, David M},
  journal={Multiscale Modeling \& Simulation},
  volume={19},
  number={3},
  pages={1474--1497},
  year={2021},
  publisher={SIAM}
}

@article{rosenfeld2022dynamic,
  title={Dynamic mode decomposition for continuous time systems with the {L}iouville operator},
  author={Rosenfeld, Joel A and Kamalapurkar, Rushikesh and Gruss, L Forest and Johnson, Taylor T},
  journal={Journal of Nonlinear Science},
  volume={32},
  pages={1--30},
  year={2022},
  publisher={Springer}
}

@article{l_half_regularization,
    title = {{L$_{1/2}$ regularization}},
    year = {2010},
    journal = {Science China Information Sciences},
    author = {Xu, ZongBen and Zhang, Hai and Wang, Yao and Chang, XiangYu and Liang, Yong},
    number = {6},
    month = {6},
    pages = {1159--1169},
    volume = {53},
    
    isbn = {1143201000},
    
    issn = {1674-733X}
}

@article{l1_regularization,
  title={Regression shrinkage and selection via the lasso},
  author={Tibshirani, Robert},
  journal={Journal of the Royal Statistical Society Series B: Statistical Methodology},
  volume={58},
  number={1},
  pages={267--288},
  year={1996},
  publisher={Oxford University Press}
}

@inproceedings{kingma2015adam,
  title={Adam: A Method for Stochastic Optimization},
  author={Kingma, Diederik P and Ba, Jimmy},
  booktitle={Proceedings of International Conference on Learning Representations},
  year={2014}
}

@inproceedings{transformers_1D,
  title={Predicting ordinary differential equations with transformers},
  author={Becker, S{\"o}ren and Klein, Michal and Neitz, Alexander and Parascandolo, Giambattista and Kilbertus, Niki},
  booktitle={International conference on machine learning},
  pages={1978--2002},
  year={2023},
  organization={PMLR}
}

@inproceedings{
odeformer,
title={{ODEF}ormer: Symbolic Regression of Dynamical Systems with Transformers},
author={St{\'e}phane d'Ascoli and S{\"o}ren Becker and Philippe Schwaller and Alexander Mathis and Niki Kilbertus},
booktitle={The Twelfth International Conference on Learning Representations},
year={2024},
}

@article{akaike1974new,
  title={A new look at the statistical model identification},
  author={Akaike, Hirotugu},
  journal={IEEE Transactions on Automatic Control},
  volume={19},
  number={6},
  pages={716--723},
  year={1974},
  publisher={Ieee}
}

@book{guckenheimer2013nonlinear,
  title={Nonlinear Oscillations, Dynamical Systems, and Bifurcations of Vector Fields},
  author={Guckenheimer, John and Holmes, Philip},
  year={1983},
  publisher={Springer},
  address={New York}
}

@book{tangirala2018principles,
  title={Principles of system identification: theory and practice},
  author={Tangirala, Arun K},
  year={2018},
  publisher={CRC press}
}

@article{chen1995universal,
  title={Universal approximation to nonlinear operators by neural networks with arbitrary activation functions and its application to dynamical systems},
  author={Chen, Tianping and Chen, Hong},
  journal={IEEE Transactions on Neural Networks},
  volume={6},
  number={4},
  pages={911--917},
  year={1995},
  publisher={IEEE}
}

@article{rossler_parameters,
  title={Unstable periodic orbits and templates of the {R}{\"o}ssler system: toward a systematic topological characterization},
  author={Letellier, C and Dutertre, P and Maheu, B},
  journal={Chaos},
  volume={5},
  number={1},
  pages={271--282},
  year={1995},
  publisher={American Institute of Physics}
}

@article{ai_hilbert,
  title={Evolving scientific discovery by unifying data and background knowledge with {AI} {H}ilbert},
  author={Cory-Wright, Ryan and Cornelio, Cristina and Dash, Sanjeeb and El Khadir, Bachir and Horesh, Lior},
  journal={Nature Communications},
  volume={15},
  number={1},
  pages={5922},
  year={2024},
  publisher={Nature Publishing Group UK London}
}

\end{document}